\documentclass[11pt]{article} 

\usepackage{times}
\usepackage{amsfonts}
\usepackage{xcolor}
\usepackage{amsthm}

\usepackage{amsmath}
\usepackage{amssymb}
\usepackage[numbers,sort&compress]{natbib}

\usepackage{graphicx}
\usepackage{url}
\usepackage{algorithm2e}
\PassOptionsToPackage{algo2e,ruled}{algorithm2e}
\usepackage{amsthm}
\usepackage[margin=1.1in]{geometry}
\usepackage{hyperref}
\hypersetup{colorlinks=true,linkcolor=blue,citecolor=blue}
\usepackage[capitalise]{cleveref}

\usepackage{times}
\usepackage[mathscr]{eucal}

\newtheorem{theorem}{Theorem}
\newtheorem{lemma}[theorem]{Lemma}
  
  \newtheorem{proposition}[theorem]{Proposition}

  \newtheorem{definition}[theorem]{Definition}

\title{Adversarial Rewards in Universal Learning for Contextual Bandits}
\usepackage{times}

\newcommand{\Acal}{\mathcal{A}}
\newcommand{\Bcal}{\mathcal{B}}
\newcommand{\Ccal}{\mathcal{C}}
\newcommand{\Dcal}{\mathcal{D}}
\newcommand{\Ecal}{\mathcal{E}}
\newcommand{\Fcal}{\mathcal{F}}
\newcommand{\Gcal}{\mathcal{G}}
\newcommand{\Hcal}{\mathcal{H}}
\newcommand{\Ical}{\mathcal{I}}

\newcommand{\Pcal}{\mathcal{P}}
\newcommand{\Rcal}{\mathcal{R}}
\newcommand{\Scal}{\mathcal{S}}
\newcommand{\Tcal}{\mathcal{T}}
\newcommand{\Ucal}{\mathcal{U}}

\newcommand{\Xcal}{\mathcal{X}}
\newcommand{\Ycal}{\mathcal{Y}}
\newcommand{\Ocal}{\mathcal{O}}

\newcommand{\Ebb}{\mathbb{E}}
\newcommand{\Nbb}{\mathbb{N}}
\newcommand{\Pbb}{\mathbb{P}}

\newcommand{\Xbb}{\mathbb{X}}
\newcommand{\Ybb}{\mathbb{Y}}
\newcommand{\Zbb}{\mathbb{Z}}

\usepackage{bbm}

\newcommand{\1}{\mathbbm{1}}

\definecolor{dark_red}{rgb}{0.2,0,0}

\newcommand{\comment}[1]{}

\newcommand{\mb}[1]{\ensuremath{\boldsymbol{#1}}}

\newcommand{\paren}[1]{\left( #1 \right)}
\newcommand{\sqb}[1]{\left[ #1 \right]}

\usepackage{cleveref}
\newlength{\dhatheight}

\newcommand{\EXPINF}{\mathrm{EXPINF}}
\newcommand{\EXPIX}{\mathrm{EXP3.IX}}

\newcommand{\X}{\mathcal X}

\newcommand{\argmax}{\mathop{\rm argmax}}
\newcommand{\argmin}{\mathop{\rm argmin}}

\renewcommand{\limsup}{\mathop{\rm limsup}}

\newcommand{\floor}[1]{\left\lfloor #1 \right\rfloor}

\DeclareSymbolFont{bbold}{U}{bbold}{m}{n}
\DeclareSymbolFontAlphabet{\mathbbold}{bbold}

\newcommand{\ProcX}{\mathbb{X}}

\newcommand{\OKC}{\mathcal{C}_{2}}
\newcommand{\SMV}{\OKC}
\newcommand{\UKC}{\mathcal{C}_{3}}
\newcommand{\FS}{\UKC}

\newcommand{\ignore}[1]{}
\newcommand{\private}[1]{}

\colorlet{sgreen}{black!45!green}

\newtheorem{condition}{Condition}

\newsavebox{\savepar}

\makeatletter
\newcommand{\vast}{\bBigg@{3}}
\newcommand{\Vast}{\bBigg@{4}}
\makeatother

\author{
  Moise Blanchard\\
  MIT\\
  \small{\texttt{moiseb@mit.edu}}
  \and
  Steve Hanneke\\
  Purdue University\\
  \small{\texttt{steve.hanneke@gmail.com}}
  \and 
  Patrick Jaillet\\
  MIT\\
  \small{\texttt{jaillet@mit.edu}}
}
\date{}
\setcitestyle{numbers,open={[},close={]}} 

\usepackage{thmtools}
\usepackage[mathscr]{eucal}
\usepackage{bbm}

\PassOptionsToPackage{algo2e,ruled,linesnumbered}{algorithm2e}

\newcommand{\acks}[1]{\section*{Acknowledgments}#1}

\newcommand{\nonl}{\renewcommand{\nl}{\let\nl\oldnl}}

\usepackage{latexsym,tikz,graphicx}
\usetikzlibrary{calc}
\usetikzlibrary{decorations.pathreplacing}
\usetikzlibrary{patterns}
\renewenvironment{proof}[1][]{\par
  \trivlist
  \item[\hskip\labelsep
        \bfseries Proof#1.\ ]\ignorespaces
}{%
 \hfill $\blacksquare$\endtrivlist
}

\usepackage{multirow}

\begin{document}
\maketitle

\begin{abstract}
We study the fundamental limits of learning in contextual bandits, where a learner's rewards depend on their actions and a known context, which extends the canonical multi-armed bandit to the case where side-information is available. We are interested in \emph{universally consistent} algorithms, which achieve sublinear regret compared to any measurable fixed policy, without any function class restriction. For stationary contextual bandits, when the underlying reward mechanism is time-invariant, \cite{blanchard:22e} characterized \emph{learnable} context processes for which universal consistency is achievable; and further gave algorithms ensuring universal consistency whenever this is achievable, a property known as \emph{optimistic universal consistency}. It is well understood, however, that reward mechanisms can evolve over time, possibly adversarially, and depending on the learner's actions. We show that optimistic universal learning for contextual bandits with adversarial rewards is impossible in general, contrary to all previously studied settings in online learning---including standard supervised learning. We also give necessary and sufficient conditions for universal learning under various adversarial reward models, and an exact characterization for online rewards. In particular, the set of learnable processes for these reward models is still extremely general---larger than i.i.d., stationary or ergodic---but in general strictly smaller than that for supervised learning or stationary contextual bandits, shedding light on new adversarial phenomena.
\end{abstract}


\tableofcontents

\section{Introduction}
\label{sec:introduction}

The contextual bandit setting is a central problem in statistical decision-making. This setting models the interaction between a learner or decision maker, and a reward mechanism. At each iteration of the learning process, the learner observes a \emph{context} $x\in\Xcal$ (also known as covariate in the statistical learning literature), then selects an \emph{action} $a\in\Acal$ to perform. The decision maker then receives a reward based on the context and selected action, which can then be used to perform informed future actions. As a classical example, this framework can model the problem of online personalized recommendations. For any new customer, an online store provides a list of product recommendations. Based on the reward obtained from actions of the customer, e.g., if they purchase an item, the store can then update its recommendations for future customers. The major difference with the standard supervised learning framework is that the learner can only observe the reward of the selected action, referred to as partial feedback, instead of the full-feedback case of supervised learning in which a learner can directly compute the reward (or loss) of non-selected actions. Further, instead of estimating the reward mechanism, the goal in contextual bandits is to achieve low regret compared to the optimal actions in hindsight. New phenomena arise from these characteristics, including the well-known exploration/exploitation trade-off: algorithms should balance between exploiting known high-reward actions and exploring new actions which potentially could yield higher rewards. In the present work, we aim to shed light on the fundamental question of \emph{learnability} in contextual bandits and unveil key differences from the classical full-feedback setting.

\paragraph{Universal consistency.}
We focus on the foundational notion of \emph{consistency}. In the contextual bandit context, a learner is consistent if its long-term excess regret vanishes. Contexts are modeled by a stochastic process $\Xbb=(X_t)_{t\geq 1}$. If $\hat a_t$ is the selected action and $r_t$ the reward function at time $t$, we ask that for any measurable policy $\pi^*$, 
\begin{equation*}
    \limsup_{T\to\infty}\frac{1}{T}\sum_{t=1}^T r_t(\pi^*(X_t)) - r_t(\hat a_t) \leq 0 \quad (a.s.).
\end{equation*}
As shown in the above equation, we follow a traditional regret analysis, where we compare the learner to a fixed policy (static regret) as opposed to switching regret where the comparison policy may also change. For robustness and generality, one commonly aims to design algorithms that ensure consistency for a large class of instances. In this paper, we consider the strongest notion of \emph{universal consistency}, introduced in \cite{hanneke:21}, which asks that a learning rule is consistent for any possible reward mechanism---informally, any form of reward functions $(r_t)_{t\geq 1}$. The notion of universal consistency was mostly studied in the full-feedback supervised learning framework. In this context, a learner observes a stream of data $(X_t,Y_t)_{t\geq 1}$ and makes predictions $\hat Y_t$ at each step. Thus, it is universally consistent if irrespective of the underlying mechanism relating values $Y$ to contexts $X$, its average excess error compared to any measurable predictor function $f:\Xcal\to\Ycal$ vanishes: $\limsup_{T\to\infty}\frac{1}{T}\sum_{t\leq T}\ell(\hat Y_t,Y_t) - \ell(f(X_t),Y_t) \leq 0\; (a.s.)$. Starting with the work of \cite{stone:77} which proved universal consistency for a large class of local average estimators in Euclidean spaces, a significant line of work focused on extending these results. Notably, one can achieve universal consistency for more general spaces and loss functions \citep{devroye:96,gyorfi:02}. More recently, \cite{hanneke:21b,gyorfi2021universal,cohen:22} provided learning rules for universal learning under a provably-minimal assumption on the context space $\Xcal$ known as essential separability. While these works focused on independent identically distributed (i.i.d.) data, more restricted consistency results were also obtained for non-i.i.d. mixing, stationary ergodic data processes \citep{morvai:96,gyorfi:99,gyorfi:02} or processes satisfying the law of large numbers \citep{morvai:99,gray:09,steinwart:09}.

\paragraph{Optimistic learning.}
Following these efforts to generalize results, a natural question arises: what are the fundamental limits of universal consistency? To answer this question, we adopt the framework of optimistic learning \cite{hanneke:21,hanneke:21b,blanchard:22a} which aims to study learning with \emph{provably-minimal} assumptions. 
As originally introduced by \cite{hanneke:21}, the notion of optimistically universal learning is motivated by the following reasoning.
If we are interested in designing a learning algorithm that achieves a particular learning guarantee (in our case, universal consistency under the process $\Xbb$), to succeed we must necessarily assume that such a guarantee is at least \emph{possible} (i.e., that there exists a learner achieving this guarantee).  Since such an assumption typically cannot be verified empirically, making such an assumption is an act of \emph{optimism}. As such, this is referred to as the \emph{optimist's assumption} \cite{hanneke:21}.
The main question in this literature is to determine whether there exists a learning algorithm which achieves the desired guarantee given \emph{only} the assumption that it is possible to do so (in our case, this means making no additional assumptions about the process $\Xbb$).  Such a learning algorithm is said to be \emph{optimistically universal}. Since the optimist's assumption is always \emph{necessary} to achieve the desired guarantee, an optimistically universal learning algorithm succeeds under the \emph{minimal} possible assumptions.  Thus, in the present context, an algorithm is called optimistically universal if it is universally consistent under every process $\Xbb$ for which there exists a universally consistent learner: that is, it \emph{learns whenever learning is possible}.  The key point is that the learner whose existence establishes that $\Xbb$ admits universal consistency may depend on the distribution of $\Xbb$, whereas an optimistically universal learner must be consistent under \emph{every} such $\Xbb$.

In the present work, we aim to understand whether optimistically universal learning is possible for contextual bandits under various categories of reward adversaries. It is useful first to understand and characterize the minimal assumptions for the existence of a universally consistent learning rule: that is, which processes $\Xbb$ satisfy the optimist's assumption. Informally, we aim to characterize
\begin{equation*}
    \Ccal=\{\Xbb:\exists \text{ learning rule }f_\cdot \text{ s.t. }\forall \text{ rewards within a given model, }f_\cdot \text{ is consistent}\}.
\end{equation*}
Second, we search for \emph{optimistically universal} procedures: i.e., which are universally consistent under all processes where this is possible ($\Xbb\in\Ccal$).
For any process $\Xbb$, if such an algorithm fails to be universally consistent, we are guaranteed that no other algorithm would be either.

\paragraph{Universal learning in contextual bandits.}
While the literature on universal learning in the case of full-feedback is very extensive, it is surprisingly sparse for partial-feedbacks. Previous literature mostly investigated stochastic contextual bandits under important structural assumptions on rewards, such as smoothness or margin conditions. Closest to universal learning---in which one relaxes assumptions on the reward mechanism---\cite{yang2002randomized} showed that for continuous rewards in the contexts, strong consistency can be achieved with traditional non-parametric methods, for Euclidean context spaces. \cite{blanchard:22e} gave the first results for contextual bandits on universal consistency per se. They focus on stationary rewards---the underlying reward mechanism is invariant over time---and show in particular that for the main case of interest---finite action spaces $\Acal$---universal consistency is achievable under the same class of processes as for the noiseless full-feedback case. In contrast with previous literature, the proposed learning rules are consistent without any assumptions on the rewards, on general spaces and under large classes of non-i.i.d. contexts. Further, they show that optimistically universal learning rules always exist for stationary bandits.

The present work challenges the stationarity assumption from \cite{blanchard:22e}. In particular, this does not allow for changes in the underlying reward mechanism, a behavior ubiquitous in current applications. It is well-known that the distribution of contexts and rewards can shift over time, such as seasonal changes in consumer behavior and can be adversarial. Our analysis mainly focuses on two models for the strengh of the adversary: oblivious rewards for which the reward mechanism can depend on the past context history, but not the past actions of the learner; and the strongest online rewards for which the rewards can be adaptive on past contexts and selected actions. This study shows that having adversarial rewards---as opposed to stationary rewards---plays a crucial role in the fundamental limits of learnability for contextual bandits, and represents a significant advancement in the general analysis of more intricate decision-making processes, such as reinforcement learning.

\subsection{Related works}

\paragraph{Literature on optimistic supervised learning.} Optimistic learning was first introduced by \cite{hanneke:21} for the \emph{realizable} (noiseless) case when values are exactly given as $Y_t=f(X_t)$ for some unknown measurable function $f:\Xcal\to\Ycal$, and provided necessary conditions and sufficient conditions for universal learning. The characterization was then completed in a subsequent line of work \citep{blanchard:22b,blanchard:22c,blanchard:22a}. In particular, while nearest-neighbor is not consistent even for i.i.d. processes in general metric spaces \citep{cerou2006nearest}, a simple variant with restricted memory is optimistically universal for general separable metric spaces. Notably, the corresponding class of learnable processes---which intuitively asks that the process visits sublinearly measurable partition of the ambient space---is significantly larger than previously considered relaxations of the i.i.d. assumptions. For more general noisy data generating processes \cite{hanneke:22a,blanchard:22d} gave complete characterizations and showed that universal learning can be achieved not only for noisy data but arbitrarily dependent values $\Ybb$ on the contexts $\Xbb$, possibly even adversarial to the learner's predictions. Specifically, \cite{blanchard:22d} showed that under mild assumptions on the value space---including totally-bounded-metric spaces---optimistically universal learning with noisy values is possible on the exact same class of processes as for noiseless values. Hence, learning with arbitrary or adversarial responses comes at no generality expense for the full-feedback setting.

\paragraph{Literature on contextual bandits and non-stationarity.} 
The concept of contextual bandits was first introduced in a limited context for single-armed bandits \cite{woodroofe1979one,sarkar1991one}. Since then, considerable effort was made to generalize the framework and provide efficient methods under important structural assumptions on the rewards. Most of the literature considered parametric assumptions \cite{wang2005bandit,langford2007epoch, goldenshluger2009woodroofe,bubeck2012regret,auer2016algorithm,rakhlin2016bistro}, but substantial progress has also been achieved in the non-parametric setting towards obtaining minimax guarantees under smoothness (e.g., Lipschitz) conditions or margin assumptions \cite{lu2009showing,rigollet2010nonparametric,slivkins2011contextual,perchet2013multi}, with further refinements including \cite{guan2018nonparametric,reeve2018k}.

While the above-cited works mostly focus on i.i.d. data, the non-stationary case has also been studied in the literature. The fact that the reward distribution can change over time has been widely acknowledged in the established parametric setting for contextual bandits, and has been explored under various models including  \cite{besbes2014stochastic,hariri2015adapting,karnin2016multi,luo2018efficient,liu2018change,wu2018learning,chen2019new}. The non-parametric case, more relevant to our work has also been considered for Lipschitz rewards and margin conditions \cite{slivkins2011contextual,suk2021self}. We note however, that these works often consider non-static regret, where the baseline is also non-stationary, while we focus on the excess regret compared to \emph{fixed} policies.

\subsection{Summary of the present work}
We mainly focus on bounded rewards. Our first main result shows that in the main case of interest of finite action spaces $\Acal$ and separable metrizable spaces $\Xcal$ admitting a non-atomic probability measure, optimistic universal learning is impossible, even under the weakest adversarial model which we call \emph{memoryless}: rewards conditionally on their selected action and context are independent but may follow different conditional distributions. This implies that adapting algorithms for specific context processes is necessary to ensure universal learning. This is the first example of such a phenomenon for online learning, for which previously considered settings always admitted optimistically universal learning rules, including realizable (noiseless) supervised learning \cite{hanneke:21,blanchard:22a,blanchard:22b}, arbitrarily noisy (potentially adversarial rewards) supervised learning \cite{hanneke:22a,blanchard:22d}, and stationary contextual bandits \cite{blanchard:22e}. Intuitively, \emph{personalization} and \emph{generalization} are incompatible for contextual bandits with adversarial rewards.

Next, we study universally learnable processes for various adversarial reward models. On the negative side, we show that in the main case of interest, the set of learnable processes for stationary contextual bandits or supervised learning denoted $\Ccal_2$ is not anymore fully learnable even for memoryless rewards: learning with adversarial rewards is fundamentally more difficult. This comes as a surprising result since $\Ccal_2$ processes admitted universal learning in all previous learning settings. We further identify novel necessary and sufficient conditions, involving intricate behavior of duplicates in the context process. In particular, for memoryless, oblivious, and online rewards, the set of learnable processes is strictly between $\Ccal_2$ and a smaller class $\Ccal_1$. For this same case of interest, we give an exact characterization of these learnable processes for online rewards: this characterization involves a sort of convergence rate of the instance process towards its limit distribution. Given the knowledge of this rate, universal learning is achievable with a learning rule that we provide; on the other hand, without a priori knowledge on this rate, universal learning is impossible since optimistic universal learning is not achievable. While we leave the exact characterization for memoryless and oblivious rewards as an open question for finite action spaces $\Acal$ and context spaces admitting a non-atomic probability measure, our characterizations in all other cases are complete.

Last, we give extensions of the above results, when the rewards are unbounded or satisfy some regularity constraints, namely uniform continuity.

\subsection{Overview of contributions and techniques}

\paragraph{Non-existence of optimistically universal learning rules.} 

The proof involves several major steps. First, one needs to show that universal learning is achievable for a large class of processes. In particular, we show that deterministic $\Ccal_2$ processes are learnable, where $\Ccal_2$ is the characterization of learnable processes for supervised learning or stationary contextual bandits. This is achieved by assigning each distinct instance a multi-armed bandit learner designed to learn the best action for this instance, which corresponds to pure \emph{personalization}. Next, we argue that $\Ccal_1$ processes---the characterization of learnable processes for countable action spaces $\Acal$ in stationary contextual bandits---can be learned with the same structural risk minimization approach introduced by \cite{blanchard:22e} for stationary contextual bandits, which corresponds to \emph{generalization}.

The main challenge is to show that one cannot universally learn both classes of processes (deterministic $\Ccal_2$ and $\Ccal_1$) with a unique algorithm. At the high level, we show that by contradiction, personalization and generalization are incompatible. We consider a $\Ccal_1$-like algorithm, where instances are i.i.d. during a phase, then the same sequence is repeated many times. The reward is identical for each duplicate and has the following behavior: one \emph{safe} action $a_2$ always has relatively high reward, and an \emph{uncertain} action $a_1$ has random reward. We then show that because of the $\Ccal_1$ property, the algorithm needs to follow the safe action in order to be consistent: if it explores the uncertain action too often, the incurred loss is significant. More precisely, we show that the exploration rate of the unsafe action $a_2$ decays to $0$. Once the algorithm reaches a certain threshold, we stop the stochastic process and consider a realization of the uncertain rewards and $\Ccal_1$-like process. Once these are taken as deterministic, the optimal policy would be to use the action $a_2$ when it has high reward, which the algorithm did not perform. Repeating this process inductively with decaying threshold, we can show that on a deterministic $\Ccal_2$ process, the algorithm is not universally consistent.

\paragraph{New classes of stochastic processes for learning theory.} We identify novel classes of processes that arise in the characterization of learnable processes. In the main case of interest, we give a new necessary condition $\Ccal_4$. Informally, while $\Ccal_2$ processes only required that the process the process visits only a sublinear number of sets from any countable partition of the context space $\Xcal$, the necessary condition $\Ccal_4$ requires this sublinear behavior to be uniform spatially in $\Xcal$. Loosely speaking, when the convergence speed of the sublinear visit property is heterogeneous across space, one can take advantage of these discrepancies with adversarial rewards together with a somewhat similar personalization/generalization incompatibility phenomenon as the one described above. More precisely, if $\Ccal_4$ is not satisfied, locally in the context space $\Xcal$, one can find the following behavior: contexts are duplicated across phases of exponential time-length, for arbitrarily small exponent. One can then consider oblivious rewards---rewards that may depend on past contexts $\Xbb_{\leq t}$ but only the selected action $\hat a_t$ at time $t$---that are identical on duplicates but with one safe and one uncertain option as above. Eventually, the algorithm's exploration rate of the uncertain action decays to $0$. However, for a given fixed realization of the rewards, this is suboptimal. In this proof, the dependence of the rewards on past contexts was necessary to make sure that during each constructed phase, no information on the rewards of future local space zones is revealed.

On the positive side, we introduce a novel condition $\Ccal_5$ that is universally learnable, with $\Ccal_1\subsetneq \Ccal_5$ in general. Intuitively, this asks that there is a specific rate at which we can add duplicates while still preserving the $\Ccal_1$ behavior. This should be related to the property observed in \cite{blanchard:22e} that if we were to replace all duplicates with an arbitrary value $x_0\in \Xcal$, $\Ccal_2$ processes would belong to $\Ccal_1$. The $\Ccal_5$ property provides an intermediary condition. We now briefly describe the algorithm we introduce to achieve universal consistency on $\Ccal_5$ processes. The learning rule heavily relies on the knowledge of the correct rate to add duplicates. For all points included within this addition rate, we can use the structural risk minimization approach since these points still have $\Ccal_1$ behavior. For the remaining duplicates, we use pure personalization by assigning a bandit learner to each distinct instance. In particular, all deterministic $\Ccal_2$ processes belong to $\Ccal_5$. Further, we can show that for online rewards, condition $\Ccal_5$ is also necessary and as a result is an exact characterization of learnable processes in this setting. In particular, for online rewards, universal learning exactly requires the a priori knowledge of the correct rate to add duplicates.

Last, in an attempt to bridge the gap $\Ccal_5\subsetneq \Ccal_4$ remaining for oblivious rewards, we propose a new condition $\Ccal_6$ on processes that is necessary for universal learning. In the general case of context spaces $\Xcal$ admitting non-atomic probability distributions, we have $\Ccal_5\subset\Ccal_6\subsetneq\Ccal_4$. This shows that further uniform continuity than the $\Ccal_4$ condition is necessary. The condition can be tightened using the same proof for a stronger type of adversary that we call prescient for which the rewards can also depend on the complete sequence $\Xbb$ instead of the past revealed contexts to the learner. For these rewards, we can show that a stronger $\Ccal_7$---and simpler than $\Ccal_6$---is necessary. We believe in general $\Ccal_7\subsetneq\Ccal_6$ but more importantly, the question of whether $\Ccal_5=\Ccal_7$, is open. Hence, possibly, our characterizations for prescient and stronger reward models are tight.

\section{Preliminaries}
\label{sec:preliminaries}

Let $(\Xcal,\Bcal)$ be a separable metrizable Borel context space and $\Acal$ a separable metrizable Borel action space $\Acal$. When considering continuity assumptions, we suppose that $\Acal$ is given with a metric $d$. For countable action spaces, we use the discrete topology. We are interested in the following sequential contextual bandit framework: at step $t\geq 1$, the learner observes a context $X_t\in\Xcal$, then selects an action $\hat a_t\in\Acal$ and last, receives a reward $r_t\in\Rcal$ which may be stochastic. Unless mentioned otherwise, we suppose that the rewards are bounded $\Rcal=[0,\bar r]$ and that the upper bound $\bar r$ is known. Hence, without loss of generality we may pose $\bar r=1$. The learner is \emph{online} and as such, can only use the current history to selects the action $\hat a_t$. 
\begin{definition}[Learning rule]
    A \emph{learning rule} is a sequence $f_\cdot =(f_t)_{t\geq 1}$ of possibly randomized measurable functions $f_t:\Xcal^{t-1}\times \Rcal^{t-1}\times \Xcal\to \Acal$. The action selected at $t$ is $\hat a_t = f_t((X_s)_{s\leq t-1},(r_s)_{s\leq t-1},X_t)$.
\end{definition}

We now precise the data generation process. We suppose that the contexts $\Xbb=(X_t)_{t\geq 1}$ are generated from a general stochastic process. To define the rewards, $(r_t)_{t\geq 1}$, many models for the underlying reward mechanism are possible. \cite{blanchard:22e} considered the case of \emph{stationary} rewards when the rewards follow a conditional distribution $P_{r\mid a,x}$ conditionally on the selected action $\hat a_t$ and the context $X_t$ at the current time $t\geq 1$. We consider the considerably more general case of adversarial rewards. Of particular interest to the discussion of this paper will be 1. \emph{oblivious} rewards which correspond to the case when the learner plays a game against an adversary oblivious to the player's actions and 2. \emph{online} rewards when the adversary can choose rewards depending on the complete history of contexts, selected actions and received rewards. For a stochastic process $\Xbb$, we will use the notation $\Xbb_{\leq t} = (X_{t'})_{t'\leq t}$. Also, for a measurable set $A\in\Bcal$, we will use the shorthand $\Xbb\cap A = \{X_t:X_t\in A, t\geq 1\}$.

\begin{definition}[Reward models]
\label{def:reward_model}
    The reward mechanism is said to be
    \begin{itemize}
        \item \emph{stationary (stat.)} if there is a conditional distribution $P_{r\mid a,x}$ such that the rewards $(r_t)_{t\geq 1}$ given their selected action $ a_t$ and context $X_t$ are independent and follow $P_{r\mid a,x}$
        \item \emph{memoryless} if there are conditional distributions $(P_{r\mid a,x,t})_{t\geq 1}$ such that $(r_t)_{t\geq 1}$ given their selected action $ a_t$ and context $X_t$ are independent for $t\geq 1$ and respectively follow $P_{r\mid a,x,t}$
        \item \emph{oblivious} if there are conditional distributions $(P_{r\mid a,\mb x_{\leq t}})_{t\geq 1}$ such that $r_t$ given the selected action $a_t$ and the past contexts $\Xbb_{\leq t}$, follows $P_{r\mid a,\mb x_{\leq t}}$
        \comment{\item \emph{prescient} if there are conditional distributions $(P_{r\mid a,\mb x_{t'\geq 1}})_{t\geq 1}$ such that $r_t$ given the selected action $a_t$ and the sequence of contexts $\Xbb$, follows $P_{r\mid a,\mb x_{t'\geq 1}}$}
        \item \emph{online} if there are conditional distributions $(P_{r\mid \mb a_{\leq t},\mb x_{\leq t},\mb r_{\leq t-1}})_{t\geq 1}$ such that $r_t$ given the sequence of selected actions $\mb a_{\leq t}$ and the sequence of contexts $\Xbb_{\leq t}$ and received rewards $\mb r_{\leq t-1}$, follows $P_{r\mid \mb a_{\leq t},\mb x_{\leq t}, \mb r_{\leq t-1}}$.
    \end{itemize}
\end{definition}

We refer to all the models except for the stationary one as \emph{adversarial}. To emphasize the dependence of the reward in the selected action, and the conditional distributions, we may write $r_t(a\mid X_t)$, $r_t(a\mid \Xbb_{\leq t})$, $r_t(a\mid \Xbb)$, and $r_t(a\mid \mb a_{\leq t-1}, \Xbb_{\leq t}, \mb r_{\leq t})$ for the corresponding reward models. When the conditioning is clear from context, we may simply write $r_t(a)$ for the reward if action $a$ is selected. The general goal in contextual bandits is to discover or approximate an optimal policy $\pi^*:\Xcal\to\Acal$ if it exists. For adversarial rewards, there may not exist a single optimal policy $\pi^*$. Instead, we aim for consistent algorithms that have sublinear regret compared to any fixed measurable policy.

\begin{definition}[Consistency and universal consistency]
    Let $\Xbb$ be a stochastic process on $\Xcal$, $(r_t)_{t\geq 1}$ be a reward mechanism and $f_\cdot$ be a learning rule. Denote by $(\hat a_t)_{t\geq 1}$ its selected actions. We say that $f_\cdot$ is \emph{consistent} under $\Xbb$ with rewards $r$ if for any measurable policy $\pi^*:\Xcal\to\Acal$,
    \begin{equation*}
        \limsup_{T\to\infty}\frac{1}{T}\sum_{t=1}^T r_t(\pi^*(X_t)) - r_t(\hat a_t) \leq 0,\quad (a.s.).
    \end{equation*}
    We say that $f_\cdot$ is \emph{universally consistent} for a given reward model if it is consistent under $\Xbb$ with any reward within the considered reward model.
\end{definition}

Even in the simplest case of full-feedback noiseless learning \cite{hanneke:21}, universal consistency is not always achievable. For instance, if the process $\Xbb$ visits a distinct instance at each step the learner, the information gathered on previous instances $\Xbb_{\leq t-1}$ does not provide information on the rewards for instance $X_t$. We are then interested in understanding the set of processes $\Xbb$ on $\Xcal$ for which universal learning is possible. More practically, we aim to provide optimistically universally consistent learning rules which, if they exist, would be universally consistent whenever this is possible.

\begin{definition}[Optimistically universal learning rule]
    For a given reward model which we write $model\in\{stat,memoryless,oblivious, prescient, online\}$, we define
    \begin{equation*}
        \Ccal_{model} = \{ \Xbb : \exists \text{ learning rule universally consistent for }model\text{ under }\Xbb \}. 
    \end{equation*}
    
    We say that a learning rule $f_\cdot$ is \emph{optimistically universal} for the reward model if it is universally consistent under any process $\Xbb\in\Ccal_{model}$ for that reward model.
\end{definition}

\noindent In general $\Ccal_{online}\comment{,\Ccal_{prescient}} \subset \Ccal_{oblivious} \subset \Ccal_{memoryless} \subset  \Ccal_{stat}.$
\comment{Although  priori $\Ccal_{online}$ and $\Ccal_{prescient}$ are not comparable, our results for online rewards can be easily extended to the stronger online and prescient model (together). Thus, we will think of online rewards as the strongest model.
}

\subsection{Two main classes of stochastic processes}

We give the definitions of two main conditions on stochastic processes arising in our characterizations of learnable processes. First, given a stochastic process $\Xbb$ on $\Xcal$, an extended process is given by $\tilde\Xbb =(X_t)_{t\in\Tcal}$ where $\Tcal\subset\Nbb$ is a possibly random subset of times---which can depend on any random variable, the process $\Xbb$ itself, rewards potentially observed by a learner, etc. We define the limit submeasure $\hat\mu_{\tilde\Xbb}$ as follows. For any $A\in\Bcal$,
\begin{equation*}
    \hat\mu_{\tilde\Xbb}(A) = \limsup_{T\to\infty} \frac{1}{T}\sum_{t\leq T,t\in\Tcal}\1_A(X_t).
\end{equation*}
The first condition intuitively asks that the expected empirical limsup frequency of sets $A\in\Bcal$ is a continuous sub-measure on $\Bcal$.

\begin{condition}[\citet{hanneke:21,blanchard:22e}]
  \label{con:kc}
  Let $\Xbb$ be a stochastic process and $\tilde\Xbb= (X_t)_{t\in\Tcal}$ an extended process. $\tilde\Xbb$ satisfies the condition if for every monotone sequence $\{A_k\}_{k=1}^{\infty}$ of measurable subsets of $\X$ with $A_k \downarrow \emptyset$,
  \begin{equation*}
    \lim_{k\to\infty}\Ebb[\hat\mu_{\tilde\Xbb}(A_k)] = 0.
  \end{equation*}
  We define $\Ccal_1'$ as the set of extended processes $\tilde \Xbb$ satisfying this condition. For clarity, we also define $\Ccal_1$ as the set of (classical) processes $\Xbb$ satisfying this condition (taking $\Tcal=\Nbb$).
\end{condition}

\noindent The next condition asks that $\Xbb$ visits a sublinear number of sets of any measurable partition of $\Xcal$.

\begin{condition}[\citet{hanneke:21}]
  \label{con:smv}
  For every sequence $\{A_k\}_{k=1}^{\infty}$ of disjoint measurable subsets of $\X$, $|\{ k : \Xbb_{\leq T} \cap A_k \neq \emptyset \}| = o(T) \text{ (a.s.)}.$ Denote by $\SMV$ the set of all processes 
  $\ProcX$ satisfying this condition.
\end{condition}

Intuitively, this condition asks that the process does not keep exploring completely different regions of the space $\Xcal$. This is known that even in the noiseless full-feedback setting, $\Ccal_2$ is a necessary condition for universal learning \cite{hanneke:21} since intuitively, the past history does not provide any information on newly visited regions for a learner. \cite{hanneke:21} showed that both classes above are very general classes of processes. Precisely, we have $\Ccal_1\subset\Ccal_2$ and i.i.d. processes, stationary ergodic processes, stationary processes and processes satisfying the law of large numbers belong to $\Ccal_1$.

\subsection{Useful algorithms}

Our learning rules will use as subroutine the following two algorithms. First, we will use the algorithm $\EXPIX$ for regret bounds with high-probability in adversarial bandits.

\begin{theorem}[\cite{neu2015explore}] \label{thm:multiarmed_bandits}
There exists an algorithm $\EXPIX$ for adversarial multi-armed bandit with $K\geq 2$ arms such that for any $\delta\in(0,1)$ and $T\geq 1$,
\begin{equation*}
    \max_{i\in[K]} \sum_{t=1}^T( r_t(a_i) - r_t(\hat a_t)) \leq 4\sqrt{KT\ln K} + \left(2\sqrt{\frac{KT}{\ln K}}+1\right)\ln \frac{2}{\delta},
\end{equation*}
with probability at least $1-\delta$.
\end{theorem}
We will always use a very simplified version of this result: there exists a universal constant $c>0$ such that 
\begin{equation*}
    \max_{i\in[K]} \sum_{t=1}^T( r_t(a_i) - r_t(\hat a_t)) \leq c\sqrt{KT\ln K}\ln \frac{1}{\delta},
\end{equation*}
with probability $1-\delta$ for $\delta\leq \frac{1}{2}$. Second, we use the $\EXPINF$ algorithm from \cite{blanchard:22e} which uses $\EXPIX$ as subroutine to achieve sublinear regret compared to an infinite countable sequence of experts.

\begin{theorem}[\cite{blanchard:22e}]\label{thm:infinite-exp4}
There is an online learning rule 
$\EXPINF$ using bandit feedback
such that for any countably infinite set of experts $\{E_1,E_2,\ldots\}$ 
(possibly randomized), for any $T\geq 1$ and $0<\delta\leq \frac{1}{2}$, with probability at least $1-\delta$,
\begin{equation*}
\max_{1\leq i \leq T^{1/8}} \sum_{t=1}^{T} \left( r_t(E_{i,t}) - r_t(\hat{a}_t) \right) 
\leq cT^{3/4}\sqrt{\ln T}\ln\frac{T}{\delta}.
\end{equation*}
where $c>0$ is a universal constant. Further, with probability one on the learning and the experts, there exists $\hat T$ such that for any $T\geq 1$,
\begin{equation*}
    \max_{1\leq i \leq T^{1/8}} \sum_{t=1}^{T} \left( r_t(E_{i,t}) - r_t(\hat{a}_t) \right) 
\leq \hat T +  cT^{3/4}\sqrt{\ln T}\ln T.
\end{equation*}
\end{theorem}

\section{Statement of results}

Our first main result is that for contextual bandits with adversarial rewards, for generic metric spaces $\Xcal$---that admit a non-atomic probability measure, e.g., any uncountable Polish space---there never exists an optimistically universal learning rule. On the other hand, if $\Xcal$ does not admit a non-atomic  probability measure, optimistic learning is possible.

\begin{theorem}\label{thm:no_opt_learning rule}
    Let $\Xcal$ be a separable metrizable Borel space.
    \begin{enumerate}
        \item Let $\Acal$ be a finite action space with $|\Acal|\geq 2$.
        \begin{itemize}
            \item If $\Xcal$ admits a non-atomic probability measure, there does not exist an optimistically universal learning rule for any adversarial reward model considered in \cref{def:reward_model} (i.e., all except stationary).
            \item Otherwise, there exists an optimistically universal learning rule for all reward models from \cref{def:reward_model} and $\Ccal_{online}=\Ccal_{stat}=\Ccal_2$.
        \end{itemize}
        \item Let $\Acal$ be a countably infinite action space, there exists an optimistically universal learning rule for all reward models from \cref{def:reward_model} and $\Ccal_{online}=\Ccal_{stat}=\Ccal_1$.
        \item Let $\Acal$ be an uncountable separable metrizable Borel space, then universal learning is never achievable and $\Ccal_{online}=\Ccal_{stat}=\emptyset$.
    \end{enumerate}
\end{theorem}

The question of whether optimistic learning is possible for finite action spaces is answered in \cref{sec:no_optim_learning_rule}. The case of infinite action spaces is treated in \cref{subsec:infinite_action_spaces}. Thus, \cref{thm:no_opt_learning rule} is a concatenation of \cref{thm:no_optimistically_universal_lr,thm:bad_borel_spaces} and \cref{subsec:infinite_action_spaces}.

The fact that optimistic learning is impossible the main case of finite action space and spaces $\Xcal$ admitting a non-atomic probability measure comes in stark contrast with all learning frameworks that have been studied in the universal learning literature. Namely, for the noiseless full-feedback \cite{hanneke:21,blanchard:22a}, noisy/adversarial full-feedback \cite{blanchard:22d} and stationary partial-feedback \cite{blanchard:22e} learning frameworks, analysis showed that there always existed an optimistically universal learning rule. Precisely, the optimistically universal learning rule for stationary contextual bandits in finite action spaces provided by \cite{blanchard:22e} combined two strategies:
\begin{itemize}
    \item A strategy 0, which treats each distinct context completely separately by assigning a distinct bandit subroutine to each new instance. Informally, this corresponds to learning the optimal action for each new context without gathering population information.
    \item A strategy 1, in which the learning rule views context in an aggregate fashion: it tries to fit the policy which performed best on the complete historical data using learning-with-experts subroutines, from a set of pre-defined policies.
\end{itemize}
The procedure to combine these strategies estimates their performance, to implement the best strategy during pre-defined periods. We show that for adversarial rewards, balancing these two strategies is impossible. In particular, an adversarial reward mechanism can fool the estimation procedure by changing behavior between the estimation period and the implementation period.

The non-existence of an optimistically universal learning rule also provides another proof that model selection is impossible for contextual bandits. A formulation of this question was posed as a COLT 2020 open problem \cite{foster:20}. The impossibility of model selection was then recently proved first with a switching bandit problem \cite{marinov:21}. Our results show this general impossibility in a completely different context. More precisely, \cref{prop:can_fixed_error} below shows that universal consistency up to a fixed error tolerance $\epsilon>0$ is always achievable under $\Ccal_2$ processes (which were necessary for universal learning even in the stationary case \cite{blanchard:22e}). However, \cref{thm:no_opt_learning rule} implies that combining these learning rules for decaying $\epsilon$ to achieve vanishing excess error is not possible in general.

\begin{restatable}{proposition}{PropFixedError}
\label{prop:can_fixed_error}
    Let $\Xcal$ be a separable metrizable Borel space and $\Acal$ a finite action space. For any $\epsilon>0$, there exists a learning rule $f^\epsilon_\cdot$ such that for any process $\Xbb\in\Ccal_2$ and adversarial reward mechanism $(r_t)_{t\geq 1}$, for any measurable policy $\pi^*:\Xcal\to\Acal$,
    \begin{equation*}
        \limsup_{T\to\infty} \frac{1}{T}\sum_{t=1}^T r_t(\pi^*(X_t)) - r_t(\hat a_t(\epsilon)) \leq \epsilon,\quad (a.s.),
    \end{equation*}
    where $\hat a_t(\epsilon)$ denotes the action selected by the learning rule at time $t$.
\end{restatable}

\noindent The proof is given in \cref{subsec:fixed_excess_error}. \cref{thm:no_opt_learning rule} provides the characterizations of universally learnable processes in all cases except the main case of interest when $\Acal$ is finite and $\Xcal$ admits a non-atomic probability measure. Giving exact characterizations for this case is rather complex and in the following, we only give necessary conditions and sufficient conditions. These require the introduction of novel classes of stochastic processes for online learning.

\subsection{Additional classes of stochastic processes}

We first give a significantly stronger assumption asking that the process only visits a finite number of distinct points. This very restrictive condition will only arise for unbounded rewards $\Rcal=[0,\infty)$.

\begin{condition}[\citet{hanneke:21,blanchard:22b}]
  \label{con:fs}
  $|\{ x : \ProcX \cap \{x\} \neq \emptyset \}| < \infty \text{ (a.s.)}.$ Denote by $\FS$ the set of all processes 
  $\ProcX$ satisfying this condition.
\end{condition}

We then introduce two novel conditions on stochastic processes. Before doing so, we need to introduce some exponential time scales. Intuitively, for $\alpha>0$, the exponential time scale at rate $\alpha$ is the sequence of times given by $T^k(\alpha) \approx \lfloor(1+\alpha)^k\rfloor$ for $k\geq 0$. For convenience, we will instead consider for all integers $i\geq 0$ the sequence of times $T^k_i = \lfloor 2^u (1+v2^{-i})  \rfloor$ where $k=u2^i + v$ and $u\geq 0, 0\leq v<2^i$ are integers. In particular, $u=\floor{k2^{-i}}$ and $v=k\bmod 2^i$. These times have an exponential behavior with rate oscillating between $2^{-i-1}$ and $2^{-i}$ but conveniently, they form periods $[T_i^k,T_i^{k+1})$ which become finer as $i$ increases. For $t\geq 1$, we then define $k_i(t)$ as the index $k$ such that $t\in [T_i^k,T_i^{k+1})$. This allows to consider the set of times $t$ such that $X_t$ is the first appearance of the instance on its period,
\begin{equation*}
    \Tcal^i = \{t\geq 1: \forall T_i^{k_i(t)}\leq t'<t, X_{t'}\neq X_t\}.
\end{equation*}
By construction, note that $\Tcal^i\subset\Tcal^{i+1}$ for all $i\geq 0$. We are now ready to define the next condition which intuitively asks that the process has a $\Ccal_1'$ behavior uniformly at any exponential scale.

\begin{restatable}{condition}{ConditionCsByScale}
\label{con:cs_by_scale}
For any sequence of disjoint measurable sets $(A_i)_{i\geq 1}$ of $\Xcal$, we have
\begin{equation*}
    \lim_{i\to\infty} \Ebb\left[\limsup_{T\to\infty}\frac{1}{T} \sum_{t\leq T, t\in\Tcal^i} \1_{A_i}(X_t)\right] = 0.
\end{equation*}
Denote by $\Ccal_4$ the set of all processes $\Xbb$ satisfying this condition.
\end{restatable}

Then, we define the next condition which asks that there exists a rate to include decreasing exponential scales while conserving the $\Ccal_1'$ property.

\begin{restatable}{condition}{ConditionExistsScaleRate}
\label{con:exists_scale_rate}
There exists an increasing sequence of integers $(T_i)_{i\geq 0}$ such that letting
\begin{equation*}
    \Tcal = \bigcup_{i\geq 0} \Tcal^i\cap \{t\geq T_i\},
\end{equation*}
we have $\tilde \Xbb = (X_t)_{t\in\Tcal}\in\Ccal_1'$. Denote by $\Ccal_5$ the set of all processes $\Xbb$ satisfying this condition.
\end{restatable}

We now introduce two new conditions on stochastic processes which we will show are necessary for some of the considered reward models. These build upon the definition of $\Ccal_4$ processes. Before introducing them, we need to analyze large deviations of the empirical measure in $\Ccal_1'$ processes. The next lemma intuitively shows that for a process $\tilde \Xbb\in\Ccal_1'$, for large enough time steps, one can bound the deviations of the empirical measure of a set $A\in\Bcal$ compared to the limit sub-measure $\hat\mu_\Xbb(A)$ uniformly in the set $A$.

\begin{lemma}\label{lemma:uniform_deviations}
Let $\Xbb$ be a stochastic process on $\Xcal$ and $\Tcal$ some random times such that $\tilde \Xbb=(X_t)_{t\in\Tcal} \in\Ccal_1'$. Then, for any $\epsilon>0$, there exists $T_\epsilon\geq 1$ and $\delta>0$ such that for any measurable set $A\in\Bcal$,
\begin{equation*}
    \Ebb[\hat\mu_{\tilde \Xbb}(A)]\leq \delta \Longrightarrow \Ebb\left[\sup_{T\geq T_\epsilon}\frac{1}{T}\sum_{t\leq T, t\in\Tcal}\1_A(X_t)\right]\leq \epsilon.
\end{equation*}
\end{lemma}

Now consider a process $\Xbb\in\Ccal_4$. For any integer $p\geq 0$, the definition of $\Ccal_4$ implies $\Xbb^p := (X_t)_{t\in\Tcal^p}\in\Ccal_1'$. Indeed, the sets $\Tcal^i$ are increasing in $i\geq 0$, hence for $i\geq p$ one has $\Tcal^p\subset \Tcal^i$. As a result, \cref{con:cs_by_scale} implies that for any disjoint measurable sets $(A_i)_{i\geq 1}$, one has $\Ebb[\hat\mu_{\Xbb^p}(A_i)]=\Ebb[\limsup_{T\to\infty} \sum_{t\leq T, t\in\Tcal^p}\1_{A_i}(X_t)] \to 0$ as $i\to\infty$. Now for any $\epsilon>0$ and $T\geq 1$, we define
\begin{multline*}
    \delta^p(\epsilon;T) := \sup\left\{0\leq \delta\leq 1:\forall A\in\Bcal \text{ s.t. }
    \sup_l \Ebb[\hat\mu_{\Xbb^l}(A)]\leq \delta,\right.\\
    \left. \forall \tau\geq T \text{ online stopping time},\quad
     \Ebb\left[\frac{1}{2\tau}\sum_{\tau\leq t<2\tau , t\in\Tcal^p}\1_A(X_t)\right]\leq \epsilon \right\},
\end{multline*}
where the $\tau$ is a stopping time with respect to the filtration generated by the instance process $\Xbb$. In particular, $\tau$ can be seen as an online procedure which decides when to count the number of instances of $\Xbb^p$ falling in the considered set $A$. Note that $\delta^p(\epsilon;T)$ satisfies the property that for all measurable set $A$ satisfying $\sup_l \Ebb[\hat\mu_{\Xbb^l}(A)]\leq \delta^p(\epsilon;T)$ and any stopping time $\tau\geq T$,
\begin{equation*}
     \Ebb\left[\frac{1}{2\tau}\sum_{\tau\leq t<2\tau, t\in\Tcal^p}\1_A(X_t)\right]\leq \epsilon,
\end{equation*}
which can be checked for all sets $A\in\Bcal$ separately. Next, the quantity $\delta^p(\epsilon;T)$ is non-decreasing in $T$. Further, as a direct application of \cref{lemma:uniform_deviations}, because $\Xbb^p\in\Ccal_1'$, there exists $T^p(\epsilon)\geq 1$ and $\delta>0$ such that for $T\geq T^p(\epsilon)$, we have $\delta^p(\epsilon;T)\geq\delta$. As a result, we have $\delta^p(\epsilon) := \lim_{T\to\infty}\delta^p(\epsilon;T)\geq \delta>0$. Also, the quantity $\delta^p(\epsilon;T)$ is non-increasing in $p$ since the sets $\Tcal^p$ are non-decreasing with $p$. Thus, $\delta^p(\epsilon)$ is also non-increasing in $p$. We are now ready to introduce the condition on stochastic processes based on the limit of the quantities $ \delta^p(\epsilon)$.

\begin{restatable}{condition}{ConditionSix}
\label{con:C6}
    $\Xbb\in\Ccal_4$ and for any $\epsilon>0$, we have $ \lim_{p\to\infty} \delta^p(\epsilon) >0.$ Denote by $\Ccal_6$ the set of all processes $\Xbb$ satisfying this condition.
\end{restatable}

Intuitively, this asks that the maximum deviations are also bounded in $p$, hence $\Ccal_6$ processes have more regularity than general $\Ccal_4$ processes. However, the maximum deviations are limited by the fact that they should be discernible through an online stopping time $\tau$.

The following inclusions hold $\Ccal_3\subset \Ccal_1\subset\Ccal_5  \subset \Ccal_6\subset \Ccal_4\subset\Ccal_2.$ Indeed, the inclusion $\Ccal_3\subset\Ccal_1$ is known \cite{hanneke:21}. $\Ccal_1\subset \Ccal_5$ and $\Ccal_6\subset\Ccal_4$ are immediate from the definition of \cref{con:exists_scale_rate} and \cref{con:C6} respectively. The inclusion $\Ccal_4\subset \Ccal_2$ is shown in \cref{prop:condition_5_stronger_2}. Last, the fact that for oblivious rewards, $\Ccal_6$ is necessary (\cref{thm:condition7_necessary}) and $\Ccal_5$ is sufficient (\cref{thm:C6_learnable}) shows that $\Ccal_5\subset\Ccal_6$.

\begin{table}
\label{table:summary_of_results}
\centering
\resizebox{\textwidth}{!}{
\begin{tabular}{|l|p{2cm}|c|c|c|} 
 \hline
 \multirow{2}{*}{\,\,\textbf{Learning setting}} & \multicolumn{2}{c|}{$\begin{array}{c}
    \textbf{Stationary} \\
     \textbf{contextual bandits \cite{blanchard:22e}}
 \end{array}$ }&   \multicolumn{2}{c|}{$\begin{array}{c}
    \textbf{Contextual bandits with} \\
     \textbf{adversarial rewards [This paper]}
 \end{array}$ }\\ 

\cline{2-5}

 & \centering$\Ccal_{stat}$ & OL? & $\begin{array}{c}
    \text{Necessary and sufficient} \\
     \text{conditions on } \Ccal
 \end{array}$   & OL?\\

 \hline
\hline

$\begin{array}{l}
    \text{Finite $\Acal$, $|\Acal|\geq 2$, $\Xcal$ with} \\
     \text{non-atomic proba. measure}
 \end{array}$
   & \centering$\Ccal_2$ & Yes &
   $\begin{array}{c}
    \Ccal_1\subsetneq \Ccal_5\subset \Ccal\subsetneq \Ccal_2 \\
    \Ccal_5=\Ccal_{online} \subset \Ccal_{oblivious}\subset \Ccal_6   \comment{\\
    \Ccal_{prescient}\subset \Ccal_7}
 \end{array}$ & No \\
\hline
$\begin{array}{l}
    \text{Finite $\Acal$, $|\Acal|\geq 2$, $\Xcal$ without} \\
     \text{non-atomic proba. measure}
 \end{array}$
 &\centering$\Ccal_2$ & Yes & $\Ccal=\Ccal_2$ & Yes\\
\hline
\,\,Countably infinite $\Acal$ &\centering $\Ccal_1$& Yes &$\Ccal =\Ccal_1$& Yes\\
\hline
\,\,Uncountable $\Acal$ &\centering $\emptyset$& N/A &$\Ccal=\emptyset$& N/A\\
 
 \hline
 
\end{tabular}
}
\caption{Characterization of learnable processes for universal learning in contextual bandits, depending on the action space $\Acal$, context space $\Xcal$ and reward model. When the model is not specified, $\Ccal$ refers to any of the considered models. OL? = Is optimistic learning possible?}

\end{table}

\subsection{Necessary and sufficient conditions for universal learning}

Our second main contribution is giving necessary and sufficient conditions for universal learning with adversarial rewards. In addition to characterizations from \cref{thm:no_opt_learning rule}, we have the following.

\begin{theorem}\label{thm:main_characterizations}
    Let $\Xcal$ be a separable metrizable Borel space admitting a non-atomic probability measure and $\Acal$ a finite action space with $|\Acal|\geq 2$. Then $\Ccal_1\subsetneq \Ccal_5= \Ccal_{online}\subset \Ccal_{oblivious}\subset \Ccal_{memoryless}\subsetneq \Ccal_2$. Further, $\Ccal_{oblivious} \subset \Ccal_6\subsetneq\Ccal_2$.\comment{ and $ \Ccal_{prescient}\subset \Ccal_7$.}
\end{theorem}

\noindent These results are proved in \cref{sec:towards_characterization_learnable_processes}. The fact that $\Ccal_{memoryless}\subsetneq \Ccal_2$ is proved in \cref{thm:learnable_process_smaller_C2}. $\Ccal_{oblivious} \subset \Ccal_6$ is proved in \cref{thm:condition7_necessary} while $\Ccal_6\subsetneq\Ccal_2$ comes from \cref{thm:learnable_process_smaller_C2} and the fact that $\Ccal_6\subset \Ccal_4$ (\cref{thm:other_example} further gives an example of processes in $\Ccal_4\setminus \Ccal_6$). $\Ccal_{online}\subset \Ccal_5$ is proved in \cref{thm:C5_necessary_online} and $\Ccal_1\subsetneq \Ccal_5\subset \Ccal_{online}$ is proved in \cref{thm:C6_learnable} and \cref{prop:larger_class}. Here is the overview of relations we show between the classes of processes: for $\Xcal$ admitting non-atomic probability measures, $\Ccal_1\subsetneq\Ccal_5\subset \Ccal_6\subsetneq \Ccal_4\subsetneq \Ccal_2$.

In particular, our characterization is complete for the strongest online rewards, unlike for memoryless and oblivious rewards. We believe that $\Ccal_5\subsetneq\Ccal_6$ in general. In fact, the proof of \cref{thm:condition7_necessary} for the necessity of $\Ccal_6$ for oblivious rewards can be tightened given a stronger reward model in which the reward adversary can additionally take into account the \emph{complete} sequence $\Xbb$---instead of the revealed contexts to the learner $\Xbb_{\leq t}$. We refer to this reward model as \emph{prescient} rewards (see \cref{def:prescient_rewards} for a formal definition) and show that in this case, a stronger $\Ccal_7$ condition is necessary (\cref{thm:prescient_implies_C7}). We leave open the question of whether $\Ccal_5=\Ccal_7$. If this were true, then we also have an exact characterization for prescient rewards.

    

Our findings are summarized in Table~{\color{blue}1}, which also compares learnable processes for stationary and adversarial contextual bandits. We leave open the exact characterization of learnable processes for memoryless and oblivious rewards in finite action spaces $\Acal$ and context spaces admitting a non-atomic probability measure.

\paragraph{Open question:}
\textit{Let $\Xcal$ be a separable metrizable Borel space admitting a non-atomic probability measure and $\Acal$ a finite action space with $|\Acal|\geq 2$. What is an exact characterization of $\Ccal_{memoryless}$ or $\Ccal_{oblivious}$?}\\

Finally, we also give results in a setting where we assume that rewards are unbounded. We answer the same questions: what are the learnable processes for which universal learning is possible, and can we obtain optimistically universal learning rules? We use a subscript $\Ccal^{unbounded}$ to specify that we consider the case of unbounded rewards. We show that in that case, results are identical to the case of stationary contextual bandits.

\begin{proposition}
    Let $\Xcal$ be a separable metrizable Borel space. For all reward models,
    \begin{itemize}
        \item if $\Acal$ is uncountable, $\Ccal^{unbounded}=\Ccal_3$ for all reward models. Further, there is an optimistically universal learning rule,
        \item if $\Acal$ is uncountable, universal learning for unbounded rewards is never achievable.
    \end{itemize}
\end{proposition}
Last, we extend our results to rewards with additional regularity assumptions. For a given metric $d$ on $\Acal$, we suppose that they are uniformly-continuous, generalizing a notion introduced in \cite{blanchard:22e}.

\begin{restatable}{definitionbis}{DefinitionUniformlyContinuousRewards}
    Let $(\Acal,d)$ be a separable metric space. The reward mechanism $(r_t)_{t\geq 1}$ is uniformly-continuous if for any $\epsilon>0$, there exists $\Delta(\epsilon)>0$ such that
\begin{multline*}
    \forall t\geq 1, \forall (\mb x_{\leq t},\mb a_{\leq t-1},\mb r_{\leq t-1}) \in\Xcal^t\times\Acal^{t-1}\times \Rcal^{t-1},\forall a,a'\in \Acal,\\
    \quad d(a,a') \leq \Delta(\epsilon)\Rightarrow \left|\Ebb[r_t(a)-r_t(a') \mid \Xbb_{\leq t} = \mb x_{\leq t}, \mb a_{\leq t-1},\mb r_{\leq t-1}]\right|\leq \epsilon,
\end{multline*}
\end{restatable}
\comment{
\noindent For prescient rewards, the above definition should be adapted to any conditioning on the complete sequence $\Xbb$, which the reward mechanism has access to. Details are not included for conciseness.
}

For uniformly-continuous rewards we use a reduction to the case of rewards without regularity assumptions, which we refer to as \emph{unrestricted} rewards. Then, we recover the same results for uniformly-continuous rewards, in totally-bounded (resp. non-totally-bounded) action spaces as for unrestricted rewards in finite (resp. countably infinite) action spaces. We adopt the subscript $\Ccal^{uc}$ to emphasize that we consider uniformly-continuous rewards.

\begin{theorem}
    Let $\Xcal$ be a metrizable Borel space and $model\in\{memoryless,oblivious\comment{,prescient},online\}$.
    \begin{itemize}
        \item If $\Acal$ is a totally-bounded metric space, all properties for $\Ccal_{model}$ for finite action spaces described in \cref{thm:main_characterizations} hold for $\Ccal_{model}^{uc}$. Further, there is an optimistically universal learning rule for uniformly-continuous rewards if and only if there is one for finite action spaces for unrestricted rewards as in \cref{thm:no_opt_learning rule}.
        \item If $\Acal$ is a non-totally-bounded metric space, all properties for $\Ccal_{model}$ for countable action spaces described in \cref{thm:main_characterizations} hold for $\Ccal_{model}^{uc}$. Further, there is always an optimistically universal learning rule for uniformly-continuous rewards.
    \end{itemize}
\end{theorem}

This result is proved in \cref{subsec:uniformly_continuous_rewards} and is a concatenation of \cref{prop:simple_upper_bounds} for necessary conditions and \cref{thm:uc_expinf} and \cref{thm:lower_bound_C5} for sufficient conditions for universal learning.

\section{Existence or non-existence of an optimistically universal learning rule}

\label{sec:no_optim_learning_rule}

In this section, we ask the question of whether there exists an optimistically universal learning rule for finite action spaces. In fact, in all the frameworks considered for universal learning---noiseless \cite{blanchard:22a} or noisy/adversarial responses \cite{blanchard:22d} in the full-feedback setting and stationary partial-feedback responses \cite{blanchard:22e}---analysis showed that optimistically universal learning always existed. However, the learning rule provided by \cite{blanchard:22e} for stationary rewards under $\Ccal_2$ processes heavily relies on the assumption that the rewards are stationary in order to make good estimates of the performance of different learning strategies. In particular, one can easily check that this learning rule would not be universally consistent under adversarial rewards even in the weakest memoryless setting. Instead, we will show that for contextual bandits with adversarial rewards, in general there does not exist optimistically universal learning rules.

To do so, we first need to argue that the set of learnable processes even in the online setting $\Ccal_{online}$ contains a reasonably large class of processes. We first show that using the $\EXPIX$ algorithm for adversarial bandits \cite{neu2015explore} as subroutine yields a universally consistent learning rule for processes $\Xbb$ which visit a sublinear number of distinct instances.

\begin{proposition}\label{prop:EXP.IX_parrallel}
Let $\Xcal$ be a metrizable separable Borel space and $\Acal$ a finite action space. There exists a learning rule which is universally consistent for online rewards under any process $\Xbb$ satisfying $|\{x\in \Xcal: \{x\}\cap \Xbb_{\leq T}\neq\emptyset\}|=o(T) \quad (a.s.).$
\end{proposition}

\begin{proof}
Consider the learning rule $f_\cdot$ which simply performs independent copies of the $\EXPIX$ algorithm in parallel such that to each distinct instance visited is assigned a $\EXPIX$. More precisely, for any $t\geq 1$, instances $\mb x_{\leq t}$ and observed rewards $\mb r_{\leq t-1}$, we define
\begin{equation*}
    f_t(\mb x_{\leq t-1}, \mb r_{\leq t-1}, x_t) = \EXPIX(\mb {\hat a}_{S_t}, \mb r_{S_t}),
\end{equation*}
where $S_t = \{t'<t:x_{t'} = x_t\}$ is the set of times that $x_t$ was visited previously and $\hat a_{t'}$ denotes the action selected at time $t'$ for $t'<t$. We now show that this learning rule is universally consistent on any process $\Xbb$ which visits a sublinear number of distinct instances almost surely. For simplicity we denote $\hat a_t$ the action selected by $f_\cdot$ at time $t$. Let $\Xbb$ such that almost surely, $\frac{1}{T} |\{x\in \Xcal: \{x\}\cap \Xbb_{\leq T}\neq\emptyset\}|\to 0$. Denote by $\Ecal$ this event, and for any $T\geq 1$ we define $\epsilon(T) = \frac{1}{T} |\{x\in \Xcal: \{x\}\cap \Xbb_{\leq T}\neq\emptyset\}|$ and $S_T = \{x\in \Xcal: \{x\}\cap \Xbb_{\leq T}\neq\emptyset\}$, hence $|S_T| = T\epsilon(T)$. Further, for any $x\in S_T$ we pose $\Tcal_T(x) = \{t\leq T:X_t = x\}$. Let $\Hcal_0(T) = \{x\in S_T: |\Tcal_T(x)|<\frac{1}{\sqrt {\epsilon(T)}}\}$, $\Hcal_1(T) = \{x\in S_T: \frac{1}{\sqrt {\epsilon(T)}}\leq|\Tcal_T(x)|<\ln ^2 T\}$ and $\Hcal_2(T) = \{x\in S_T:|\Tcal_T(x)|\geq \ln ^2 T\}$, so that $S_T=\Hcal_0(T)\cup\Hcal_1(T)\cup\Hcal_2(T)$. Note that
\begin{equation*}
    \sum_{x\in \Hcal_0(T)}\sum_{t\in\Tcal_T(x)}  r_t(\pi(X_t)) - r_t(\hat a_t) \leq \frac{|\Hcal_0(T)|}{\sqrt{\epsilon(T)}}\leq \sqrt{\epsilon(T)}T.
\end{equation*}
Now fix a measurable policy $\pi:\Xcal\to\Acal$. Then,
\begin{equation*}
     \sum_{x\in \Hcal_2(T)}\sum_{t\in\Tcal_T(x)}  r_t(\pi(X_t)) - r_t(\hat a_t) \leq  \sum_{x\in \Hcal_2(T)} \max_{a\in \Acal} \sum_{t\in\Tcal_T(x)} (r_t(a) - r_t(\hat a_t)).
\end{equation*}
 Now recall that for any $x\in S_T$, on $\Tcal_T(x)$ the algorithm $\EXPIX$ was performed. As a result, by \cref{thm:multiarmed_bandits}, conditionally on the realization $\Xbb$, for any $x\in \Hcal_2(T)$, with probability $1-\frac{1}{T^3}$, conditionally on $\Xbb$,
\begin{equation*}
    \max_{a\in \Acal} \sum_{t\in \Tcal_T(x)} (r_t(a) - r_t(\hat a_t)) \leq 3c \sqrt{|\Acal| |\Tcal_T(x)| \ln |\Acal|}\ln T 
    \leq |\Tcal_T(x)|\cdot 3c \frac{\sqrt{|\Acal| \ln |\Acal|}}{\ln T}.
\end{equation*}
Noting that $|\Hcal_2(T)|\leq T$, we obtain by the union bound that (conditionally on $\Xbb$) with with probability $1-\frac{1}{T^2}$,
\begin{equation*}
    \sum_{x\in \Hcal_2(T)} \max_{a\in \Acal} \sum_{t\in \Tcal_T(x)} (r_t(a) - r_t(\hat a_t)) \leq 3c \frac{\sqrt{|\Acal| \ln |\Acal|}}{\ln T} \sum_{x\in\Hcal_2(T)}|\Tcal_T(x)|\leq 3c \sqrt{|\Acal| \ln |\Acal|} \frac{T}{\ln T}. 
\end{equation*}
We denote by $\Fcal_T$ the event when the above equation holds. We have $\Pbb[\Fcal_T]\geq 1-\frac{1}{T^2}$ where the probability is also taken over $\Xbb$. We now turn to points in $\Hcal_1(T)$ for which we need to go back to the proof of \cref{thm:multiarmed_bandits} from \cite{neu2015explore}. Taking the same notations as in the original proof, for $u\geq 1$, let $\eta_u=2\gamma_u=\sqrt{\frac{\ln |\Acal|}{|\Acal|u}}$, and for any $t\geq 1$, $a\in\Acal$ denote by $p_{t,a}$ the probability that the learning rule selects action $a$ at time $t$, and let $\ell_{t,a}=1-r_t(a)$. Next, let $u(t) = |\{s\leq t:X_s=X_t\}|$ and pose $\tilde \ell_{t,a} = \frac{1-r_t(a)}{p_{t,a}+\gamma_u}\1[\hat a_t=a]$. Using the derivations of the proof of \cref{thm:multiarmed_bandits}, for any $x\in S_T$, writing $\Tcal_T(x)=\{t_1(x),\ldots,t_{|\Tcal_T(x)|}\},$ for any $a'\in\Acal$,
\begin{equation*}
    \sum_{u=1}^{|\Tcal_T(x)|}\left(\ell_{t_u,\hat a} - \tilde \ell_{t_u,a'}\right) \leq \frac{\ln |\Acal|}{\eta_{|\Tcal_T(x)|}} + \sum_{u=1}^{|\Tcal_T(x)|}\eta_u\sum_{a\in\Acal}\tilde\ell_{t_u,a}.
\end{equation*}
Summing these equations with $a'=\pi(x)$, we obtain
\begin{equation*}
    \sum_{x\in\Hcal_1(T)}\sum_{t\in\Tcal_T(x)} (1-\tilde\ell_{t,\pi(X_t)}) - r_t(\hat a_t) \leq \sum_{x\in\Hcal_1(T)}\sqrt{|\Acal|\ln|\Acal| |\Tcal_T(x)|} + \sum_{x\in\Hcal_1(T)}\sum_{t\in\Tcal_T(x)}\eta_{u(t)} \sum_{a\in\Acal}\tilde \ell_{t,a}.
\end{equation*}
Now let for any $a\in\Acal$, conditionally on $\Xbb$, the sequence $(\sum_{x\in\Hcal_1(T')}\sum_{t\in\Tcal_{T'}(x)}\eta_{u(t)} (\tilde\ell_{t,a}-\ell_{t,a}))_{T'\leq T}$ is a super-martingale (the immediate expected value of $\tilde \ell_{t,a}$ is $\frac{p_{u(t)}}{p_{u(t)}+\gamma_{u(t)}}\ell_{t,a}$) and each increment is upper-bounded by 2 in absolute value: $0\leq \eta_{u(t)}\tilde\ell_{t,a}\leq \eta_{u(t)}\frac{\ell_{t,a}}{p_{u(t),a}+\gamma_{u(t)}}\leq \frac{\eta_{u(t)}}{\gamma_{u(t)}}\leq 2$. Therefore, Azuma's inequality implies
\begin{equation*}
    \Pbb\left[ \sum_{x\in\Hcal_1(T)}\sum_{t\in\Tcal_T(x)}\eta_{u(t)} \sum_{a\in\Acal} (\tilde\ell_{t,a}-\ell_{t,a})\leq 4T^{3/4} \mid \Xbb\right] \geq 1-e^{-2\sqrt T}.
\end{equation*}
Similarly, because $0\leq\tilde \ell_{t,a}\leq \frac{1}{\gamma_{u(t)}}=2\sqrt{\frac{|\Acal|u(t)}{\ln |\Acal|}}$, we have
\begin{equation*}
    \Pbb\left[ \sum_{x\in\Hcal_1(T)}\sum_{t\in\Tcal_T(x)} \sum_{a\in\Acal}  (\tilde\ell_{t,\pi(X_t)}-\ell_{t,\pi(X_t)})\leq 4\sqrt{\frac{|\Acal|}{\ln |\Acal|}}T^{3/4}\ln T \mid \Xbb\right] \geq 1-e^{-2\sqrt T}.
\end{equation*}
As a result, on an event $\Gcal_T$ of probability at least $1-(1+|\Acal|)e^{-2\sqrt T}$, we have
\begin{align*}
    \sum_{x\in\Hcal_1(T)}\sum_{t\in\Tcal_T(x)} r_t(\pi(X_t)) - r_t(\hat a_t) &\leq \sum_{x\in\Hcal_1(T)}\sqrt{|\Acal|\ln|\Acal| |\Tcal_T(x)|} + \sum_{x\in\Hcal_1(T)}\sum_{t\in\Tcal_T(x)}\eta_{u(t)} \sum_{a\in\Acal} \ell_{t,a}\\
    &\quad\quad\quad\quad+ 4\sqrt{\frac{|\Acal|}{\ln|\Acal|}} T^{3/4}\ln T + 4T^{3/4}\\
    &\leq \sum_{x\in\Hcal_1(T)}\sqrt{|\Acal|\ln|\Acal| |\Tcal_T(x)|} + \sum_{x\in\Hcal_1(T)}|\Acal|\sum_{t\in\Tcal_T(x)}\eta_{u(t)} \\
    &\quad\quad\quad\quad+ 4\sqrt{\frac{|\Acal|}{\ln|\Acal|}} T^{3/4}\ln T + 4T^{3/4}\\
    &\leq \sum_{x\in\Hcal_1(T)}3\sqrt{|\Acal|\ln|\Acal| |\Tcal_T(x)|} + 8\sqrt{|\Acal|} T^{3/4}\ln T\\
    &\leq 3\sqrt{|\Acal|\ln|\Acal|} \epsilon(T)^{1/4}T + 8\sqrt{|\Acal|} T^{3/4}\ln T.
\end{align*}
Combining all our estimates, we showed that on $\Fcal_T\cap\Gcal_T$,
\begin{equation*}
    \sum_{t\leq T} r_t(\pi(X_t)) - r_t(\hat a_t) \leq  8|\Acal| T^{3/4}\ln T + 3c \sqrt{|\Acal| \ln |\Acal|} \frac{T}{\ln T} + (\sqrt{\epsilon(T)}+3\sqrt{|\Acal|\ln|\Acal| } \epsilon(T)^{1/4})T 
\end{equation*}
Now note that $\sum_{T\geq 1}\Pbb[\Fcal_T^c]+\Pbb[\Gcal_T^c]<\infty$. Hence, the Borel-Cantelli lemma implies that on an event $\Acal$ of probability one, there exists $\hat T\geq 1$ such that for any $T\geq \hat T$, the event $\Fcal_T\cap\Gcal_T$ is satisfied. As a result, on the event $\Ecal\cap\Acal$, since $\epsilon(T)\to 0$, we obtain
\begin{equation*}
    \limsup_{T\to\infty} \frac{1}{T}\sum_{t=1}^T r_t(\pi(X_t)) - r_t(\hat a_t) \leq 0.
\end{equation*}
By union bound, $\Ecal\cap\Acal$ has probability one, hence we proved that the learning rule $f_\cdot$ is universally consistent on $\Xbb$. This ends the proof of the proposition.
\end{proof}

As a simple consequence of \cref{prop:EXP.IX_parrallel}, deterministic $\Ccal_2$ processes are always universally learnable even in the online rewards setting.

\begin{proposition}
\label{prop:deterministic_C2}
Let $\Xcal$ be a metrizable separable Borel space and $\Acal$ a finite action space. There exists a learning rule which is universally consistent for any deterministic process $\Xbb\in\Ccal_2$ under online rewards.
\end{proposition}

\begin{proof}
We first show that any deterministic process $\Xbb\in\Ccal_2$ visits a sublinear number of distinct instances almost surely. Denote $S_T = \{X_t:t\leq T\}$ the set of visited instances until time $T$ and let $S = \bigcup_{T\to\infty} S_T$. Then, $\{x\}_{x\in S}$ forms a countable sequence of disjoint sets. Hence, by the $\Ccal_2$ property and because $\Xbb$ is deterministic, we have that
\begin{equation*}
    |\{x:\{x\}\cap\Xbb_{\leq T}\neq \emptyset\}| = |S_t| = |\{x\in S:\{x\}\cap\Xbb_{\leq T} \neq\emptyset \}| = o(T),\quad (a.s.).
\end{equation*}
Hence, by \cref{prop:EXP.IX_parrallel}, the learning rule which performs $\EXPIX$ independently for each distinct visited instance is universally consistent under $\Xbb$. This ends the proof of the proposition.
\end{proof}

Next, we argue that $\Ccal_1$ processes are also universally learnable in the online rewards setting. In the case of countable action sets $\Acal$, \cite{blanchard:22e} gave a universally consistent learning rule $\EXPINF$ under $\Ccal_1$ processes using \cref{thm:infinite-exp4}.  Precisely, the learning rule uses a result from \cite{hanneke:21} showing that there exists a countable set of policies $\Pi=\{\pi^i:\Xcal\to\Acal, i\geq 1\}$ that is empirically dense within measurable policies under any $\Ccal_1$ process. As a result, to yield a universally consistent learning rule under $\Ccal_1$ processes, it suffices to have a learning rule with sublinear regret compared to any policy $\pi\in\Pi$. The algorithm $\EXPINF$ achieves this property using restarted $\EXPIX$ subroutines with slowly increasing finite set of experts from the sequence $\Pi$. Because the subroutines $\EXPIX$ have guarantees in the adversarial bandit framework, $\EXPINF$ directly inherits this guarantee and is a result universally consistent under $\Ccal_1$ processes for online rewards. Thus, $\Ccal_1\subset\Ccal_{online}$.

We are now ready to show that for spaces $\Xcal$ on which there exists a non-atomic probability measure on the space $\Xcal$, there does not exist any optimistically universally consistent learning rule. Precisely, we show that there is no learning rule that is universally consistent both on $\Ccal_1$ and deterministic $\Ccal_2$ processes.
Note that most context spaces $\Xcal$ of interest would admit a non-atomic probability measure, in particular any uncountable Polish space.

\begin{theorem}\label{thm:no_optimistically_universal_lr}
Let $\Xcal$ a metrizable separable Borel space such that there exists a non-atomic probability measure $\mu$ on $\Xcal$, i.e., such that $\mu(\{x\})=0$ for all $x\in\Xcal$. If $\Acal$ is a finite action space with $|\Acal|\geq 2$, then there does not exist an optimistically universal learning rule for memoryless rewards (a fortiori for oblivious, prescient or online rewards).
\end{theorem}

\begin{proof}
We fix $a_1,a_2\in\Acal$ two distinct actions. Suppose that there exists an optimistically universal learning rule $f_\cdot$. For simplicity, we will denote by $\hat a_t$ the action chosen by this learning rule at step $t$. We will construct a deterministic process $\Xbb \in \Ccal_2$ and rewards $r_t$ for which $f_\cdot$ does not achieve universal consistency.

We construct the process $\Xbb$ and rewards $(r_t)_{t\geq 1}$ recursively. Let $\epsilon_k = 2^{-k}$ for $k\geq 1$. The process and rewards are constructed together with times $T_k$ such that a significant regret is incurred to the learner between times $T_k$ and $T_{k+1}$ for all $k\geq 1$. We pose $T_0=0$. We are now ready to start the induction. Suppose that we have already defined $T_l$ for $l<k$ and the deterministic process $\Xbb_{\leq T_{k-1}}$ as well as the deterministic rewards $r_t$ for $t\geq T_{k-1}$. Let $\Zbb = (Z_i)_{i\geq 1}$ be an i.i.d. sequence on $\Xcal$ with distribution $\mu$. Pose $T^i = \frac{ (1+i)!}{\epsilon_k} T_{k-1} $ for $i\geq 0$ and $k_i = \epsilon_k T^i \left(= (1+i)! T_{k-1}\right)$, $n_i = \sum_{j< i} k_j$ for $i\geq 0$. Letting $\bar x\in \Xcal$ an arbitrary instance, we now consider the following process $\tilde \Xbb$:

\begin{equation*}
    \tilde X_t = \begin{cases}
    X_t,& t\leq T_{k-1},\\
    \bar x, & T_{k-1} < t < T^0, \\
    Z_{n_i+l},& t = T^i + p \cdot k_i + l,\quad 0\leq p < \frac{1}{\epsilon_k},\; 0\leq l < k_i,i\geq 0,\\
    \bar x, & 2T^i\leq t<T^{i+1},\quad i\geq 0.
    \end{cases}
\end{equation*}
The process is deterministic until time $T^0$. From this point, the process is constructed by periods, where period $i\geq 0$ corresponds to times $T^i\leq t<T^{i+1} = (1+i) T^i$. Each period $i$ has a first phase $T^i\leq t<2T^i$ composed of $\frac{1}{\epsilon_k}$ sub-phases of length $k_i=\epsilon_k T^i$ on which the process repeats exactly. We can therefore focus on the first sub-phase $T^i\leq t < T^i(1+\epsilon_k)$, which is constructed as an i.i.d. process following distribution $\mu$ independent from the past samples. In the second phase of period $i$ for $2T^i\leq T^{i+1}$ the process is idle equal to $\bar x$. This ends the construction of the process $\tilde \Xbb$.

We now argue that $\tilde \Xbb\in\Ccal_1$. Indeed, note that forgetting about the part for $t\leq T^0$, and idle phases where the process visits $\bar x$ only, this process takes values from an i.i.d. process $\Zbb$ and each value is duplicated $\frac{1}{\epsilon_k} $ times throughout the whole process. Formally, let $(A_p)_{p\geq 1}$ be a decreasing sequence of measurable sets with $A_p\downarrow \emptyset$. Then for any $T^i<T\leq T^{i+1}$ with $i\geq 1$ we have, for $p$ sufficiently large so that $\bar x\notin A_p$,
\begin{align*}
    \frac{1}{T}\sum_{t=1}^T\1_{A_p}(\tilde X_t) &\leq \frac{2T^{i-1}}{T^i} + \frac{1}{\epsilon_k T^i}  \sum_{l=n_i}^{n_i + k_i-1}\1_{A_p}(Z_l)\\
    &\leq \frac{2}{1+i} +  \frac{n_i+k_i}{k_i}  \frac{1}{n_i+k_i}\sum_{l=0}^{n_i + k_i-1}\1_{A_p}(Z_l).
\end{align*}
Last, we note that $\frac{n_i+k_i}{k_i}\to 1$ as $i\to\infty$. As a result, we obtain $\hat\mu_{\tilde \Xbb}(A_p)\leq \hat\mu_{\Zbb}(A_p)$. Because $\Zbb\in\Ccal_1$, we have $ \Ebb[\hat \mu_\Zbb(A_p)]\to 0$ as $p\to\infty$, which proves $\Ebb[\hat \mu_{\tilde \Xbb} (A_p)] \to 0$ as well. This ends the proof that $\tilde \Xbb\in\Ccal_1$.

We now construct rewards. Before doing so, for any $i\geq 0$, let $\delta_i$ such that
\begin{equation*}
    \Pbb\left[\min_{1\leq u<v<n_{i+1}}  \rho(Z_i,Z_j) \leq \delta_i\right] \leq 2^{-i-2}.
\end{equation*}
This is possible because $\mu$ is non-atomic, as a result with probability one, all $Z_k$ for $k\geq 1$ are distinct. Then, by the union bound, with probability at least $1-\frac{1}{2}=\frac{1}{2}$, for all $i\geq 0$ we have
\begin{equation*}
    \min_{1\leq u<v<n_{i+1}}  \rho(Z_u,Z_v) > \delta_i.
\end{equation*}
We denote by $\Ecal$ the event where the above inequality holds for all $i\geq 1$ and for all $u\geq 1$, $Z_u\neq \bar x$. Because $\mu$ is non-atomic, we still have $\Pbb[\Ecal]\geq \frac{1}{2}$. We now construct a partition of $\Xcal$ as follows. Let $(x^k)_k$ be a dense sequence of $\Xcal$. We denote by $B(x,r) = \{x'\in\Xcal,\rho(x,x')<r\}$ the ball centered at $x$ of radius $r>0$. For any $k\geq 1$ and $\delta>0$ let $P_k(\delta) = B(x^k,\delta)\setminus\bigcup_{l<k} B(x^l,\delta)$. Then, $(P_k(\delta))_k$ forms a partition of $\Xcal$. For any $\delta>0$ and sequence $\mb b =(b_k)_{k\geq 1}$ in $\{0,1\}$ we consider the following deterministic rewards
\begin{equation*}
    r_{\delta,\mb b}(a\mid x) = \begin{cases}
        b_k & a=a_1,\; x\in P_k(\delta),\\
        \frac{3}{4} &a=a_2,\\
        0 & a\notin\{a_1,a_2\}.
    \end{cases}
\end{equation*}
Now for any sequence of binary sequences $\mathbf b =(\mb {b^i})_{i\geq 0}$ where $\mb{b^i}=(b^i_k)_{k\geq 1}$, we will consider the memoryless rewards $\mb r^{\mathbf b}$ defined as follows. The deterministic rewards $r_t$ being constructed for $t\leq T_{k-1}$, we pose $r^{\mathbf b}_t = r_t$ for $t\leq T_{k-1}$. For all idle phases, i.e., $T_{k-1}<t< T^0$ or $2T^i\leq T^{i+1}$ for $i\geq 0$, we pose $r^{\mathbf b}_t = 0$. Last, for any $i\geq 0$ and $T^i\leq t<2T^i$ we pose $r^{\mathbf b}_t = r_{\delta_i,\mb{b^i}}$. Now let $\mathbf{b}$ be a random sequence such that all $\mb{b^i}$ are independent i.i.d. Bernouilli $\Bcal(\frac{1}{2})$ sequences in $\{0,1\}$. On the event $\Ecal$, all new instances fall in distinct sets of the partitions defining the rewards. Hence, with this perspective, the reward of the action $a_2$ is always $\frac{3}{4}$ while on the event $\Ecal$, for each new instance value, the reward of $a_1$ is a random Bernouilli $\Bcal(\frac{1}{2})$. Intuitively, for a specific instance $x$, if the learner has not yet explored the arm $a_1$, selecting $a_1$ incurs an average regret $\frac{1}{4}$ compared to selecting the fixed arm $a_2$. We will then argue that there is a time $T_k$ and a realization of $\tilde \Xbb_{\leq T_k}$ and rewards, such that on this realization, the regret compared to the best actions for each instance in hindsight is significantly large. We now formalize these ideas.

Because $\tilde \Xbb$ is a $\Ccal_1$ process, there exists a universally consistent learning rule under $\tilde \Xbb$. Then, because $f_\cdot$ is optimistically universal, it is universally consistent under $\tilde\Xbb$. Now fix a specific realization of the sequences in $\mathbf{b}$, considering the policy which always plays action $a_2$, i.e. $\pi_0:x\in\Xcal\mapsto a_2\in\Acal$, we have
\begin{equation*}
    \limsup_{T\to\infty}\frac{1}{T} \sum_{t=1}^T r^{\mathbf{b}}_t(a_2\mid X_t) - r^{\mathbf{b}}_t(\hat a_t\mid X_t)  \leq  0,\quad(a.s.).
\end{equation*}
In particular, since $\Pbb[\Ecal]\geq \frac{1}{2}$, we have
\begin{equation*}
    \Ebb\left[\limsup_{T\to\infty}\frac{1}{T} \sum_{t=1}^T r^{\mathbf{b}}_t(a_2\mid X_t) - r^{\mathbf{b}}_t(\hat a_t\mid X_t)  \mid \Ecal,\mathbf{b}\right] \leq 0.
\end{equation*}
As a result, taking the expectation over $\mathbf{b}$ then applying Fatou's lemma gives
\begin{equation*}
    \limsup_{T\to\infty}\Ebb\left[\frac{1}{T} \sum_{t=1}^T r^{\mathbf{b}}_t(a_2\mid X_t) - r^{\mathbf{b}}_t(\hat a_t\mid X_t)\mid \Ecal \right] \leq 0.
\end{equation*}
Now let $\alpha_k:=\frac{1}{16\cdot 4^{1/\epsilon_k}}$. In particular, there exists $i\geq \frac{4}{\alpha_k}$ such that for all $T\geq T^i$,
\begin{equation}\label{eq:upper_bound_base_regret}
    \Ebb\left[\frac{1}{T} \sum_{t=1}^{T} r^{\mathbf{b}}_t(a_2\mid X_t) - r^{\mathbf{b}}_t(\hat a_t\mid X_t)\mid \Ecal \right] \leq \frac{\alpha_k}{4}.
\end{equation}
For simplicity, we may write $r^{\mathbf{b}}_t(a)$ instead of $r^{\mathbf{b}}_t(a\mid x)$, when it is clear from context that $x=X_t$. We now focus on period $[T^i,2T^i)$ and denote by $\Scal_p^i:=\{T^i + (p-1)\cdot \epsilon_k T^i\leq t< T^i + p\cdot \epsilon_k T^i\}$ the sub-phase $p$ for $1\leq p\leq \frac{1}{\epsilon_k}$ of this period. Also note by $A_p^i$ the number of new exploration steps for arm $a_1$ during $\Scal_p^i$, i.e., times when the learner selected $a_1$ for an instance that had not previously been explored
\begin{equation*}
    \Acal_p^i = \{ t\in\Scal_p^i: \hat a_t = a_1 , \forall 1\leq q<p: \hat a_{t+ (q-p)\epsilon_k T^i} \neq a_1\},\quad A_p^i=|\Acal_p^i|.
\end{equation*}
We show by induction that $\Ebb[A^i_p\mid\Ecal] \leq 4^{p+1}  \alpha_k T^i$ for all $1\leq p\leq \frac{1}{\epsilon_k}$. Let $1\leq p\leq \frac{1}{\epsilon_k}$. Suppose that the result was shown for $1\leq q<p$ (if $p=1$ this is directly satisfied). We have
\begin{align*}
    &\Ebb  \left[\sum_{t=1}^{T^i (1 + p\epsilon_k) - 1} r^{\mathbf{b}}_t(a_2) - r^{\mathbf{b}}_t(\hat a_t) \mid \Ecal \right]\\ 
    &\geq -2T^{i-1} + \Ebb\left[ \sum_{t=T^i(1+(p-1)\epsilon_k)}^{T^i (1 + p\epsilon_k) - 1} (r_t^{\mathbf{b}}(a_2) - r_t^{\mathbf{b}}(\hat a_t))\1_{\Acal_p^i}(t) - \sum_{q<p}\frac{(p+1-q)A_q^i}{4} \mid \Ecal \right] \\
    &  = -2T^{i-1} -\sum_{q<p}\frac{p+1-q}{4} \Ebb[A^i_q\mid \Ecal]+ \Ebb\left[ \left.\sum_{t=T^i(1+(p-1)\epsilon_k)}^{T^i (1 + p\epsilon_k) - 1} \1_{\Acal^i_p}(t) \Ebb[r^{\mathbf{b}}_t(a_2) - r^{\mathbf{b}}_t(\hat a_t) |t\in\Acal^i_p,\Ecal] \right| \Ecal\right]
\end{align*}
where in the first inequality we discard times from phase $\Scal_p^i$ for which an exploration of the corresponding instance during phases $\Scal_1^i,\ldots \Scal_{p-1}^i$: these yield a regret least $(3/4-1)=-1/4$ compared to the fixed arm $a_2$. For each instance newly explored during phase $\Scal_q^i$, i.e. $t\in\Scal_q^i$, it affects potentially the $(p+1-q)$ next times with the same instance in phases $\Scal_q^i,\ldots,\Scal_p^i$. Now, note that all elements in $\mathbf{b}$ are together independent, and independent from the process $\Xbb$, in particular independent from $\Ecal$. As a result, the rewards at a time $\Acal_p^i$ are independent from the past because $X_t$ visits a set of the partition $(P_k(\delta_i))_k$ which has never been visited. Thus, we have
\begin{equation*}
    \Ebb[r^{\mathbf{b}}_t(a_2) - r^{\mathbf{b}}_t(\hat a_t) |t\in\Acal_p,\Ecal] = \frac{3}{4} - \frac{0+1}{2} = \frac{1}{4}.
\end{equation*}
Combining the above estimates with Eq~\eqref{eq:upper_bound_base_regret} then gives
\begin{multline*}
     -2T^{i-1} -\frac{1}{4}\sum_{q<p}(p+1-q)\Ebb[A^i_q\mid\Ecal] +  \frac{1}{4}\Ebb[A^i_p\mid\Ecal] \leq \Ebb  \left[\sum_{t=1}^{T^i (1 + p\epsilon_k) - 1} r^{\mathbf{b}}_t(a_2) - r^{\mathbf{b}}_t(\hat a_t) \mid \Ecal \right]\\
     \leq  \frac{\alpha_k}{4} T^i(1+p\epsilon_k)  \leq \frac{\alpha_k}{2} T^i.
\end{multline*}
Thus,
\begin{align*}
    \Ebb[A^i_p\mid \Ecal] &\leq \left(\frac{8}{1+i} + 2\alpha_k \right)T^i + \sum_{q<p}(p+1-q)\Ebb[A^i_q\mid\Ecal] \\
    &\leq 4\alpha_k T^i \left(1 + \sum_{q=1}^{p-1} (p+1-q) 4^q \right) \\
    &\leq 4\alpha_k T^i \left(1 + \sum_{q=1}^{p-1} 2^{p-q} 4^q \right) = 4\alpha_k T^i \left(1 + 2^p(2^p-1)\right)\leq 4^{p+1} \alpha_k T^i.
\end{align*}
This completes the induction. 

For any time $t$, denote $a_t^* = \arg\max_{a\in\Acal}r^{\mathbf{b}}_t(a)$ the optimal arm in hindsight. Note that $a^*_t\in\{a_1,a_2\}$. We lower bound the regret of the learner compared to the best action in hindsight until time $T^{i+1}$. To do so, define $\Bcal = \bigcup_{p=1}^{1/\epsilon_k}\{t\in\Scal_p^i: \forall 1\leq q\leq p, t + (q-p)\epsilon_k T^i \notin\Acal^i_q \}$ the set of times $t$ such that the learner never explored $a_1$ on the present and past appearances of the instance $X_t$. We also define $\Ccal = \{T^i\leq t<2T^i:a^*_t = a_1\}$ the set of times when $a_1$ was the optimal action. One can observe that for any time in $\Bcal$, because no exploration on $a_1$ was performed up for the corresponding instance $X_t$ in the past history, $\Pbb[t\in\Ccal|t\in\Bcal,\Ecal]=\frac{1}{2}$. Hence, if $t\in\Bcal\cap\Ccal\cap\Ecal$, the learner incurs a regret at least $\frac{1}{4}$ compared to the best arm $a^*_t =a_1$. Therefore,
\begin{equation*}
    \Ebb\left[\sum_{t=1}^{2T^i-1} r^{\mathbf{b}}_t(a_t^*) - r^{\mathbf{b}}_t(\hat a_t) \mid \Ecal\right] 
    \geq \frac{1}{4}\Ebb\left[\sum_{t\in\Bcal} \1_{\Ccal}(t) \mid \Ecal \right]
    = \frac{1}{8}\Ebb[|\Bcal| \mid\Ecal].
\end{equation*}
where by construction, we have $|\Bcal| + \sum_{p=1}^{1/\epsilon_k} \left(\frac{1}{\epsilon_k} - p+1\right)A^i_p = 2T^i - T^i = T^i$. As a result,
\begin{align*}
    \Ebb\left[\sum_{t=1}^{2T^i-1} r^{\mathbf{b}}_t(a_t^*) - r^{\mathbf{b}}_t(\hat a_t) \mid\Ecal \right] &\geq \frac{T^i}{8}    -  \frac{\alpha_k}{2} T^i \sum_{p=1}^{1/\epsilon_k} \left(\frac{1}{\epsilon_k} - p+1\right)4^p\\
    &\geq \frac{T^i}{8} - \alpha_k T^i 4^{1/\epsilon_k}\\
    &\geq \frac{T^i}{16} \geq \frac{2T^i-1}{32}.
\end{align*}
Hence, there exist a realization of instances $\mb X_{<2T^i} \leq \tilde \Xbb_{<2T^i}$ falling in $\Ecal$ and of rewards $(r_t)_{<2T^i}$ such that the regret compared to the best action in hindsight for on this specific instance sequence and for these rewards is at least $\frac{T^i}{16}$. We then pose $T_k:=2T^i-1$, and use the realization $\mb X_{\leq T_k}$, $(r_t)_{\leq T_k}$ for the deterministic process $\Xbb_{\leq T_k}$ and $(r_t)_{t\leq T_k}$. We recall that by construction, the realizations are consistent with the previously constructed process $\Xbb_{\leq T_{k-1}}$ and rewards $(r_t)_{\leq T_{k-1}}$. Further, to each new instance between times $T^i$ and $2T^i-1$ corresponded a best action in hindsight: this gives a collection of pairs $(x,a)$ where $x\in\Xcal$ is an instance visited by the deterministic process $\Xbb$ between times $T^i$ and $2T^i-1$ and $a\in\{a_1,a_2\}$ is the corresponding best action. Let $\Dcal_k$ denote this collection. This ends the recursive construction of the deterministic process $\Xbb$ and rewards.

Because we enforced that the samples of $\mu$ be always distinct and different from $\bar x$ across the construction of $\Xbb$, the countable collection $\bigcup_{k\geq 1}\Dcal_k$ of pairs instance/optimal-action never contains pairs with the same instance $x$. Hence, we can consider the following measurable policy $\pi^*:\Xcal\to\Acal$ defined by
\begin{equation*}
    \pi^*(x) = \begin{cases}  
        a &\text{if }(x,a)\in \bigcup_{k\geq 1}\Dcal_k,\\
        a_2 &\text{otherwise}.
    \end{cases}
\end{equation*}
This policy always performs the optimal action in hindsight. Hence by construction, for any $k\geq 1$,
\begin{equation*}
    \Ebb\left[\frac{1}{T_k}\sum_{t=1}^{T_k} r_t(\pi^*(X_t)\mid X_t) - r_t(\hat a_t\mid X_t) \right] \geq \frac{1}{32},
\end{equation*}
where $\hat a_t$ refers to the learner's decisions on the constructed process $\Xbb$ and rewards $(r_t)_{t\geq 1}$. Note that the expectation is taken only with respect to the learner's randomness given that $\Xbb$ and $(r_t)_{t\geq 1}$ are deterministic. Because the above equation holds for all $k\geq 1$ and $(T_k)_{k\geq 1}$ is an increasing sequence of times, we have
\begin{equation*}
    \Ebb\left[\limsup_{T\to\infty}\frac{1}{T}\sum_{t=1}^T r_t(\pi^*(X_t)) - r_t(\hat a_t)\right] \geq \limsup_{T\to\infty} \Ebb\left[\frac{1}{T}\sum_{t=1}^T r_t(\pi^*(X_t)) - r_t(\hat a_t)\right] \geq \frac{1}{32},
\end{equation*}
where we used Fatou's lemma. This proves that $f_\cdot$ is not universally consistent on $\Xbb$.

We now show that $\Xbb\in\Ccal_2$. It suffices to check that it visits a sublinear number of distinct points---this is also necessary since $\Xbb$ is deterministic. For $t\geq 1$, denote by $N_t$ the number of dinstint instances visited by the process $\Xbb_{\leq t}$. Fix $k\geq 1$. The process $\Xbb_{\leq T_k}$ being constructed from the process $\tilde \Xbb_{\leq T_k}$ above, we re-use the same notations. Let $i\geq 1$ such that $T_k = 2T^i-1$. For $1\leq j\leq i$ and $T^j\leq t<\min(T^{j+1},T_k)$ we have $N_t \leq T_{k-1} + 1 + n_j+k_j \leq 1 + \epsilon_k T^0 + 2k_j\leq 1+ 3\epsilon_k T^j \leq 1+ 3\epsilon_k t$. (The additional 1 accounts for $\bar x$.) For $T_{k-1}<t<T^0$, we have $N_t \leq 1+ N_{t_{k-1}} \leq 2+3\epsilon_{k-1} t$. As a result for all $T_{k-1}<t\leq T_k$ we have
\begin{equation*}
    N_t \leq 2 + 3\epsilon_{k-1} t.
\end{equation*}
Because $\epsilon_k\to 0 $ as $k\to \infty$, we obtain that $\frac{N_t}{t}\to 0$ as $t\to\infty$. This shows that $\Xbb\in\Ccal_2$. Because $\Xbb$ is deterministic and in $\Ccal_2$, \cref{prop:deterministic_C2} shows that there exists an universally consistent learning rule on $\Xbb$. However $f_\cdot$ is not universally consistent under $\Xbb$ which contradicts the hypothesis. This ends the proof that there does not exist an optimistically universal learning rule.
\end{proof}

We now turn to the case of spaces $\Xcal$ which do not have a non-atomic measure and show that in this case, the learning rule for processes visiting a sublinear number of distinct instances in \cref{prop:EXP.IX_parrallel} is optimistically universal learning rule for all settings including online rewards.

\begin{theorem}\label{thm:bad_borel_spaces}
Let $\Xcal$ a metrizable separable Borel space such that there does not exist a non-atomic probability measure on $\Xcal$, and $\Acal$ a finite action space. Then, learnable processes are exactly $\Ccal_{stat}=\Ccal_{online}=\Ccal_2$ and there exists an optimistically universal learning rule for all settings.
\end{theorem}

\begin{proof}
We show that any process $\Xbb\in\Ccal_2$ visits a sublinear number of distinct instances almost surely. Fix $\Xbb\in\Ccal_2$. Using \cite[Lemma 5.1]{blanchard:22e}, because $\Xcal$ does not admit a non-atomic probability measure, there exists a countable set $Supp(\Xbb)$ such that on an event $\Ecal$ of probability one, for all $t\geq 1$, $X_t\in Supp(\Xbb)$. Then consider the sequence $(\{x\})_{x\in Supp(\Xbb)}$ of disjoint measurable sets of $\Xcal$. Applying the $\Ccal_2$ property of $\Xbb$ to this sequence yields $|\{x\in Supp(\Xbb): \{x\}\cap \Xbb_{\leq T}\}| = o(T),\; (a.s.).$ We denote by $\Fcal$ the corresponding event of probability one. By union bound $\Pbb[\Ecal\cap\Fcal]=1$. Now on the event $\Ecal$, for any $T\geq 1$ we have
\begin{equation*}
    |\{x\in \Xcal: \{x\}\cap \Xbb_{\leq T}\neq\emptyset\}| = |\{x\in Supp(\Xbb): \{x\}\cap \Xbb_{\leq T}\}|.
\end{equation*}
As a result, on the event $\Ecal\cap\Fcal$ we have $|\{x\in \Xcal: \{x\}\cap \Xbb_{\leq T}\neq\emptyset\}|=o(T)$, which proves the claim that $\Ccal_2$ visit a sublinear number of distinct instances almost surely. As a result, the learning rule $f_\cdot$ from \cref{prop:EXP.IX_parrallel} which simply performs independent copies of the $\EXPIX$ algorithm for each distinct visited instance is universally consistent under all processes $\Xbb\in\Ccal_2$. Now recall that in the stationary case, the condition $\Ccal_2$ is already necessary for universal learning. In fact, this condition is already necessary for universal learning in the noiseless full-feedback setting \cite{hanneke:21}. As a result, $\Ccal_{online}\subset\Ccal_{stat}=\Ccal_2$. Therefore, universally learnable processes are exactly $\Ccal_2$ even in the online rewards setting and $f_\cdot$ is optimistically universal, which completes the proof.
\end{proof}

\section{Universally learnable processes for context spaces with non-atomic probability measures}
\label{sec:towards_characterization_learnable_processes}

\subsection{Necessary conditions on learnable processes}
\label{subsec:necessary_conditions_learnable_processes}

In the previous section, we showed that for spaces $\Xcal$ that do not have non-atomic probability measures, the set of learnable processes is exactly $\Ccal_2$, independently of the learning setting. Here, we focus on the remaining case of universal learning for spaces $\Xcal$ that admit a non-atomic probability measure for adversarial rewards and aim to understand which processes admit universal learning. We focus here on necessary conditions; sufficient conditions are given in the next section.

\subsubsection{Condition 4 is necessary for universal learning with oblivious rewards}
\label{subsubsec:necessary_condition}

We quickly recall the definition of condition $\Ccal_4$. For an integer $i\geq 0$ and any $k\geq 1$, we define $T^k_i = \lfloor 2^u (1+v2^{-i})  \rfloor$ where $k=u2^i + v$ and $u\geq 0, 0\leq v<2^i$ are integers. In particular, $u=\floor{k 2^{-i}}$ and $v= k\bmod 2^i$. These times form periods $[T_i^k,T_i^{k+1})$ which become finer as $i$ increases. Then consider the set of times $t$ such that $X_t$ is the first appearance of the instance on its period,
\begin{equation*}
    \Tcal^i = \{t\geq 1: T_i^k\leq t<T_i^{k+1}, \; \forall T_i^k\leq t'<t, X_{t'}\neq X_t\}.
\end{equation*}
We note that the sets $\Tcal^p$ are increasing with $p$. Condition $\Ccal_4$ is defined as follows.

\ConditionCsByScale*

We first give an alternative definition of $\Ccal_4$ which will be useful in the next results.

\begin{proposition}\label{prop:C5_equivalent_form}
    Let $\Xcal$ be a metrizable separable Borel space and $\Xbb$ a stochastic process on $\Xcal$. The following are equivalent.
    \begin{itemize}
        \item $\Xbb\in\Ccal_4$,
        \item For any sequence of decreasing measurable sets $(A_i)_{i\geq 1}$ with $A_i\downarrow\emptyset$,
        \begin{equation*}
            \sup_{p\geq 0}\Ebb\left[\limsup_{T\to\infty}\frac{1}{T} \sum_{t\leq T, t\in\Tcal^p} \1_{A_i}(X_t)\right] \underset{i\to\infty}{\longrightarrow} 0.
        \end{equation*}
        \item For any sequence of decreasing measurable sets $(A_i)_{i\geq 1}$ with $A_i\downarrow\emptyset$,
        \begin{equation*}
            \Ebb\left[\sup_{p\geq 0} \limsup_{T\to\infty}\frac{1}{T} \sum_{t\leq T, t\in\Tcal^p} \1_{A_i}(X_t)\right] \underset{i\to\infty}{\longrightarrow} 0.
        \end{equation*}
    \end{itemize}
\end{proposition}

\begin{proof}
Suppose that the second proposition is not satisfied. We aim to show that $\Xbb\notin\Ccal_4$. By hypothesis, there exists measurable sets $A_i\downarrow\emptyset$, $\epsilon>0$, and an increasing sequence of indices $(i_p)_{p\geq 1}$  such that
\begin{equation*}
    \sup_{l\geq 0}\Ebb\left[\limsup_{T\to\infty}\frac{1}{T} \sum_{t\leq T, t\in\Tcal^l} \1_{A_{i_p}}(X_t)\right] \geq \epsilon.
\end{equation*}
Now let $i\geq 1$ and $p\geq 1$ such that $i_p\geq i$. We observe that because $A_{i_p}\subset A_i$,
\begin{equation*}
    \sup_{l\geq 0}\Ebb\left[\limsup_{T\to\infty}\frac{1}{T} \sum_{t\leq T, t\in\Tcal^l} \1_{A_i}(X_t)\right] \geq \sup_{l\geq 0}\Ebb\left[\limsup_{T\to\infty}\frac{1}{T} \sum_{t\leq T, t\in\Tcal^l} \1_{A_{i_p}}(X_t)\right] \geq \epsilon.
\end{equation*}
Hence, for any $i\geq 1$, there exists $p(i)>0$ such that
\begin{equation*}
    \Ebb\left[\limsup_{T\to\infty}\frac{1}{T} \sum_{t\leq T, t\in\Tcal^{p(i)}} \1_{A_i}(X_t)\right] \geq \frac{\epsilon}{2}.
\end{equation*}
\paragraph{Case 1.}We consider a first case where there exists $\eta_i>0$ such that for any $j\geq i$,
\begin{equation*}
    \Ebb\left[\limsup_{T\to\infty}\frac{1}{T} \sum_{t\leq T, t\in\Tcal^{p(i)}} \1_{A_j}(X_t)\right] \geq \eta_i.
\end{equation*}
For simplicity, we will write $T^k = T^k_{p(i)}$. We will also drop the indices $i$ of $p(i)$ and $\eta_i$ for conciseness. We now construct by induction a sequence of indices $(k(l))_{l\geq 0}$ together with indices $(j(l))_{l\geq 0}$ with $k(0)=1$, $j(0)=i$ and such that for any $l\geq 1$,
\begin{equation*}
    \Ebb\left[\sup_{T^{k(l-1)}<T\leq T^{k(l)}}\frac{1}{T} \sum_{t\leq T, t\in\Tcal^p} \1_{A_{j(l-1)}\setminus A_{j(l)}}(X_t)\right] \geq \frac{\eta}{2}.
\end{equation*}
Suppose that we have already constructed $j(0),\ldots,j(l-1)$ and $k(0),\ldots,k(l-1)$. Note that
\begin{equation*}
    \Ebb\left[\sup_{T>T^{k(l-1)}}\frac{1}{T} \sum_{t\leq T, t\in\Tcal^p} \1_{A_{j(l-1)}}(X_t)\right] \geq \Ebb\left[\limsup_{T\to\infty}\frac{1}{T} \sum_{t\leq T, t\in\Tcal^p} \1_{A_{j(l-1)}}(X_t)\right] \geq \eta.
\end{equation*}
Therefore, by the dominated convergence theorem, there exists $k(l)>k(l-1)$ such that
\begin{equation*}
     \Ebb\left[\sup_{T^{k(l-1)}<t\leq T^{k(l)}}\frac{1}{T} \sum_{t\leq T, t\in\Tcal^p} \1_{A_{j(l-1)}}(X_t)\right] \geq \frac{3\eta}{4}.
\end{equation*}
Now because $A_i\downarrow \emptyset$, there exists $j(l)>j(l-1)$ such that $\Pbb[A_{j(l)}\cap\Xbb_{\leq T^{k(l)}}=\emptyset]\geq 1- \frac{\eta}{4}$. Let us denote by $\Ecal$ this event. Then,
\begin{align*}
    \Ebb&\left[\sup_{T^{k(l-1)}<t\leq T^{k(l)}}\frac{1}{T} \sum_{t\leq T, t\in\Tcal^p} \1_{A_{j(l-1)}\setminus A_{j(l)}}(X_t)\right] \\
    &\geq \Ebb\left[\1[\Ecal]\sup_{T^{k(l-1)}<t\leq T^{k(l)}}\frac{1}{T} \sum_{t\leq T, t\in\Tcal^p} \1_{A_{j(l-1)}\setminus A_{j(l)}}(X_t)  \right]\\
    &= \Ebb\left[\1[\Ecal]\sup_{T^{k(l-1)}<t\leq T^{k(l)}}\frac{1}{T} \sum_{t\leq T, t\in\Tcal^p} \1_{A_{j(l-1)}}(X_t) \right]\\
    &\geq \Ebb\left[\sup_{T^{k(l-1)}<t\leq T^{k(l)}}\frac{1}{T} \sum_{t\leq T, t\in\Tcal^p} \1_{A_{j(l-1)}}(X_t) \right] - \frac{\eta}{4}\geq \frac{\eta}{2}.
\end{align*}
This ends the construction of the indices $k(l)$ and $j(l)$ for $l\geq 1$. Now for any $u\geq 1$, let $S_u = \{l\geq 1:l\equiv 2^{u-1}\bmod 2^u\}$. The main remark is that $S_u$ is infinite for all $u\geq 1$ and they are all disjoint. We then pose $B_u = \bigcup_{l\in S_u} A_{j(l-1)}\setminus A_{j(l)}$. Because all $S_u$ are disjoint, this implies that the sets $(B_u)_u$ are also disjoint. Then, using Fatou's lemma together with the fact that all $S_u$ are infinite, we obtain
\begin{align*}
    \Ebb\left[\limsup_{T\to\infty}\frac{1}{T} \sum_{t\leq T, t\in\Tcal^p} \1_{B_u}(X_t)\right]
    &\geq \limsup_{k\in S_u} \Ebb\left[\sup_{T^{k(l-1)}<T\leq T^{k(l)}}\frac{1}{T} \sum_{t\leq T, t\in\Tcal^p} \1_{B_u}(X_t)\right]\\
    &\geq \limsup_{k\in S_u}\Ebb\left[\sup_{T^{k(l-1)}<T\leq T^{k(l)}}\frac{1}{T} \sum_{t\leq T, t\in\Tcal^p} \1_{A_{j(l-1)}\setminus A_{j(l)}}(X_t)\right]\\
    &\geq \frac{\eta}{2}.
\end{align*}
We obtain therefore for any $u\geq p$
\begin{equation*}
    \Ebb\left[\limsup_{T\to\infty}\frac{1}{T} \sum_{t\leq T, t\in\Tcal^u} \1_{B_u}(X_t)\right] \geq \Ebb\left[\limsup_{T\to\infty}\frac{1}{T} \sum_{t\leq T, t\in\Tcal^p} \1_{B_u}(X_t)\right] \geq \frac{\eta}{2}.
\end{equation*}
This ends the proof that $\Xbb\notin\Ccal_4.$

\paragraph{Case 2.} Recalling that the sets $(A_i)_i$ are decreasing, we can now suppose that for all $i\geq 1$, one has $\Ebb\left[\limsup_{T\to\infty}\frac{1}{T} \sum_{t\leq T, t\in\Tcal^{p(i)}} \1_{A_j}(X_t)\right] \to 0$ as $j\to\infty$. We now construct a sequence of indices $(i(u))_{u\geq 1}$ as follows such that $i(1)=1$ and for any $u\geq 1$,
\begin{equation*}
    \Ebb\left[\limsup_{T\to\infty}\frac{1}{T} \sum_{t\leq T, t\in\Tcal^{p(i(u))}} \1_{A_{i(u)}\setminus A_{i(u+1)}}(X_t)\right] \geq \frac{\epsilon}{4}.
\end{equation*}
Suppose we have constructed $i(u)$. Then, by the hypothesis of this case, there exists $i(u+1)>i(u)$ such that
\begin{equation*}
    \Ebb\left[\limsup_{T\to\infty}\frac{1}{T} \sum_{t\leq T, t\in\Tcal^{p(i(u))}} \1_{A_{i(u+1)}}(X_t)\right] \leq \frac{\epsilon}{4}.
\end{equation*}
Now note that
\begin{multline*}
    \frac{\epsilon}{2} \leq \Ebb\left[\limsup_{T\to\infty}\frac{1}{T} \sum_{t\leq T, t\in\Tcal^{p(i(u))}} \1_{A_{i(u)}}(X_t)\right] \leq \Ebb\left[\limsup_{T\to\infty}\frac{1}{T} \sum_{t\leq T, t\in\Tcal^{p(i(u))}} \1_{A_{i(u)}\setminus A_{i(u+1)}}(X_t)\right]\\
    + \Ebb\left[\limsup_{T\to\infty}\frac{1}{T} \sum_{t\leq T, t\in\Tcal^{p(i(u))}}  \1_{A_{i(u+1)}}(X_t)\right].
\end{multline*}
As a result, the induction at step $p$ is complete. We then define a sequence of measurable sets $(B_j)_{j\geq 1}$ such that for any $u\geq 1$, $B_{p(i(u))} = A_{i(u)}-A_{i(u+1)}$, and for all other indices $j\notin\{p(i(u)),u\geq 1\}$ we set $B_j=\emptyset$. All these sets are disjoint, and we have for any $u\geq 1$,
\begin{equation*}
    \Ebb\left[\limsup_{T\to\infty}\frac{1}{T} \sum_{t\leq T, t\in\Tcal^{p(i(u)}} \1_{B_{p(i(u))}}(X_t)\right] \geq \frac{\epsilon}{4}.
\end{equation*}
Therefore, $\Xbb\notin\Ccal_4$.\\

We now show that if $\Xbb$ satisfies the second property, then $\Xbb\in\Ccal_4$. Let $(A_i)_i$ be a sequence of disjoint measurable sets, and define $B_i=\bigcup_{j\geq i}A_j$. Then,
\begin{align*}
    0\leq \Ebb\left[\limsup_{T\to\infty}\frac{1}{T} \sum_{t\leq T, t\in\Tcal^i} \1_{A_i}(X_t)\right] &\leq \Ebb\left[\limsup_{T\to\infty}\frac{1}{T} \sum_{t\leq T, t\in\Tcal^i} \1_{B_i}(X_t)\right]\\
    &\leq \sup_{p\geq 0}\Ebb\left[\limsup_{T\to\infty}\frac{1}{T} \sum_{t\leq T, t\in\Tcal^p} \1_{B_i}(X_t)\right].
\end{align*}
Hence, because $B_i\downarrow\emptyset$, the second property implies that $    \Ebb\left[\limsup_{T\to\infty}\frac{1}{T} \sum_{t\leq T, t\in\Tcal^i} \1_{A_i}(X_t)\right] \to 0$ as $i\to\infty$.

Now for any Borel set $A$, by the dominated convergence theorem and the fact that the sets $\Tcal^p$ are increasing for $p\geq 0$, we obtain
\begin{equation*}
    \lim_{p\to \infty}\Ebb\left[\limsup_{T\to\infty}\frac{1}{T} \sum_{t\leq T, t\in\Tcal^p} \1_A(X_t)\right]  = \Ebb\left[\lim_{p\to\infty}\limsup_{T\to\infty}\frac{1}{T} \sum_{t\leq T, t\in\Tcal^p} \1_A(X_t)\right],
\end{equation*}
where both terms are bounded by $1$. In other terms,
\begin{equation*}
    \sup_{p\geq 0}\Ebb\left[\limsup_{T\to\infty}\frac{1}{T} \sum_{t\leq T, t\in\Tcal^p} \1_A(X_t)\right]  = \Ebb\left[\sup_{p\geq 0}\limsup_{T\to\infty}\frac{1}{T} \sum_{t\leq T, t\in\Tcal^p} \1_A(X_t)\right].
\end{equation*}
As a result, the second and third condition of the proposition are equivalent.
\end{proof}

The main result of this section is that the $\Ccal_4$ condition is necessary for universal learning with oblivious rewards.

\begin{theorem}\label{thm:condition_5_necessary}
Let $\Xcal$ a metrizable separable Borel space, and a finite action space $\Acal$ with $|\Acal|\geq 2$. Then, $\Ccal_{oblivious}\subset\Ccal_4$.
\end{theorem}

\begin{proof}
Fix, $a_1,a_2\in\Acal$ two distinct actions. By contradiction, let $\Xbb\notin\Ccal_4$ and $f_\cdot$ a universally consistent learning rule under $\Xbb$ for oblivious rewards. For simplicity, we will denote by $\hat a_t$ the action selected by the learning rule at time $t$. By hypothesis, let $(A_i)_{i\geq 1}$ be a sequence of disjoint measurable sets and $0<\epsilon\leq 1$ such that
\begin{equation*}
    \limsup_{i\to\infty}\Ebb\left[\limsup_{T\to\infty}\frac{1}{T} \sum_{t\leq T, t\in\Tcal^i} \1_{A_i}(X_t)\right]\geq \epsilon.
\end{equation*}
Then, there exists an increasing sequence $(j(i))_{i\geq 1}$ such that for any $p\geq 1$,
\begin{equation*}
    \Ebb\left[\limsup_{T\to\infty}\frac{1}{T} \sum_{t\leq T, t\in\Tcal^{j(i)}} \1_{A_{j(i)}}(X_t)\right]\geq \frac{\epsilon}{2}.
\end{equation*}
We write $\Ical=\{j(i),i\geq 1\}$. Without loss of generality, we can suppose $A_j=\emptyset$ if $j\in \Ical$.
We now construct recursively rewards $(r_t)_{t\geq 1}$ on which this algorithm is not consistent, as well as a policy $\pi^*:\Xcal\to\Acal$ compared to which the algorithm has high regret. The reward functions and policy are constructed recursively together with an increasing sequence of times $(T^p)_{p\in\Ical}$ such that after the $p-$th iteration of the construction process, the rewards $r_t$ for $t\leq T^p$ have been defined such that $r_t(\cdot\mid x_{\leq t})=0$ if $x\notin\bigcup_{i<p}A_i$, the policy $\pi^*(\cdot)$ is defined on $\bigcup_{i< p}A_i$ and always the best action in hindsight until $T^{p-1}$. For $p=j(p')$, suppose that we have performed $p'-1$ iterations of this construction and have constructed the times $T^{j(1)},\ldots,T^{j(p'-1)}$. For convenience, let $\alpha_p=2^{-p-1}$ and define $K_p =\left\lceil\frac{2}{\alpha_{p}}\log\frac{2^6}{\epsilon}\right\rceil$, $\beta_p=\frac{\epsilon}{2^{10}(1+2\alpha_p)^{(K_p-1)K_p} 4^{K_p}}$, $\tilde K_p = \left\lceil\frac{2}{\alpha_{p}}\log \frac{8}{\beta_p}\right\rceil$ and $M_p = \max(\frac{8}{\epsilon \alpha_{p} }, (1+2\alpha_{p})^{K_p+\tilde K_p})$. We first construct by induction an increasing sequence of indices $(k(l))_{l\geq 0}$ with $k(0) = \min\{k\geq 2^p:T_p^k > M_p T^{j(p'-1)}\}$ and such that for any $l\geq 1$, $T_p^{k(l)} > M_p T_p^{k(l-1)}$ and
\begin{equation*}
    \Ebb\left[\max_{M_p T_p^{k(l-1)}< T\leq T_p^{k(l)}} \frac{1}{T} \sum_{t\leq T,t\in\Tcal^p} \1_{A_{p}}(X_t)\right] \geq \frac{\epsilon}{4}.
\end{equation*}
To do so, suppose that we have constructed $k(l')$ for $0\leq l'<l$. Note that
\begin{equation*}
    \Ebb\left[\sup_{T> M_p T_p^{k(l-1)}}\frac{1}{T} \sum_{t\leq T, t\in\Tcal^p} \1_{A_{p}}(X_t)\right]\geq \Ebb\left[\limsup_{T\to\infty}\frac{1}{T} \sum_{t\leq T, t\in\Tcal^p} \1_{A_{p}}(X_t)\right]\geq\frac{\epsilon}{2}.
\end{equation*}
Then, by dominated convergence theorem, there exists $k(l)>k(l-1)$ such that $T_p^{k(l)}>M_pT_p^{k(l-1)}$ and
\begin{equation*}
    \Ebb\left[\max_{M_p T_p^{k(l-1)} < T\leq T_p^{k(l)}}\frac{1}{T} \sum_{t\leq T, t\in\Tcal^p} \1_{A_{p}}(X_t)\right] \geq \Ebb\left[\sup_{T> M_p T_p^{k(l-1)}}\frac{1}{T} \sum_{t\leq T, t\in\Tcal^p} \1_{A_{p}}(X_t)\right] - \frac{\epsilon}{4} \geq \frac{\epsilon}{4}.
\end{equation*}
This ends the construction of the sequence $(k(l))_{l\geq 0}$. We then denote by $\hat k(l)$ the index of a phase $(T^{k-1}_p,T^{k}_p]$ where the max is attained, i.e.
\begin{equation*}
    \hat k(l) = \argmax_{k\leq k(l)} \left(\max_{M_p T_p^{k(l-1)},T_p^{k-1}< T\leq T_p^k} \frac{1}{T} \sum_{t\leq T,t\in\Tcal^p} \1_{A_{p}}(X_t) \right).
\end{equation*}
Ties can be broken with alphabetical order. Because $T^k_p\leq 2T^{k-1}_p$, we have in particular,
\begin{equation*}
    \Ebb\left[\frac{1}{T^{\hat k(l)}_p} \sum_{t\leq T^{\hat k(l)}_p,t\in\Tcal^p} \1_{A_{p}}(X_t) \right]\geq \frac{\epsilon}{8}.
\end{equation*}
Now for any $l\geq 1$, let $\delta_l$ such that
\begin{equation*}
    \Pbb\left[ \min_{1\leq t,t'\leq T_p^{k(l)},X_t\neq X_{t'}} \rho(X_t,X_{t'}) \leq \delta_l\right]\leq \frac{\epsilon}{2^{l+10}}.
\end{equation*}
Then, let $\Ecal$ be the event when for all $l\geq 1$, we have $\min_{1\leq t,t'\leq T_p^{k(l)},X_t\neq X_{t'}} \rho(X_t,X_{t'}) > \delta_l$. By the union bound, $\Pbb[\Ecal]\geq 1-\frac{\epsilon}{2^{10}}$. As a result, we have
\begin{equation}\label{eq:def_hat_k}
    \Ebb\left[\frac{1}{T^{\hat k(l)}_p} \sum_{t\leq T^{\hat k(l)}_p,t\in\Tcal^p} \1_{A_{p}}(X_t) \mid\Ecal \right]\geq \frac{\epsilon}{16}.
\end{equation}
Now for $\delta>0$ and $u\geq 1$, define the sets $P_u(\delta) = (A_{p}\cap B(x^u,\delta))\setminus\bigcup_{v<u} B(x^v,\delta)$ which form a partition of $A_p$. For any $\delta>0$ and sequence $\mb b =(b_u)_{u\geq 1}$ in $\{0,1\}$ we consider the following deterministic rewards
\begin{equation*}
    r_{\delta,\mb b}(a\mid x) = \begin{cases}
        b_u & a=a_1,\; x\in P_u(\delta),\\
        \frac{3}{4} &a=a_2,\\
        0 & a\notin\{a_1,a_2\},
    \end{cases}  \text{if } x\in A_p,\quad\quad 
    r_{\delta,\mb b}(\cdot \mid x) = 0  \text{ if } x\notin A_p.
\end{equation*}
For any sequence of binary sequences $\mathbf b =(\mb {b^k})_{k\geq 0}$ where $\mb{b^k}=(b^k_u)_{u\geq 1}$, and binary sequence $\mb c = (c_k)_{k\geq 0}$ we construct the rewards $\mb r^{\mathbf b,\mb c}$ as follows. For $t\leq T^{j(p'-1)}$ we pose $r^{\mathbf b,\mb c}_t = r_t$ so that the rewards $\mb r^{\mathbf b,\mb c}$ coincide with those constructed by induction so far. For $T^{j(p'-1)}<t\leq T_p^{k(0)}$ we pose $r_t^{\mathbf b,\mb c}=0$. For $t>T^{k(0)}_p$ let $l\geq 1$ such that $T^{k(l-1)}_p<t\leq T^{k(l)}_p$ and $k> k(0)$ such that $T^{k-1}_p<t\leq T^k_p$. Then, we pose
\begin{equation*}
    r^{\mathbf b,\mb c}_t(a\mid x_{\leq t}) =\begin{cases}
        0 & \exists t'\leq T^{k(l-1)}_p: x_{t'}=x_t\\
        0 & \text{o.w. }  c_k=0,\\
        r_{\delta_l,\mb{b^l}}(a\mid x_t) & \text{o.w. } c_k=1,\forall T^{k-1}_p<t'<t:x_{t'}\neq x_t,\\
        0 & \text{o.w. } c_k=1,\exists T^{k-1}_p<t'<t:x_{t'}= x_t,\\
    \end{cases}
\end{equation*}
for $a\in\Acal,x_{\leq t}\in\Xcal^t.$ Note that these rewards coincide on the rewards that have been constructed by induction so far. Now let $\mathbf b$ be generated such that all $\mb{b^k}$ are independent i.i.d. Bernouilli $\Bcal(\frac{1}{2})$ random sequences in $\{0,1\}$, and $\mb c$ is also an independent i.i.d. $\Bcal(\frac{1}{2})$ process. The sequence is used to delete some periods $(T^{k-1}_p,T^k_p]$. Precisely, for any $l\geq 1$,  we consider the following event where we deleted the periods between $\hat k(l)-K_p-\tilde K_p$ and $\hat k(l)-K_p$ but did not delete periods after this phase until period $\hat k(l)$,
\begin{equation*}
    \Fcal^p_l = \bigcap_{\hat k(l) - K_p-\tilde K_p < k \leq \hat k(l)-K_p}\{c_k=0\} \cap \bigcap_{\hat k(l) - K_p < k\leq \hat k(l)}\{c_k=1\}.
\end{equation*}
One can note that the events $\Fcal^p_l$ for $l\geq 1$ are together independent. Indeed, $\hat k(l)\leq k(l)$ and $T^{\hat k(l)}_p> M_p T^{k(l-1)} \geq (1+2\alpha_p)^{K_p+\tilde K_p} T^{k(l-1)}$, which yields $\hat k(l)>k(l-1) + K_p+\tilde K_p$. As a result, the indices of $\mb c$ considered in the events $\Fcal^p$ all lie in distinct intervals $(k(l-1),k(l)]$, hence their independence. Further, we have $\Pbb[\Fcal^p_l]=2^{-K_p-\tilde K_p}$. Then, the Borel-Cantelli implies that on an event $\Fcal^p$ of probability one, there is an infinite number of $l\geq 1$ such that $\Fcal^p_l$ is satisfied.

Next, define $\pi_0:x\in\Xcal\mapsto a_2\in\Acal$, the policy which always selects arm $a_2$. Fix any realization of $\mathbf b$ and $\mb c$. Because $f_\cdot$ is universally consistent for oblivious rewards, it has in particular sublinear regret compared to $\pi_0$ under rewards $\mb r^{\mathbf b,\mb c}$, i.e., almost surely $\limsup_{T\to\infty} \frac{1}{T}\sum_{t=1}^T r^{\mathbf b,\mb c}_t(a_2\mid X_t) - r^{\mathbf b,\mb c}_t(\hat a_t \mid X_t)) \leq 0$. 
Now observe that the event $\Fcal^p$ only depends on $\mb c$ and $\Xbb$ and is in particular independent from $\mathbf b$. Therefore, $\Pbb[\Ecal\cap\Fcal^p\mid \mathbf b] =\Pbb[\Ecal\cap\Fcal^p] \geq 1-\frac{\epsilon}{2^{10}} $, where we used $\Pbb[\Fcal^p]=1$. Therefore,
\begin{equation*}
    \Ebb\left[\limsup_{T\to\infty}\frac{1}{T} \sum_{t=1}^T r^{\mathbf{b},\mb c}_t(a_2\mid \Xbb_{\leq t}) - r^{\mathbf{b}, \mb c}_t(\hat a_t\mid \Xbb_{\leq t})  \mid \Ecal,\Fcal^p,\mathbf{b}\right] \leq 0.
\end{equation*}

For conciseness, we will omit the terms $\Xbb_{\leq t}$ in the rest of the proof. We then take the expectation over $\mathbf b$ and $\mb c$. Thus, by the dominated convergence theorem, there exists $l_0\geq 1$ such that

\begin{equation*}
    \Ebb\left[\sup_{T> T_p^{k(l_0)}}\frac{1}{T} \sum_{t=1}^T r^{\mathbf{b},\mb c}_t(a_2) - r^{\mathbf{b}, \mb c}_t(\hat a_t)  \mid \Ecal,\Fcal^p\right] \leq \frac{\beta_p}{8}.
\end{equation*}
On the event $\Fcal^p$, there exists $\hat l> l_0$ such that the event $\Fcal^p_{\hat l}$ is met. For convenience, we take $\hat l$ the minimum index satisying these conditions. Then, we have
\begin{align*}
    &\Ebb\left[\sup_{T_p^{\hat k(\hat l)-K_p} < T\leq T_p^{\hat k(\hat l)}}\frac{1}{T} \sum_{t=1}^T r^{\mathbf{b}, \mb c}_t(a_2) - r^{\mathbf{b}, \mb c}_t(\hat a_t)  \mid \Ecal, \Fcal^p\right] \\
    &\leq \Ebb\left[\sup_{T_p^{k(\hat l-1)} < T\leq T_p^{k(\hat l)}}\frac{1}{T} \sum_{t=1}^T r^{\mathbf{b}, \mb c}_t(a_2) - r^{\mathbf{b}, \mb c}_t(\hat a_t)  \mid \Ecal, \Fcal^p\right] \leq \frac{\beta_p}{8}.
\end{align*}
Now let $l^p$ such that $\Pbb[\hat l \leq l^p\mid \Fcal^p] \geq \frac{1}{2}$. Then,
\begin{equation}\label{eq:upper_bound_exploration_error}
    \Ebb\left[\sup_{T_p^{\hat k(\hat l)-K_p} < T\leq T_p^{\hat k(\hat l)}}\frac{1}{T} \sum_{t=1}^T r^{\mathbf{b}, \mb c}_t(a_2) - r^{\mathbf{b}, \mb c}_t(\hat a_t)  \mid \Ecal, \Fcal^p,\hat l\leq l_p\right] \leq \frac{\beta_p}{4}.
\end{equation}

For conciseness, we will write $\hat k$ for $\hat k(\hat l)$, let $\Gcal^p = \Ecal\cap\Fcal^p\cap\{\hat l\leq l_p\}$. We now use similar same arguments as in the proof of \cref{thm:no_optimistically_universal_lr}, to show that the learning rule incurs a large regret compared to the best action in hindsight, before time $T^{\hat k}_p$. We focus on the period $(T_p^{\hat k-K_p},T^{\hat k}_p]$, which we decompose using the sets
\begin{equation*}
    \Scal_q = \{T_p^{\hat k-K_p-1+q}<t\leq T^{\hat k-K_p + q}_p: X_t\in A_p\}\cap \Tcal^p ,\quad 1\leq q\leq K_p.
\end{equation*}
We also define $E_q$ the number of new exploration steps for arm $a_1$ during $\Scal_q$,
\begin{equation*}
    Exp_q = \left\{t\in\Scal_q: \hat a_t = a_1\text{ and }\forall t'\in \bigcup_{q'<q}\Scal_{q'}:X_{t'}=X_t,\; \hat a_{t'}\neq a_1\right\} \setminus \{t:\exists t'\leq T^{\hat k -\tilde K_p}_p, X_{t'}=X_t\},
\end{equation*}
and $E_q = |Exp_q|$. We now show by induction on $i$ that $\Ebb\left[\frac{E_q}{T_p^{\hat k}}\mid \Gcal^p\right]\leq (1+2\alpha_p)^{(q-1)K_p}4^{q+1}\beta_p $ for all $1\leq q\leq K_p$. Suppose that this is shown for all $1<q'<q$. Recalling that on the event $\Gcal^p$, for any $T^{\hat k-K_p-\tilde K_p}_p<t\leq T^{\hat k-K_p}_p$ we have $r^{\mathbf b, \mb c}_t=0$, we can use the same arguments as in \cref{thm:no_optimistically_universal_lr} to obtain
\begin{align*}
    \Ebb&\left[\frac{1}{T_p^{\hat k-K_p+q}}\sum_{t=1}^{T_p^{\hat k-K_p+q}} r^{\mathbf{b}, \mb c}_t(a_2) - r^{\mathbf{b}, \mb c}_t(\hat a_t) \mid \Gcal^p\right]\\
    &\geq -\Ebb\left[ \frac{T^{\hat k-K_p-\tilde K_p}_p}{T_p^{\hat k-K_p+q}}\mid \Gcal^p\right] + \sum_{q'=1}^q \Ebb\left[\frac{1}{T_p^{\hat k-K_p+q}}\sum_{t = T_p^{\hat k-K_p-1+q'}+1}^{ T_p^{\hat k-K_p+q'}} r^{\mathbf{b}, \mb c}_t(a_2) - r^{\mathbf{b}, \mb c}_t(\hat a_t) \mid \Gcal^p\right]\\
    &= -\Ebb\left[ \frac{T^{\hat k-K_p-\tilde K_p}_p}{T_p^{\hat k-K_p+q}}\mid \Gcal^p\right] + \sum_{q'=1}^q \Ebb\left[\frac{1}{T_p^{\hat k-K_p+q}}\sum_{t\in\Scal^{q'}} r^{\mathbf{b}, \mb c}_t(a_2) - r^{\mathbf{b}, \mb c}_t(\hat a_t) \mid \Gcal^p\right]\\
    &\geq -(1+q)\Ebb\left[ \frac{T^{\hat k-K_p-\tilde K_p}_p}{T_p^{\hat k-K_p+q}}\mid \Gcal^p\right] - \sum_{q'<q} \frac{q+1-q'}{4}\Ebb\left[\frac{E_{q'}}{T_p^{\hat k-K_p}}\mid\Gcal^p\right] \\
    &\quad\quad\quad\quad\quad\quad\quad+ \Ebb\left[\frac{1}{T_p^{\hat k}}\sum_{t = T_p^{\hat k-K_p-1+q}+1}^{ T_p^{\hat k-K_p+q}} \1_{Exp_q}(t) (r^{\mathbf{b}, \mb c}_t(a_2) - r^{\mathbf{b}, \mb c}_t(\hat a_t) ) \mid \Gcal^p\right],
\end{align*}
where the additional terms $-T^{\hat k-\tilde K_p}_p$ compared to the computations in \cref{thm:no_optimistically_universal_lr} are due to the fact that in $Exp_q$ we also discard times of instances that were visited before $T^{\hat k-\tilde K_p}$, and that in a single period $\Scal_q$, there are no duplicates. Now for any $T^{\hat k-K_p-1+q}<t\leq T^{\hat k-K_p+q}$ such that a pure exploration was performed $t\in Exp_q$, we have
\begin{equation*}
    \Ebb[r^{\mathbf{b}, \mb c}_t(a_2) - r^{\mathbf{b}, \mb c}_t(\hat a_t)\mid t\in Exp_q,\Gcal^p,\hat k] = \frac{3}{4} - \frac{0+1}{2} = \frac{1}{4},
\end{equation*}
because $X_t$ visits a set of the partition $(P_u(\delta_{\hat k-K_p+q}))_u$ which has never been visited in the past, hence the reward of $a_1$ on this set is equally likely to be $0$ or $1$ (depending on $\mathbf b$), and $\Gcal^p$ is independent from $\mathbf b$. Also, using the inequality $\log(1+z)\geq \frac{z}{2}$ for $0\leq z\leq 1$ we obtain $T^{\hat k-\tilde K_p}_p \leq (1+\alpha_p)^{-\tilde K_p}(1+T^{\hat k-K_p}_p)\leq \frac{\beta_p}{8}(1+ T^{\hat k-K_p}_p)\leq \frac{\beta_p}{4}T^{\hat k-K_p}_p$. Lastly, $T_p^{\hat k-K_p}\geq T_p^{\hat k}/(1+2\alpha_p)^{K_p}$. Combining these results with Eq~\eqref{eq:upper_bound_exploration_error} yields
\begin{equation*}
    \frac{\beta_p}{4} \geq -(1+q)\frac{\beta_p}{4} - \frac{1}{4}\sum_{q'<q} (q+1-q')(1+2\alpha_p)^{K_p}\Ebb\left[\frac{E_{q'}}{T_p^{\hat k}}\mid\Gcal^p\right] +\frac{1}{4}\Ebb\left[\frac{E_q}{T_p^{\hat k}}\mid\Gcal^p\right].
\end{equation*}
Thus,
\begin{align*}
    \Ebb\left[\frac{E_q}{T_p^{\hat k}}\mid\Gcal^p\right] &\leq  (2 + q)\beta_p+(1+2\alpha_p)^{K_p} \sum_{q'<q}(q+1-q')\Ebb\left[\frac{E_{q'}}{T_p^{\hat k}}\mid\Gcal^p\right] \\
    &\leq  (1+2\alpha_p)^{(q-1)K_p}\beta_p  \left(2+q + 4\sum_{q'=1}^{q-1} (q+1-q') 4^{q'} \right) \\
    &\leq (1+2\alpha_p)^{(q-1)K_p}4^{q+1} \beta_p .
\end{align*}
This completes the induction. Now for any $t\geq 1$, denote by $a^*_t = \argmax_{a\in\Acal} r_t^{\mathbf b,\mb c}(a\mid \Xbb_{\leq t})$ the optimal action in hindsight. In particular, $a^*_t\in\{a_1,a_2\}$. Now define
\begin{equation*}
    \Bcal = \bigcup_{q=1}^{K_0} \left\{t\in \Scal_q: \forall t'\in\bigcup_{q'<q} \Scal_{q'}: X_{t'}=X_t, t\notin Exp_{q'} \right\}.
\end{equation*}
These are times such that we never explored the action $a_2$. In particular, on $\Gcal^p$, the learner incurs an average regret of at least $\frac{1}{8}$ on these times since action $a_2$ would be optimal with probability $\frac{1}{2}$ with a reward excess $\frac{1}{4}$ over action $a_1$. Therefore,
\begin{align*}
    \Ebb\left[\frac{1}{T_p^{\hat k}}\sum_{t=1}^{T^{\hat k}_p} r_t^{\mathbf b,\mb c}(a^*_t)-r_t^{\mathbf b,\mb c}(\hat a_t) \mid \Gcal^p\right] &\geq \Ebb\left[\frac{1}{T_p^{\hat k}}\sum_{T^{\hat k-K_p}_p<t\leq T^{\hat k}_p} r_t^{\mathbf b,\mb c}(a^*_t)-r_t^{\mathbf b,\mb c}(\hat a_t) \mid \Gcal^p\right] \\
    &\geq \frac{1}{8} \Ebb\left[\frac{|\Bcal|}{T_p^{\hat k}} \mid \Gcal^p\right].
\end{align*}
Now denote by $T^*_p = |\{t\leq T^{\hat k}_p: X_t\in A_p\}\cap \Tcal^p|.$ Recall that because $\Fcal^p$ and $\hat l$ are independent from $\Ecal$, by Eq~\eqref{eq:def_hat_k}, we have $\Ebb\left[\frac{T^*_p}{T_p^{\hat k}}\mid\Gcal^p\right]=\left[\frac{T^*_p}{T_p^{\hat k}}\mid \Ecal\right]\geq \frac{\epsilon}{16}$. By construction, we have $|\Bcal| + \sum_{q=1}^{K_p} (K_p-q+1)E_q + K_p T_p^{\hat k-K_p-\tilde K_p} \geq T^*_p - T^{\hat k-K_p}_p$. Thus,
\begin{align*}
    \Ebb\left[\frac{1}{T_p^{\hat k}}\sum_{t=1}^{T^{\hat k}_p} r_t^{\mathbf b,\mb c}(a^*_t)-r_t^{\mathbf b,\mb c}(\hat a_t) \mid \Gcal^p\right] &\geq \frac{\epsilon}{2^7} -\frac{K_p}{4}\Ebb\left[ \frac{T_p^{\hat k-K_p-\tilde K_p}}{T_p^{\hat k}}\right] \\
    &\quad- \frac{\beta_p}{2} \sum_{q=1}^{K_p}(K_p-q+1)(1+2\alpha_p)^{(q-1)K_p}4^q\\
    &\geq \frac{\epsilon}{2^7} -\frac{\beta_p K_p}{16} - \frac{\beta_p}{2} (1+2\alpha_p)^{(K_p-1)K_p}4^{K_p+1}\\
    &\geq \frac{\epsilon}{2^8}.
\end{align*}
Recall that by construction $\Pbb[\hat l\leq l^p\mid \Fcal^p]\geq \frac{1}{2}$. Also, $\Pbb[\Fcal^p]=1$ and both these events are independent from $\Ecal$, hence , letting $T^p = T^{k(l^p)}_p$ we have
\begin{equation*}
    \Ebb\left[\sup_{T^{p-1}<T\leq T^p}\frac{1}{T}\sum_{t=1}^T r_t^{\mathbf b,\mb c}(a^*_t)-r_t^{\mathbf b,\mb c}(\hat a_t) \mid \Ecal \right] \geq \frac{1}{2}
    \Ebb\left[\frac{1}{T^{\hat k}_p}\sum_{t=1}^{ T^{\hat k}_p} r_t^{\mathbf b,\mb c}(a^*_t)-r_t^{\mathbf b,\mb c}(\hat a_t) \mid \Gcal^p \right] \geq  \frac{\epsilon}{2^9}.
\end{equation*}
This ends the construction of the sequence $T^p$. Then, for any binary sequences $\mb b$ and $\mb c$ we introduce slightly different rewards $(\tilde r_t^{\mb b,\mb c})_{t\leq T^p}$ as follows: for $t\leq T^{j(p'-1)}$, $\tilde r_t^{\mb b,\mb c}=r_t$, for $T^{j(p'-1)}<t\leq T_p^{k(0)}$ let $\tilde r_t^{\mb b,\mb c}=0$. For $t>T^{k(0)}_p$ let $l\geq 1$ such that $T^{k(l-1)}_p<t\leq T^{k(l)}_p$ and $k> k(0)$ such that $T^{k-1}_p<t\leq T^k_p$. Then, we pose
\begin{equation*}
    \tilde r^{\mb b,\mb c}_t(a\mid x_{\leq t}) =\begin{cases}
        0 & \exists t'\leq T^{k(l-1)}_p: x_{t'}=x_t\\
        0 & \text{o.w. }  c_k=0,\\
        r_{\delta_{l^p},\mb b}(a\mid x_t) & \text{o.w. } c_k=1,\forall T^{k-1}_p<t'<t:x_{t'}\neq x_t,\\
        0 & \text{o.w. } c_k=1,\exists T^{k-1}_p<t'<t:x_{t'}= x_t,\\
    \end{cases}
\end{equation*}
for $a\in\Acal,x_{\leq t}\in\Xcal^t.$ The only difference with the previous oblivious rewards is that we use the same reward function $r_{\delta_{l^p},\mb b}$ across phases $(T^{k(l-1)}_p,T^{k(l)}_p]$ for $l\leq l^p$. Then, consider the following policy,
\begin{equation*}
    \pi^{\mb b}(x) = \begin{cases} 
        a_1& \text{if } b_u= 1, x\in P_u(\delta_{l^p})\cap A_p,\\
        a_2& \text{if } b_u= 0, x\in P_u(\delta_{l^p})\cap A_p\\
        \pi^*(x) &\text{if }x\in \bigcup_{i<p}A_i\\
        a_1 &\text{if }x\notin\bigcup_{i\leq p}A_i.
    \end{cases}
\end{equation*}
Note that by induction hypothesis on the rewards $r_t$ for $t\leq T^{j(p'-1)}$, using the rewards $\tilde{\mb r}^{\mb b,\mb r}$, $\pi^{\mb b}$ always selects the best action in hindsight for times $t\leq T^{j(p'-1)}$. Also, by construction, $\pi^{\mb b}$ also selects the best action in hindsight for times $T^{j(p'-1)}<t\leq T^p$. 

Similarly to before, suppose that $\mb b,\mb c$ are generated as independent i.i.d. $\Bcal(\frac{1}{2})$ processes. We now argue that on the event $\Ecal$, the learning process with rewards $\mb r^{\mathbf b,\mb c}$ until $T^p$ is stochastically equivalent to the learning process with rewards $\tilde {\mb r}^{\mb b,\mb c}$ until $T^p$. Indeed, these rewards only differ in that for different periods $(T^{k(l-1)}_p,T^{k(l)}_p]$, we may have reward $r_{\delta_l,\mb {b^l}}$ instead of $r_{\delta_{l^p},\mb b}$. However, on the event $\Ecal$, new instances always fall in portions where the reward of $a_1$ is still $\Bcal(\frac{1}{2})$ conditionally on the current available history. This holds for both reward sequences. Further, duplicates can only affect rewards during the same period $(T^{k(l-1)}_p,T^{k(l)}_p]$ by construction---if $x_t$ is a duplicate from a previous period, the reward function is $0$. Hence, even though for $\mb r^{\mathbf b, \mb c}$, we have distinct sequences $\mb{b^l}$, these are all consistent with a single sequence $\mb b$ based on a finer partition at scale $\delta_{l^p}$. Precisely, we have
\begin{align*}
    \Ebb_{\mb b ,\mb c}&\left[\Ebb_{\Xbb,\hat a}\left[ \sup_{T^{j(p'-1)}<T\leq T^p} \frac{1}{T} \sum_{t=1}^T \tilde r^{\mb b,\mb c}_t(\pi^{\mb b}(X_t)) -  \tilde r^{\mb b,\mb c}(\hat a_t)\mid \Ecal\right]\right] \\
    &= \Ebb_{\Xbb}\left[\Ebb_{\mb b ,\mb c}\Ebb_{\hat a}\left[\sup_{T^{j(p'-1)}<T\leq T^p} \frac{1}{T} \sum_{t=1}^T \tilde r^{\mb b,\mb c}_t(a^*_t)) -  \tilde r^{\mb b,\mb c}(\hat a_t) \mid \Xbb,\Ecal\right] \mid \Ecal\right] \\
    &= \Ebb_{\Xbb}\left[\Ebb_{\mathbf b,\mb c}\Ebb_{\hat a}\left[\sup_{T^{j(p'-1)}<T\leq T^p} \frac{1}{T} \sum_{t=1}^T  r^{\mathbf b,\mb c}_t(a^*_t) -  r^{\mathbf b,\mb c}(\hat a_t)\mid \Xbb,\Ecal\right] \mid \Ecal\right]\\
    &= \Ebb\left[ \sup_{T^{j(p'-1)}<T\leq T^p} \frac{1}{T} \sum_{t=1}^T  r^{\mathbf b,\mb c}_t(a^*_t) -  r^{\mathbf b,\mb c}(\hat a_t) \mid \Ecal\right] \geq \frac{\epsilon}{2^9}.
\end{align*}
As a result, there exists a specific realization of $\mb b$ and $\mb c$ such that
\begin{equation*}
    \Ebb_{\Xbb,\hat a}\left[ \sup_{T^{j(p'-1)}<T\leq T^p} \frac{1}{T} \sum_{t=1}^T  \tilde r^{\mb b,\mb c}_t(\pi^{\mb b}(X_t)) -  \tilde r^{\mb b,\mb c}(\hat a_t) \mid \Ecal\right] \geq \frac{\epsilon}{2^9}.
\end{equation*}
Hence, because $\Pbb[\Ecal^c]\leq \frac{\epsilon}{2^{10}}$, we obtain
\begin{equation*}
    \Ebb_{\Xbb,\hat a}\left[ \sup_{T^{j(p'-1)}<T\leq T^p} \frac{1}{T} \sum_{t=1}^T  \tilde r^{\mb b,\mb c}_t(\pi^{\mb b}(X_t)) -  \tilde r^{\mb b,\mb c}(\hat a_t)\right]\geq \frac{\epsilon}{2^9}\left(1-\frac{\epsilon}{2^{10}}\right) - \frac{\epsilon}{2^{10}} \geq  \frac{\epsilon}{2^{11}}.
\end{equation*}
Now for all $t\leq T^p$ we pose $r_t=\tilde r_t^{\mb b,\mb c}$, and complete the definition of $\pi^*$ by setting $\pi^*(x)=\pi^{\mb b}(x)$ on $\bigcup_{i\leq p}A_i$. Note that these definitions are consistent with the previously constructed rewards and the actions selected by the policy on $\bigcup_{i<p}A_i$. This ends the recursive construction of the rewards $\mb r = (r_t)_{t\geq 1}$ and the policy $\pi^*$ on $\bigcup_{i\geq 1}A_i$. We close the definition of $\pi^*$ by setting $\pi^*(x)=a_1$ for $x\notin\bigcup_{i\geq 1}A_i$ arbitrarily. The constructed policy $\pi^*$ is measurable because it is measurable on each $A_i$ for $i\geq 1$.

We now analyze the regret of the algorithm compared to $\pi^*$ for the rewards $(r_t)_t$. First, note that the rewards are deterministic and that $\pi^*$ is the optimal policy, i.e., which always selects the best arm in hindsight. Also, if $\mb b,\mb c$ denote the  realizations used in the iteration $p=j(p')$ of the above recursion, for any $t\leq T^p$ we have $r_t = \tilde r_t^{\mb b,\mb c}$. As a result,
\begin{equation*}
    \Ebb\left[ \sup_{T^{j(p'-1)}<T\leq T^{j(p')}} \frac{1}{T} \sum_{t=1}^T  r_t(\pi^b(X_t)) -  r_t(\hat a_t)\right]\geq \frac{\epsilon}{2^{11}}.
\end{equation*}
Now by Fatou's lemma, we have
\begin{align*}
    \Ebb&\left[\limsup_{T\to\infty} \frac{1}{T} \sum_{t=1}^T  r_t(\pi^b(X_t)) -  r_t(\hat a_t)\right] \\
    &= 
    \Ebb\left[\limsup_{p'\to\infty} \sup_{T^{j(p'-1)}}<T\leq T^{j(p')} \frac{1}{T} \sum_{t=1}^T  r_t(\pi^b(X_t)) -  r_t(\hat a_t)\right]\\
    &\geq \limsup_{p\to\infty}\Ebb\left[\sup_{T^{j(p'-1)}}<T\leq T^{j(p')} \frac{1}{T} \sum_{t=1}^T  r_t(\pi^b(X_t)) -  r_t(\hat a_t)\right]\\
    &\geq \frac{\epsilon}{2^{11}}.
\end{align*}
As a result, $f_\cdot$ is not consistent on the oblivious rewards $(r_t)_t$ under $\Xbb$, which contradicts the hypothesis that $f_\cdot$ is universally consistent under $\Xbb$. This ends the proof of the theorem.
\end{proof}

Recall that the condition $\Ccal_2$ is necessary for universal learning because this is already the case for noiseless online learning \cite{hanneke:21} and is also a sufficient for universal learning in noiseless online learning \cite{blanchard:22a}, online learning with adversarial responses \cite{blanchard:22b} and stationary contextual bandits \cite{blanchard:22e}. In the next proposition, we show that our new necessary condition $\Ccal_4$ is a stronger condition than $\Ccal_2$.

\begin{proposition}\label{prop:condition_5_stronger_2}
Let $\Xcal$ be a metrizable separable Borel space. Then, $\Ccal_4\subset\Ccal_2$.
\end{proposition}

\begin{proof}
Suppose that $\Xbb\notin\Ccal_2$, then there exists a sequence of disjoint sets $(A_i)_{i\geq 1}$ and $\epsilon>0$ such that $\Ebb[\limsup_{T\to\infty}\frac{1}{T}|\{i\geq 1,A_i\cap\Xbb_{\leq T}\neq\emptyset\}| ] \geq \epsilon.$ We now let $B_i = \bigcup_{j\geq i}A_j$. We define $\bar\Tcal = \{t\geq 1:\forall t'<t,X_{t'}\neq X_t\}$ the set of new instances times. Then, for any $i\geq 1$,
\begin{align*}
    \Ebb\left[\limsup_{T\to\infty}\frac{1}{T}\sum_{t\leq T,t\in\Tcal^i}\1_{B_i}(X_t)\right] 
    &\geq \Ebb\left[\limsup_{T\to\infty}\frac{1}{T}\sum_{t\leq T,t\in\bar \Tcal}\1_{B_i}(X_t)\right]\\
    &\geq \Ebb\left[\limsup_{T\to\infty}\frac{|\{j\geq i: A_j\cap\Xbb_{\leq T}\neq\emptyset \}|}{T}\right]\\
    &= \Ebb\left[\limsup_{T\to\infty}\frac{|\{j\geq 1: A_j\cap\Xbb_{\leq T}\neq\emptyset \}|}{T}\right]\\
    &\geq \epsilon.
\end{align*}
This holds for all $i\geq 1$ and $B_i\downarrow\emptyset$. Hence, the second property of \cref{prop:C5_equivalent_form} implies $\Xbb\notin\Ccal_4$.
\end{proof}

In fact, $\Ccal_4$ is a strictly stronger condition than $\Ccal_2$ provided that $\Xcal$ admits a non-atomic probability measure. More precisely, in the next result, we explicitly construct a process $\Xbb\in\Ccal_2\setminus\Ccal_4$ which does not admit universal learning even in the memoryless setting. As a result, for memoryless, oblivious, prescient and online rewards, one cannot universally learn all $\Ccal_2$ processes, while this was achievable for stationary rewards. Thus having adversarial partial-feedback on the losses of each action strictly reduces the set of learnable processes $\Ccal_{online}\subset \Ccal_{oblivious}\subset \Ccal_{memoryless}\subsetneq \Ccal_2$.

\begin{theorem}\label{thm:learnable_process_smaller_C2}
Let $\Xcal$ be a metrizable separable Borel space such that there exists a non-atomic probability measure on $\Xcal$, and a finite action space $\Acal$ with $|\Acal|\geq 2$. Then, $\Ccal_4\subsetneq\Ccal_2$ and the set of learnable processes also satisfies $\Ccal_{memoryless} \subsetneq\Ccal_2$.
\end{theorem}

Before proving this result, we present a lemma which allows to have a countable sequence of non-atomic measures with disjoint support.

\begin{lemma}\label{lemma:disjoint_non-atomic}
Let $\Xcal$ be a metrizable separable Borel space such that there exists a non-atomic probability measure on $\Xcal$. Then, there exists a sequence of disjoint non-empty measurable sets $(A_i)_{i\geq 0}$ and probability measures $(\nu_i)_{i\geq 0}$ on $\Xcal$ such that $\nu_i(A_i)=1$.
\end{lemma}

\begin{proof}
Let $\rho$ denote the metric on $\Xcal$. First, let $(x^i)_{i\geq 1}$ be a dense sequence on $\Xcal$. For any $x\in\Xcal$ and $r>0$ we denote by $B(x,r)=\{x'\in\Xcal:\rho(x,x')<\delta\}$ the open ball centered at $x$ of radius $r$. Then, for any $\delta>0$, we define the partition $\Pcal(\delta) = (P_i(\delta))_{i\geq 1}$ by $P_i(\delta) = B(x^i,\delta)\setminus \bigcup_{j<i} B(x^j,\delta)$.

Let $\mu_{-1}$ a non-atomic probability measure on $\Xcal$. We construct the disjoint measures and sets recursively. We pose $B_0=\Xcal$. Suppose for $p\geq 1$ that we have constructed disjoint sets $(A_i)_{i\leq p-1}$, disjoint with $B_{p-1}$, as well as non-atomic probability measures $(\nu_i)_{i\leq p-1}$ and $\mu_{p-1}$ satisfying $\nu_i(A_i)=1$ for $i\leq p-1$ and $\mu_{p-1}(B_{p-1})=1$. Now let $Z_1,Z_2\sim\mu_{p-1}$ two independent random variables with distribution $\mu_{p-1}$. Because $\mu_{p-1}$ is non-atomic, $Z_1\neq Z_2$ almost surely. Thus, there exists $\delta_p>0$ such that
$\Pbb[\rho(Z_1,Z_2)\leq \delta_p]\leq \frac{1}{2}$. As a result, with probability at least $\frac{1}{2}$, $Z_1$ and $Z_2$ fall in distinct sets of the partition $\Pcal(\delta_p)$. Hence, there exists at least two indices $i<j$ such that $\Pbb[Z_1\in P_i(\delta_p)],\Pbb[Z_2\in P_j(\delta_p)]>0$. We then pose $A_p = B_{p-1}\cap P_i(\delta_p)$ and $B_p = B_{p-1} \cap P_j(\delta_p)$. Because $\mu_{p-1}(B_{p-1})=1$, we have $\mu_{p-1}(A_p) = \mu_{p-1}(P_i(\delta_p))>0$. Similarly, $\mu_{p-1}(B_p)>0$. Hence, we can consider the probability measure $\nu_p$ of $\mu_{p-1}$ conditionally on $A_p$ (i.e. $\nu_p(A) = \frac{\mu_{p-1}(A\cap A_p)}{\mu_{p-1}(A_p)}$ for all measurable $A$). Similarly, let $\mu_p$ the probability measure of $\mu_{p-1}$ conditionally on $B_p$. Both are non-atomic because the original measure $\mu_{p-1}$ is non-atomic. This ends the recursion and the proof of the lemma.
\end{proof}
We are now ready to prove the theorem.

\begin{proof}[ of \cref{thm:learnable_process_smaller_C2}]
Fix $a_1,a_2\in\Acal$ two distinct actions. Let $(x^i)_{i\geq 1}$ be a dense sequence of $\Xcal$ and denote by $B(x,r)$ denotes the open ball centered at $x\in\Xcal$ with radius $r>0$. Using, \cref{lemma:disjoint_non-atomic}, let $(A_i)_{i\geq 0}$ disjoint measurable sets together with non-atomic probability measures $(\nu_i)_{i\geq 0}$ such that $\nu_i(A_i)=1$. We then fix $x_0\in A_0$ (we will not use the set $A_0$ any further and from now will only reason on the sets $(A_i)_{i\geq 1}$) and for $i\geq 1$, we define $S_i = \{k\geq 1:k\equiv 2^{i-1}\bmod 2^i\}$. Then let $\Zbb^i$ for $i\geq 1$ be independent processes where $\Zbb^i$ is an i.i.d. process following the distribution $\nu_i$. We now construct a process $\Xbb$ on $\Xcal$. For any $k\geq 1$, let $T_k = 2^k k!$, $n_i = 2^{\lfloor \log_2 i\rfloor}$ for $i\geq 1$, and $l_k = \sum_{l\in S_i, l<k} \frac{T_k}{n_i}$, where $k\equiv 2^{i-1}\bmod 2^i$. For any $t\geq 1$, we pose
\begin{equation*}
    X_t = \begin{cases}
    Z^i_{l_k + r}& \text{if } T_k\leq t<2T_k, k\equiv 2^{i-1}\bmod 2^i, t-T_k \equiv r\bmod \frac{T_k}{n_i}, 1\leq r\leq \frac{T_k}{n_i}, \\
    x_0 &\text{otherwise}.
    \end{cases}
\end{equation*}
This ends the construction of $\Xbb$. We now argue that $\Xbb\in\Ccal_2$. Let $(B_l)_{l\geq 1}$ be a sequence of disjoint measurable sets of $\Xcal$. Because $\Zbb^i$ is an i.i.d. process for any $i\geq 1$, the event $\Ecal_i$ where $|\{l:\Zbb_{\leq T}^i\cap B_l\neq\emptyset\}|=o(T)$ has probability one. Now define $\Ecal = \bigcap_{i\geq 1}\Ecal_i$, which has probability one by the union bound. Fix $\epsilon>0$ and $i^*=\lceil\frac{2}{\epsilon}\rceil$ so that $\epsilon\leq \frac{1}{n_{i^*}}$. On the event $\Ecal$ for any $i\leq i^*$ there exists $T_i$ such that for any $T\geq T_i$ we have $|\{l:\Zbb_{\leq T}^i\cap B_l\neq\emptyset\}|\leq \frac{\epsilon}{2^i}T$. Now let $T^0=\max_{i\leq i^*}T_in_i$. Then, on $\Ecal$, for any $T\geq T^0$,
\begin{align*}
    |\{l: \Xbb_{\leq T}\cap B_l\neq\emptyset\}|&\leq 1 + \sum_{i=1}^{i^*} |\{l: \Zbb^i_{\leq \lfloor T/n_i\rfloor}\cap B_l\neq\emptyset\}|\\
    &\quad + |\{l: \exists t\leq T: X_t\in B_l, T_k\leq t<2 T_k, k\equiv 0\bmod 2^{i^*}\}|\\
    &\leq 1+ \epsilon T + |\{X_t,\quad  t\leq T,T_k\leq t<2 T_k, k\equiv 0\bmod 2^{i^*}\}|\\
    &\leq 1+ \epsilon T + \frac{T}{n_{i^*}} +\frac{T}{n_{i^*}}\\
    &\leq 3\epsilon T + 1.
\end{align*}
In the first inequality, the additional $1$ is due to the visit of $x_0$, and in the third inequality, we used the fact that in a phase $i>i^*$, each point is duplicated $n_i\geq n_{i^*}$ times. This yields a term $\frac{T}{n_{i^*}}$. The second term $\frac{T}{n_{i^*}}$ in the third inequality is due to boundary effects for times close to $T$, the worst-case scenarios being attained for $T$ of the form $T_k(1+\frac{1}{n_i})$. As a result, on $\Ecal$, we have $\limsup_{T\to\infty}\frac{1}{T} |\{l: \Xbb_{\leq T}\cap B_l\neq\emptyset\}|\leq 3\epsilon$, which holds for any $\epsilon>0$. Thus, $\frac{1}{T} |\{l: \Xbb_{\leq T}\cap B_l\neq\emptyset\}|\to 0$ on $\Ecal$, which ends the proof that $\Xbb\in\Ccal_2$.

We now show that there does not exist an universally consistent algorithm under $\Xbb$ for memoryless rewards. One can easily check that $\Xbb\notin\Ccal_4$, since for any $i\geq 1$, we have
\begin{align*}
    \Ebb\left[\limsup_{T\to\infty}\frac{1}{T} \sum_{t\leq T,t\in\Tcal^{\lfloor \log_2 i\rfloor}} \1_{A_i}(X_t)\right]
    &\geq \Ebb\left[\limsup_{k\to\infty}\frac{\1_{S_i}(k)}{2T_k} \sum_{t\leq 2T_k,t\in\Tcal^{\lfloor \log_2 i\rfloor}} \1_{A_i}(X_t)\right]\\
    &\geq \Ebb\left[\limsup_{k\to\infty}\frac{\1_{S_i}(k)}{2}\right] \geq \frac{1}{2}.
\end{align*}
This already shows that $\Ccal_{online}\subset\Ccal_{oblivious}\subset\Ccal_4\subsetneq\Ccal_2$. However, we will show a stronger statement that $\Xbb\notin\Ccal_{memoryless}$. The proof uses the same techniques as \cref{thm:condition_5_necessary}, but leverages the fact that the phases $S^i$ are deterministic and instances from previous phases $[T_k,2T_k)$ do not appear in future phases.
By contradiction, suppose that $f_\cdot$ is a universally consistent learning rule. We will refer to its decision at time $t$ as $\hat a_t$ for simplicity. We will construct recursively rewards $(r_t)_{t\geq 1}$ on which this algorithm is not consistent, as well as a policy $\pi^*:\Xcal\to\Acal$ compared to which the algorithm has high regret. The rewards and policy are constructed recursively together with an increasing sequence of times $(T^p)_{p\geq 1}$ and indices $(i_p)_{p\geq 1}$ with $i_1=1$ such that after the $p-$th iteration of the construction process, the rewards $r_t(a\mid \cdot)$ have been defined for all $t\leq T^p$ and the policy $\pi^*(\cdot)$ has been defined $\bigcup_{i< i_p}A_i$. The rewards will be deterministic and stationary, hence we may omit the subscript $t$. Suppose that we have performed $p-1$ iterations of this construction for $p\geq 1$. We will drop the subscripts $p$ for simplicity and simply assume that we have defined the reward $r(a\mid \cdot)$ and the value of the policy $\pi^*(\cdot)$ on $\bigcup_{j<i} A_j$ for some $i\geq 1$ ($i=i_p$). We now construct the rewards on $A_i$. To do so, we will first introduce other memoryless rewards. For any $k\in S_i$, because $\nu_i$ is non-atomic, there exists $\delta_k$ such that
\begin{equation*}
    \Pbb\left[ \min_{1\leq u<v\leq l_k+\frac{T_k}{n_i}} \rho(Z^i_u,Z^i_v) \leq \delta_k\right]\leq 2^{-k-5}.
\end{equation*}
Then, let $\Ecal^i$ be the event when for all $k\in S_i$, we have $\min_{1\leq u<v \leq l_k+\frac{T_k}{n_i}} \rho(Z^i_u,Z^i_v) > \delta_k$, and $\Zbb^i$ takes values in $A_i$ only---this is almost sure since $\nu_i(A_i)=1$. By the union bound, $\Pbb[\Ecal^i]\geq 1-\frac{1}{32}$. Now for $\delta>0$ and $u\geq 1$, define the sets $P_u(\delta) = (A_i\cap B(x^u,\delta))\setminus\bigcup_{v<u} B(x^v,\delta)$ which form a partition of $A_i$. For any $\delta>0$ and sequence $\mb b =(b_u)_{u\geq 1}$ in $\{0,1\}$ we consider the following deterministic rewards
\begin{equation*}
    r_{\delta,\mb b}(a\mid x) = \begin{cases}
        b_u & a=a_1,\; x\in P_u(\delta),\\
        \frac{3}{4} &a=a_2,\\
        0 & a\notin\{a_1,a_2\},
    \end{cases}  \text{if } x\in A_i,\quad\quad 
    r_{\delta,\mb b}(a\mid x) = r(a\mid x)  \text{ if } x\in \bigcup_{j< i}A_j,
\end{equation*}
and $r_{\delta,\mb b}(\cdot \mid x)=0$ if $x\notin\bigcup_{j\leq i}$.
Now for any sequence of binary sequences $\mathbf b =(\mb b^k)_{k\in S_i}$ where $\mb b^k=(b^k_u)_{u\geq 1}$, we will consider the memoryless rewards $\mb r^{\mathbf b}$ defined as follows. For any $t\geq 2$, let $k\geq 1$ such that $T^k\leq t<T^{k+1}$, and $k'=\min\{l\in S_i: l\geq k\}$. We pose $r^{\mathbf b}_t = r_{\delta_{k'},\mb b^{k'}}$, and $r^{\mathbf b}_1=r^{\mathbf b}_2$. Now let $\mathbf b$ be generated such that all $\mb b^i$ are independent i.i.d. Bernouilli $\Bcal(\frac{1}{2})$ random sequences in $\{0,1\}$. Next, define $\pi_0:x\in\Xcal\mapsto a_2\in\Acal$, the policy which always selects arm $a_2$. Now fix any realization of $\mb r^{\mathbf b}$. Because $f_\cdot$ is universally consistent for memoryless rewards, it has in particular sublinear regret compared to $\pi_0$ under rewards $\mb r^{\mathbf b}$, i.e., almost surely $\limsup_{T\to\infty} \frac{1}{T}\sum_{t=1}^T r^{\mathbf b}_t(a_2\mid X_t) - r^{\mathbf b}_t(\hat a_t \mid X_t)) \leq 0$. The same arguments as in \cref{thm:no_optimistically_universal_lr} with Fatou's lemma give
\begin{equation*}
    \limsup_{T\to\infty}\Ebb\left[\frac{1}{T}\sum_{t=1}^T r^{\mathbf b}_t(a_2\mid X_t) - r^{\mathbf b}_t(\hat a_t \mid X_t) \mid \Ecal^i \right]\leq 0,
\end{equation*}
where the expectation is now also taken over $\mathbf b$. Therefore, with $\alpha_i:=\frac{1}{16\cdot 4^{n_i}}$, there exists $t_0$ such that for all $T\geq t_0$, we have $\Ebb\left[\frac{1}{T}\sum_{t=1}^{T} r^{\mathbf b}_t(a_2\mid X_t) - r^{\mathbf b}_t(\hat a_t \mid X_t) \mid \Ecal^i\right] \leq \frac{\alpha_i}{4}.$ In particular, there exists $k\in S_i$ such that $k\geq \frac{4}{\alpha_i}$ and $T_k\geq t_0$ and the above inequality holds for all $T_k\leq T<2T_k$. Then, using the same arguments as in the proof of \cref{thm:no_optimistically_universal_lr}, if $a^*_t$ denotes the best action in hindsight at time $t$, we have
\begin{equation*}
    \Ebb\left[ \sum_{t=T_k}^{2T_k-1} r^{\mathbf b}_t(a_t^*\mid X_t) - r^{\mathbf b}_t(\hat a_t\mid X_t) \mid \Ecal^i\right]\geq \frac{T_k}{16}.
\end{equation*}
For any binary sequence $\mb b$, we will write for conciseness $r^{\mb b}=r_{\delta_k,\mb b}$. We also define the following policy, restricted to instances in $A_i$:
\begin{equation*}
    \pi^{\mb b}:x\in A_i \mapsto \begin{cases} 
        a_1& \text{if } b_u= 1, x\in P_u(\delta_k),\\
        a_2& \text{if } b_u= 0, x\in P_u(\delta_k).
    \end{cases}
\end{equation*}
Now consider the case where $\mb b$ is an i.i.d. sequence of Bernouillis $\Bcal(\frac{1}{2})$. We argue that on the event $\Ecal^i$, the learning process before time $2T_k-1$ and under rewards $\mb r^{\mathbf b}$ is stochastically equivalent to the learning under stationary rewards $\mb r^{\mb b}:=(r^{\mb b})_{t\geq 1}$ before $2T_k-1$. Precisely, we have
\begin{align*}
    \Ebb_{\mb b \sim\Bcal(\frac{1}{2})}&\left[\Ebb_{\Xbb,\hat a}\left[ \sum_{t=T_k}^{2T_k-1} r^{\mb b}(\pi^{\mb b}(X_t)\mid X_t) -  r^{\mb b}(\hat a_t\mid X_t)\mid \Ecal^i\right]\right] \\
    &= \Ebb_{\Xbb}\left[\Ebb_{\mb b \sim\Bcal(\frac{1}{2})}\Ebb_{\hat a}\left[\sum_{t=T_k}^{2T_k-1} r^{\mb b}(\pi^{\mb b}(X_t)\mid X_t) -  r^{\mb b}(\hat a_t\mid X_t)\mid \Xbb,\Ecal^i\right] \mid \Ecal^i\right] \\
    &= \Ebb_{\Xbb}\left[\Ebb_{\mathbf b}\Ebb_{\hat a}\left[\sum_{t=T_k}^{2T_k-1} r^{\mathbf b}_t(a^*_t\mid X_t) - r^{\mathbf b}_t(\hat a_t\mid X_t)\mid \Xbb,\Ecal^i\right] \mid \Ecal^i\right]\\
    &= \Ebb\left[ \sum_{t=T_k}^{2T_k-1} r^{\mathbf b}_t(a^*_t\mid X_t) - r^{\mathbf b}_t(\hat a_t\mid X_t) \mid \Ecal^i\right]\\
    &\geq \frac{T_k}{16},
\end{align*}
where in the second inequality we used the fact that on the event $\Ecal^i$, until time $2T_k-1$ all distinct instances in $A_i$ fall in distinct sets of the partition $(P_u(\delta_k))_u$: for both rewards $\mb r^{\mb b}$ and $\mb r^{\mathbf b}$, the reward on a new instance $A_i$ is independent from the past and has the distribution $\Bcal(\frac{1}{2})$ for action $a_1$ and deterministic $\frac{3}{4}$ for action $a_2$. As a result, there exists a specific realization of $\mb b$ such that
\begin{equation*}
    \Ebb_{\Xbb,\hat a}\left[\sum_{t=T_k}^{2T_k-1} r^{\mb b}(\pi^{\mb b}(X_t)\mid X_t) - r^{\mb b}(\hat a_t\mid X_t) \mid \Ecal^i\right] \geq \frac{T_k}{16}.
\end{equation*}
Hence, because $\Pbb[(\Ecal^i)^c]\leq \frac{1}{32}$, we obtain
\begin{equation*}
    \Ebb_{\Xbb,\hat a}\left[\sum_{t=T_k}^{2T_k-1} r^{\mb b}(\pi^{\mb b}(X_t)\mid X_t) - r^{\mb b}(\hat a_t\mid X_t)\right]\geq \frac{T_k}{16}\left(1-\frac{1}{32}\right) - \frac{T_k}{32} \geq  \frac{2T_k-1}{2^7}.
\end{equation*}
Now denote $T^p = 2T_k-1$, and let $i_{p+1} =1+ \max\{j\geq i: \exists l\in S_i, T_l\leq T^i\} = 1 +\max\{j\geq i: T_{2^{j-1}}\leq T^i\}$. The index $i_{p+1}$ is chosen so that until time $T^p$, the process $\Xbb$ has not visited $\bigcup_{j\geq i_p}A_j$ yet. Note that this index is well defined since $T_k\to\infty$ as $k\to\infty$. We then pose $r(\cdot \mid x) = r^{\mb b}(\cdot \mid x)$ for all $x\in \bigcup_{i\leq j < i_{p+1}}A_j$. In particular, we have $r(a\mid x)=0$ for all $x\in\bigcup_{i_p<j<i_{p+1}}A_j$. Then pose
\begin{equation*}
    \pi^*(x) = \begin{cases}
        \pi^{\mb b}(x) & x\in A_i\\
        a_2 &x\in\bigcup_{i<j< i^{p+1}}A_j.
        \end{cases}
\end{equation*}
This ends the recursive construction of the reward $r$ and the policy $\pi^*$, i.e., we have constructed $r(\cdot\mid x)$ and $\pi^*(x)$ for all $x\in\bigcup_{i\geq 1} A_i$. We end the definition of the rewards by posing $r_t(\cdot\mid x) = 0$ and $\pi^*(x) = a_2$ if $x\notin \bigcup_{i\geq 1} A_i$. Note that $(r_t)_{t\geq 1}$ forms a valid sequence of rewards since by construction on each $A_i$ they are deterministic. Similarly, $\pi^*$ is measurable because it is measurable on each $A_i$.

We now analyze the regret of the algorithm compared to $\pi^*$ for the rewards $(r_t)_t$. First, note that the rewards are deterministic, time independent, and that $\pi^*$ is the optimal policy, i.e., which always selects the best arm in hindsight. Then, for any $p\geq 1$, we have
\begin{equation*}
    r(\cdot\mid x) = r^{\mb b}(\cdot\mid x), \quad \forall x\in \Xcal\setminus \bigcup_{i\geq i_{p+1}}A_i.
\end{equation*}
where $r^{\mb b}$ denotes the rewards defined at the $p$-th iteration of the construction process. Now recall that by construction, the sets $A_i$ visited by the process $\Xbb_{\leq T^p}$ all satisfy $i<i_{p+1}$, which is the first index for which the rewards would differ. As a result, we have
\begin{align*}
    \Ebb\left[ \frac{1}{T^p}\sum_{t=1}^{T^p} r(\pi^*(X_t)\mid X_t) - r(\hat a_t\mid X_t) \right] 
    &\geq \Ebb\left[ \frac{1}{T^p} \sum_{t=(T_p+1)/2}^{T^p} r(\pi^*(X_t)\mid X_t) - r(\hat a_t\mid X_t) \right]\\
    &= \Ebb\left[\frac{1}{T^p} \sum_{t=(T_p+1)/2}^{T^p} r^{\mb b}(\pi^{\mb b}(X_t)\mid X_t) - r^{\mb b}(\hat a_t\mid X_t) \right]\\
    &\geq \frac{1}{2^7},
\end{align*}
where in the first inequality we used the fact that $\pi^*$ always selects the best action in hindsight. Because this holds for any $p\geq 1$, we can use Fatou's lemma to obtain
\begin{multline*}
    \Ebb\left[\limsup_{T\to\infty} \frac{1}{T}\sum_{t=1}^T r_t(\pi^*(X_t)\mid X_t) - r_t(\hat a_t\mid X_t)\right]\\
    \geq \limsup_{T\to\infty}\Ebb\left[ \frac{1}{T}\sum_{t=1}^T r_t(\pi^*(X_t)\mid X_t) - r_t(\hat a_t\mid X_t)\right] \geq \frac{1}{2^7}.
\end{multline*}
As a result, $f_\cdot$ is not consistent on the stationary rewards $(r)_t$ under $\Xbb$, which ends the proof of the theorem.
\end{proof}

\subsubsection{A tighter necessary condition 6 for oblivious rewards}
\label{subsubsec:tighter_necessary_condition}

This section proves that $\Ccal_6$ is necessary for stochastic processes, which is tighter than the family $\Ccal_4$. We first prove the lemma on large deviations of the empirical measure in $\Ccal_1'$ processes.

\begin{proof}[ of \cref{lemma:uniform_deviations}]
Let $\epsilon>0$ and suppose by contradiction that for all $T\geq 1$ and $\delta>0$ there exists a measurable set $A(\delta;T)$ such that $\Ebb[\hat \mu_{\tilde \Xbb}(A(\delta;T))]\leq \delta$ and 
\begin{equation*}
    \Ebb\left[\sup_{T'> T}\frac{1}{T'}\sum_{t\leq T',t\in\Tcal}\1_{A(\delta;T)}(X_t)\right]>\epsilon.
\end{equation*}
We now construct by induction a sequence of sets $(A_i)_{i\geq 1}$ together with times $(T_i)_{i\geq 0}$ such that $T_0=0$. Now suppose that we have constructed $T_{i-1}$ for $i\geq 1$. We take $A_i = A(\epsilon2^{-i-2}; T_{i-1})$. Then, because $\Ebb[\hat\mu_{\tilde \Xbb}(A_i)]\leq \epsilon 2^{-i-2}$, by the dominated convergence theorem, there exists $T_i> T_{i-1}$ such that
\begin{equation*}
    \Ebb\left[\sup_{T> T_i}\frac{1}{T}\sum_{t\leq T, t\in\Tcal}\1_{A_i}(X_t)\right] \leq \frac{\epsilon}{2^{i+1}}.
\end{equation*}
This ends the construction of the sequences. For any $i\geq 1$, let $B_i = A_i\setminus\bigcup_{j<i} A_j$ and note that
\begin{align*}
    \Ebb&\left[\sup_{T> T_{i-1}}\frac{1}{T}\sum_{t\leq T, t\in\Tcal}\1_{A_i}(X_t)\right]\\
    &\leq \Ebb\left[\sup_{T> T_{i-1}}\frac{1}{T}\sum_{t\leq T, t\in\Tcal}\1_{B_i}(X_t)\right] + \sum_{j<i} \Ebb\left[\sup_{T> T_{i-1}}\frac{1}{T}\sum_{t\leq T, t\in\Tcal}\1_{A_j}(X_t)\right]\\
    &\leq \Ebb\left[\sup_{T> T_{i-1}}\frac{1}{T}\sum_{t\leq T, t\in\Tcal}\1_{B_i}(X_t)\right]  + \sum_{j<i} \Ebb\left[\sup_{T> T_{j}}\frac{1}{T}\sum_{t\leq T, t\in\Tcal}\1_{A_j}(X_t)\right]\\
    &\leq \Ebb\left[\sup_{T> T_{i-1}}\frac{1}{T}\sum_{t\leq T, t\in\Tcal}\1_{B_i}(X_t)\right] + \frac{\epsilon}{2}.
\end{align*}
By construction $\Ebb\left[\sup_{T> T_{i-1}}\frac{1}{T}\sum_{t\leq T, t\in\Tcal}\1_{A_i}(X_t)\right] >\epsilon$. Hence,  letting $C_i = \bigcup_{j\geq i}B_j$, we obtain that for any $j\geq i$,
\begin{equation*}
    \Ebb\left[\sup_{T> T_j}\frac{1}{T}\sum_{t\leq T, t\in\Tcal}\1_{C_i}(X_t)\right] \geq \Ebb\left[\sup_{T> T_j}\frac{1}{T}\sum_{t\leq T, t\in\Tcal}\1_{B_{j+1}}(X_t)\right] \geq \frac{\epsilon}{2}.
\end{equation*}
As a result, by the dominated convergence theorem we have $\Ebb[\hat\mu_{\tilde \Xbb}(C_i)]\geq \frac{\epsilon}{2}$.
Further, all sets $B_i$ are disjoint. But $C_i\downarrow\emptyset$, which contradicts the hypothesis that $\tilde\Xbb\in\Ccal_1'$. This ends the proof of the lemma.
\end{proof}

We recall the necessary definitions to introduce condition $\Ccal_6$. For a process $\Xbb\in\Ccal_4$, any $\epsilon>0$ and $T\geq 1$,
\begin{multline*}
    \delta^p(\epsilon;T) := \sup\left\{0\leq \delta\leq 1:\forall A\in\Bcal \text{ s.t. }
    \sup_l \Ebb[\hat\mu_{\Xbb^l}(A)]\leq \delta,\right.\\
    \left. \forall \tau\geq T \text{ online stopping time},\quad
     \Ebb\left[\frac{1}{2\tau}\sum_{\tau\leq t<2\tau , t\in\Tcal^p}\1_A(X_t)\right]\leq \epsilon \right\},
\end{multline*}
and $\delta^p(\epsilon) := \lim_{T\to\infty}\delta^p(\epsilon;T)>0$. We recall condition $\Ccal_6$.

\ConditionSix*

The main result of this section is that this condition is necessary for oblivious rewards.

\begin{theorem}\label{thm:condition7_necessary}
Let $\Xcal$ be a metrizable separable Borel space, and a finite action space $\Acal$ with $|\Acal|\geq 2$. Then, $\Ccal_{oblivious}\subset \Ccal_6$.
\end{theorem}

\begin{proof}
Fix $\Xbb\in\Ccal_4\setminus\Ccal_6$. By hypothesis, there exists $\epsilon>0$ such that $\delta^p(\epsilon) \to 0$ as $p\to\infty$. Let $(p(i))_{i\geq 1}$ be the set of increasing indices such that $\delta^{p(i)}(\epsilon)\leq  \epsilon 2^{-i-3}$. Similarly to the proof of \cref{thm:condition_5_necessary}, we suppose by contradiction that there is a universally consistent learning rule $f_\cdot$ under $\Xbb$ and we will construct by induction some rewards on which the learning rule is not consistent. We will denote by $\hat a_t$ the action selected by the learning rule at time $t$. Precisely, suppose that we have performed $i-1$ iterations of the construction process for some $i\geq 1$, and have constructed times $T^1,\ldots,T^{i-1}$ as well as rewards $(r_t)_{t\leq T^{i-1}}$, disjoint sets $A^1,\ldots,A^{i-1}$ satisfying
\begin{equation*}
    \sup_l \Ebb[\hat\mu_{\Xbb^l}(A^j)]\leq \epsilon2^{-j-2}
\end{equation*}
for all $j<i$, and a policy $\pi^*$ on $\bigcup_{j<i}A^i$. We will now focus on the times $\Tcal^{p(i)}$. For convenience, in the rest of the proof, when clear from context, we will write $p$ instead of $p(i)$.

First, by hypothesis, for any $1\leq j<i$, we have $\Ebb[\hat \mu_{\Xbb^p}(A^j)]\leq \epsilon 2^{-j-2}$. Thus, by the dominated convergence theorem, there exists $t(j)$ such that
\begin{equation*}
    \Ebb\left[ \sup_{T\geq t(j)}\frac{1}{T}\sum_{t\leq T, t\in\Tcal^p}\1_{A^j}(X_t)\right] \leq \frac{\epsilon}{2^{j+1}}.
\end{equation*}
Therefore, summing these equations yields
\begin{equation*}
    \Ebb\left[ \sup_{T\geq \max_{j<i} t(j)}\frac{1}{T}\sum_{t\leq T, t\in\Tcal^p}\1_{\bigcup_{j<i} A^j}(X_t)\right] \leq \frac{\epsilon}{2}.
\end{equation*}
We define $\tilde T^{i-1}= \max(T^{i-1},t(1),\ldots,t(i-1))$. Now by construction, $\delta^{p(i)}(\epsilon)\leq  \epsilon 2^{-i-3}$. Therefore, there exists $T_0\geq \tilde T^{i-1}$ such that for any $T\geq T_0$, we have $\delta^p(\epsilon;T)\leq \epsilon2^{-i-2}$. Now for $T\geq T_0$, let $A^i(T)\in\Bcal$ and $\tau^i(T)\geq T$ be a stopping time such that
\begin{equation*}
    \sup_l\Ebb[\hat\mu_{\Xbb^l}(A^i(T))]\leq \epsilon 2^{-i-2} \quad \text{and}\quad \Ebb\left[ \frac{1}{2\tau^i(T)} \sum_{\tau^i(T)\leq t <2\tau^i(T),t\in\Tcal^{p}} \1_{A^i(T)}(X_t)\right]>\epsilon.
\end{equation*}
Last, let $U(T)$ be such that
\begin{equation*}
    \Pbb[2\tau^i(T) > U(T)] \geq \frac{\epsilon}{2^{T+10}}.
\end{equation*}
Then, by the union bound, with probability at least $1-\epsilon2^{-10}$, for all $T\geq T_0$, we have $2\tau^i(T)\leq U(T)$. Denote by $\Hcal$ this event. Next, let $k_i= 2^{p}+1$, $\alpha_i=2^{-p-1}$, $\beta_i = \frac{\epsilon}{2^{10}(1+2\alpha_i)^{(k_i-1)k_i}4^{k_i}}$, $\tilde K_i = \left\lceil \frac{2}{\alpha_i}\log\frac{8}{\beta_i}\right\rceil$ and $M_i =\max( (1+2\alpha_i)^{\tilde K_i}, \frac{2^{10}}{\epsilon})$. We first construct by induction of increasing times $(T(l))_{l\geq 0}$ with $T(0)=M_i T_0$ and $T(l)\geq M_i U(T(l-1))$. For convenience, we use the notation $\tau^i_l=\tau^i(T(l))$, $A^i_l=A^i(T(l))\setminus \bigcup_{1\leq j<i} A^j$ for $l\geq 0$. Then, by construction, $\tau^i(T(l))\geq M_i U(T(l-1))$ and
\begin{align*}
    &\Ebb\left[ \frac{1}{2\tau^i_l}\sum_{\tau^i_l\leq t< 2\tau^i_l, t\in\Tcal^p} \1_{A^i_l}(X_t)\right]\\
    &\geq 
    \Ebb\left[\frac{1}{2\tau^i_l}\sum_{\tau^i_l\leq t< 2\tau^i_l, t\in\Tcal^p}  \1_{A^i(T(l))}(X_t)\right] - \Ebb\left[\frac{1}{2\tau^i_l}\sum_{\tau^i_l\leq t< 2\tau^i_l, t\in\Tcal^p} \1_{\bigcup_{j<i}A^j}(X_t)\right]\\
    &> \epsilon - \Ebb\left[ \sup_{T\geq \tilde T^{i-1}}\frac{1}{T}\sum_{t\leq T, t\in\Tcal^p}\1_{\bigcup_{j<i} A^j}(X_t)\right]\\
    &>\frac{\epsilon}{2}.
\end{align*}
 For any $l\geq 1$, let $\delta_l>0$ such that
\begin{equation*}
    \Pbb\left[\min_{1\leq t,t'\leq U(T(l)),X_t\neq X_{t'}} \rho(X_t,X_{t'}) \leq \delta_l\right] \leq \frac{\epsilon}{2^{l+10}}.
\end{equation*}
Let $\Ecal$ be the event when for all $l\geq 1$, we have $\min_{1\leq t,t'\leq U(T(l)),X_t\neq X_{t'}} \rho(X_t,X_{t'})>\delta_l$ and $\Hcal$ is satisfied. By the union bound, $\Pbb[\Ecal]\geq 1-\frac{\epsilon}{2^{9}}$. We now construct similar rewards to those in the proof of \cref{thm:condition_5_necessary}. Then, for any $\delta>0$ and $u\geq 1$, define the sets $P_u(\delta) = B(x^u,\delta)\setminus\bigcup_{v<u} B(x^v,\delta)$ where $(x^u)_{u\geq 1}$ is a dense sequence of $\Xcal$, which form a partition of $\Xcal$. For any binary sequence $\mb b=(b_u)_{u\geq 1}$ in $\{0,1\}$ define the deterministic rewards
\begin{equation*}
    r_{\delta,\mb b;l}(a\mid x) = \begin{cases} 
        b_u \1_{x\in A^i_l} & a=a_1, x\in  P_u(\delta),\\
        \frac{3}{4}\1_{x\in A^i_l} & a=a_2,\\
        0 & a\notin\{a_1,a_2\}.
    \end{cases}
\end{equation*}
Next, for any sequence of binary sequences $\mathbf b:=(\mb{b^l})_{l\geq 1}$, we construct the deterministic rewards $\mb r^{\mathbf b}$ as follows. First, for $t\leq T^{i-1}$, $r^{\mathbf b}_t = r_t$ the rewards already constructed. Also, for $T^{i-1}<t\leq U(T(0))$, we pose $r^{\mathbf b}_t = 0$. Next, observe that $\tau^i_l$ is an online stopping time. Therefore, for any $l\geq 0$, $U(T(l-1))<t< \tau^i_l$ or $2\tau^i_l\leq t\leq U(T(l))$, we pose $r^{\mathbf b}_t = 0$. Finally, for $\tau^i_l \leq t< 2\tau^i_l,U(T(l))$ and $k$ such that $T_p^{k-1}<t\leq T_p^k$, we pose
\begin{equation*}
    r^{\mathbf b}_t (a\mid x_{\leq t}) = \begin{cases}
        0 &\exists t'\leq U(T(l-1)):x_{t'}=x_t,\\
        0& \text{o.w., } \exists T_p^{k-1}<t'\leq t: x_{t'} = x_t,\\
        r_{\delta_l,\mb {b^l};l}(a\mid x_t) & \text{o.w., } \forall T_p^{k-1}<t'\leq t: x_{t'} \neq x_t,
    \end{cases}
\end{equation*} 
for any $a\in\Acal$ and $x_{\leq t}\in\Xcal^t$. Now generate $\mathbf b$ as independent i.i.d. Bernouilli $\Bcal(\frac{1}{2})$ processes. We now compare the predictions of the learning rule compared to the constant policy which selects action $a_2$. Because the learning rule is consistent under any rewards ${\mb r}^{\mathbf b}$ for any realization $\mathbf b$, and because $\Pbb[\Ecal]>0,$ taking the expectation over $\mathbf b$, we obtain
\begin{equation*}
    \Ebb\left[\limsup_{T\to\infty} \frac{1}{T}\sum_{t=1}^T r_t^{\mathbf b}(a_2) - r_t^{\mathbf b}(\hat a_t) \mid\Ecal\right] \leq 0.
\end{equation*}
Next, we use the dominated convergence theorem to find $l^i\geq 1$ such that
\begin{equation*}
    \Ebb\left[\sup_{T\geq T(l_i)/2} \frac{1}{T}\sum_{t=1}^T r_t^{\mathbf b}(a_2) - r_t^{\mathbf b}(\hat a_t) \mid\Ecal\right] \leq \frac{\beta_p}{4}.
\end{equation*}
We now define $A^i=A^i_{l^i}$, $T^i=U(T(l^i))$ and focus on the period $[\tau^i_l,2\tau^i_l)$. Let $\hat k= \max\{k:\tau^i_l\geq T^k_p\}$. Then, $[\tau^i_l,2\tau^i_l)\subset [T^{\hat k}_p,T^{\hat k+2^p+1}_p)$ and we construct the following sets
\begin{equation}
    \Scal_q = \{T_p^{\hat k+q-1}<t\leq T_p^{\hat k+q}: X_t\in A^i\}\cap \Tcal^p ,\quad 1\leq q\leq 2^p+1=k_i.
\end{equation}
We also define $Exp_q$ the exploration steps of arm $a_1$ during $\Scal_q$.
\begin{multline*}
    Exp_q = \left\{t\in\Scal_q: \hat a_t = a_1\text{ and }\forall t'\in \bigcup_{q'<q}\Scal_{q'}:X_{t'}=X_t,\; \hat a_{t'}\neq a_1\right\} \\
    \setminus \{t:\exists t'\leq U(T(l^i-1)), X_{t'}=X_t\},
\end{multline*}
and $E_q=|Exp_q|$. The same arguments as in \cref{thm:condition_5_necessary} show that for all $1\leq q\leq k_1$, we have $\Ebb\left[\frac{E_q}{T_p^{\hat k+k_i}}\mid\Ecal\right]\leq 4^{q+1}(1+2\alpha_i)^{(k_i-1)k_i}\beta_p$. For any $t\geq 1$, let $a_t^*$ be the optimal action in hindsight and define
\begin{equation*}
    \Bcal_q=\bigcup_{q\leq \hat q}\left\{ t\in\Scal_{q}:\forall t'\in\bigcup_{q'<q}\Scal_{q'}:X_{t'}=X_t,t\notin Exp_{q'} \right\},
\end{equation*}
the times such that we never explored action $a_2$, before time $T_p^{\hat k+q}$. As in the proof of \cref{thm:condition_5_necessary}, for times in $\Bcal$, the learner incurs an average regret at least $\frac{1}{8}$. Therefore,
\begin{equation*}
    \Ebb\left[\frac{1}{T_p^{\hat k+k_i}} \sum_{t=1}^{T_p^{\hat k+k_i}} r_t^{\mathbf b}(a_t^*)-r_t^{\mathbf b}(\hat a_t) \mid\Ecal \right] \geq \frac{1}{8} \Ebb\left[\frac{|\Bcal_q|}{T_p^{\hat k+k_i}} \mid\Ecal\right].
\end{equation*}
Finally, let $T_p^*=|\{t\leq T_p^{\hat k+k_i}:X_t\in A^i\}\cap\Tcal^p|$. Noting that we have $\Ebb\left[\frac{T_p^*}{T_p^{\hat k+k_i}}\mid\Ecal\right]\geq \frac{1}{2}\Ebb\left[\frac{T_p^*}{2\tau^i_l}\mid\Ecal\right] \geq  \frac{\epsilon}{4}\geq \frac{\epsilon}{16}$, the same arguments as in the original proof give directly
\begin{equation*}
    \Ebb\left[\frac{1}{T_p^{\hat k+k_i}} \sum_{t=1}^{T_p^{\hat k+k_i}} r_t^{\mathbf b}(a_t^*)-r_t^{\mathbf b}(\hat a_t) \mid\Ecal\right] \geq \frac{\epsilon}{2^8}. 
\end{equation*}
As a result, there exists a realization of $\mathbf b$ such that the above equation holds for this specific realization. We then pose $r_t=r^{\mathbf b}_t$ for all $t\leq T^i$ and define a policy $\pi^i$ on $A^i$ as follows,
\begin{equation*}
    \pi^i(x) = \begin{cases}
    a_1 &\text{if }b_u^l=1,x\in P_u(\delta_{l^i})\cap A^i,\\
    a_2 &\text{if }b_u^l=0,x\in P_u(\delta_{l^i})\cap A^i.
    \end{cases}
\end{equation*}
for any $x\in A^i$, which is possible because $A^i$ is disjoint from $\bigcup_{j<i}A^j$. Now observe that the policy selects the best action in hindsight during the interval $[T(l^i),U(T(l^i))$, irrespective on how it is defined outside of $A^i$. As a result, we have
\begin{align*}
    \Ebb&\left[ \sup_{T^{i-1}<T\leq T^i} \frac{1}{T} \sum_{t=1}^T  r_t(\pi^*(X_t)) -  r_t(\hat a_t)\mid \Ecal\right]\\
    &\geq \Ebb\left[  \frac{1}{T^{\hat k+k_i}} \sum_{t=1}^{T^{\hat k+k_i}}  r_t^{\mathbf b}(\pi^*(X_t)) -  r_t^{\mathbf b}(\hat a_t)\mid \Ecal\right]\\
    &\geq \Ebb\left[  -\frac{2U(T(l^i-1))}{T^{\hat k+k_i}} +\frac{1}{T^{\hat k+k_i}}  \sum_{t=1}^{T^{\hat k+k_i}}  r_t^{\mathbf b}(a^*_t) -  r_t^{\mathbf b}(\hat a_t)\mid \Ecal\right]\\
    &\geq -\frac{2}{M_i} + \frac{\epsilon}{2^8}\\
    &\geq \frac{\epsilon}{2^9}.
\end{align*}
This ends the recursive construction of the rewards. We close the definition of $\pi^*$ by setting $\pi^*(x)=a_1$ for $x\notin\bigcup_{i\geq 1}A^i$ arbitrarily. The constructed policy is measurable and we showed that for all $i\geq 1$,
\begin{equation*}
    \Ebb\left[ \sup_{T^{i-1}<T\leq T^i} \frac{1}{T} \sum_{t=1}^T  r_t(\pi^*(X_t)) -  r_t(\hat a_t)\right] \geq \frac{\epsilon}{2^9}.
\end{equation*}
Using Fatou's lemma, this shows that $\Ebb\left[ \limsup_{T\to\infty} \frac{1}{T} \sum_{t=1}^T \tilde r_t(\pi^*(X_t)) - \tilde r_t(\hat a_t)\right] \geq \frac{\epsilon}{2^9}.$ This ends the proof that $f_\cdot$ is not universally consistent under $\Xbb$ and ends the proof of the theorem.
\end{proof}

We now give an example of process $\Xbb\in\Ccal_4\setminus\Ccal_6$.

\begin{theorem}\label{thm:other_example}
For $\Xcal=[0,1]$ with usual topology, $\Ccal_6\subsetneq \Ccal_4$.
\end{theorem}

\begin{proof}
We construct a process $\Xbb$ on $[0,1]$ by phases $[2^l,2^{l+1})$ for $l\geq 0$. We set $X_1=0$ arbitrarily and divide phases by categories $S_p=\{l\geq 1: l\equiv 2^{p-1}\bmod 2^p\}$ for any $p\geq 1$. Next, for any $l\in S_p$, let
\begin{equation*}
    A_p(l) = \bigcup_{0\leq i<2^l}\left[\frac{i2^p}{2^{p+l}},\frac{i2^p+1}{2^{p+l}} \right].
\end{equation*}
Importantly, $A_p(l)$ has Lebesgue measure $2^{-p}$. Next, noting that $l\geq 2^{p-1}\geq p$, for $2^l\leq t<2^{l+1}$ we define
\begin{equation*}
    X_t =\begin{cases}
        \Ucal_t(A_p(l)) &2^l\leq t<2^l + 2^{l-p},\\
        X_{t'} &t\geq 2^l + 2^{l-p},2^l\leq t'<2^l + 2^{l-p}, t'\equiv t\bmod 2^{l-p}
    \end{cases}
\end{equation*}
where $\Ucal_t(A_p(l)) $ denotes a uniform random variable on $A_p(l)$ independent from all past random variables. The process on $S_p$ is constructed so that it has $2^p$ duplicates. This ends the construction of $\Xbb$.

We now show that $\Xbb\in\Ccal_4$. For convenience, for any $l\geq 1$, let $p(l)$ be the index such that $l\in S_{p(l)}$. Next, let $\Xbb^p:=(X_t)_{t\in\Tcal^p}$ for $p\geq 0$. we will show the stronger statement that for any measurable set $A\in\Bcal$, we have $\hat\mu_{\Xbb^p}(A) \leq \mu(A)\;(a.s.),$ where $\mu$ is the Lebesgue measure. To do so, fix $A\in\Bcal$ and $\epsilon>0$. Since $A$ is Lebesgue measurable, there exists a sequence of disjoint intervals $(I_k)_{k\geq 0}$ within $\Xcal=[0,1]$ such that $A\subset \bigcup_{k\geq 0} I_k$ and
\begin{equation*}
    \sum_{k\geq 0} \ell(I_k) \leq \mu(A)+\epsilon,
\end{equation*}
where $\ell(I)$ is the length of an interval $I$. Then, let $k_0$ such that $\sum_{k\geq k_0} I_k\leq \frac{\epsilon^2}{2^{p+1}}$ and pose $\ell_0=\min_{k<k_0} \ell(I_k)$. Then, for any $l\geq \max(2,\log_2 \frac{k_0}{\epsilon}):=l_0$, with $l\in S_q$,
\begin{align*}
    \frac{\mu(A\cap A_q(l))}{\mu(A_q(l))} &\leq \sum_{k< k_0}\frac{\mu(I_k\cap A_q(l))}{\mu(A_q(l))} +   2^q\mu\left(\bigcup_{k\geq k_0} I_k \right)\\
    &\leq \sum_{k< k_0}(\ell(I_k) + 2^{-l})+\epsilon^2 2^{q-p-1}\\
    &\leq \mu(A) + 2\epsilon + \epsilon^2 2^{q-p-1}.
\end{align*}
Let $q_0 = p+\log_2 \frac{1}{\epsilon}$. For any $l\geq l_0$ with $l\in\bigcup_{q<q_0}S_q$, we have $\frac{\mu(A\cap A_q(l))}{\mu(A_q(l))} \leq \mu(A)+3\epsilon$. Now for any $l\geq l_0$, if $l\in\bigcup_{q<q_0}S_q$, Hoeffding's inequality implies that for any $l\leq r\leq 2^{l-q}$,
\begin{equation*}
    \Pbb\left[\sum_{2^l\leq t<2^l+r}\1_A(X_t) \leq  r(\mu(A)+4\epsilon)\right]\geq 1-e^{-2\epsilon^2 r^2}\geq 1-e^{-2\epsilon^2 lr}.
\end{equation*}
Note that we always have $2^{l-q}\geq l$ since $l\geq 2^{q-1}$ and $l\geq 2$.
In particular, because we have $\sum_{r\geq 1}\sum_{l\geq 1}e^{-2\epsilon^2 lr}<\infty$, on an event $\Ecal(\epsilon)$ of probability one, there exists $\hat l\geq l_0$ such that the above equation holds for all $l\geq \hat l$ with $l\in\bigcup_{q<q_0}S_q$ and $l\leq r\leq 2^{l-q}$. Then, for $T\geq 2^{\hat l}$, letting $l(T)\geq 1$ such that $2^{l(T)}\leq T<2^{l(T)+1}$, we have
\begin{align*}
    \sum_{t\leq T, t\in\Tcal^p}\1_A(X_t)&= \sum_{l<l(T)}\min(2^{p(l)},2^p) \sum_{2^l\leq t<2^l+2^{l-p(l)}}\1_A(X_t) + \sum_{2^{l(T)}\leq t\leq T,t\in\Tcal^p}\1_A(X_t)\\
    &\leq  \sum_{l<l(T)}\epsilon 2^l \1[p(l)\geq q_0] +2^{\hat l} + \sum_{\hat l\leq l<l(T)}2^l(\mu(A)+4\epsilon) \1[p(l)< q_0]\\
    &\quad\epsilon 2^{l(T)}\1[p(l(T))\geq q_0] + [(T-2^{l(T)}+1)(\mu(A)+4\epsilon) + l(T)]\1[p(l(T))< q_0]\\
    &\leq 2^{\hat l}+l(T) + 2\epsilon 2^{l(T)} + (\mu(A)+4\epsilon)T\\
    &\leq 2^{\hat l} + \log_2 T + (\mu(A)+6\epsilon)T.
\end{align*}
where in the first inequality, we used the fact that for $q\geq q_0$, $2^p\leq \epsilon 2^q$. Further, the additional term $l(T)$ comes from the fact that the estimates on $\Ecal(\epsilon)$ held for $r\geq l$: writing $T=2^{l(T)} + u2^{l(T)-p(l(T))} + v$, we first use $\Ecal(\epsilon)$ with $r=2^{l(T)-p(l(T))}$, then with $r=\max(v,l(T))$. As a result, on $\Ecal(\epsilon)$, we have $\hat\mu_{\Xbb^p}(A)\leq \mu(A)+6\epsilon$. Thus, on $\bigcap_{j\geq 0}\Ecal(2^{-j})$ of probability one, we have $\hat\mu_{\Xbb^p}(A)\leq \mu(A)$, and this holds for all $p\geq 1$ and $A\in\Bcal$. Using this property, verifying the $\Ccal_4$ condition is straightforward. For disjoint measurable sets $A_i$, we have $\Ebb[\hat\mu_{\Xbb^i}(A_i)]\leq\mu(A_i)\to 0$ because $\sum_i \mu(A_i)\leq 1$.

We now show that $\Xbb\notin\Ccal_6$. First, on an event $\Fcal$ of probability one, all samples $\Ucal_t(A_p(l))$ are distinct. As a result, on $\Fcal$, except for the intended duplicates, all instances of $\Xbb$ are distinct. Thus, for any $l\in S_p$, and any $2^l\leq t<2^{l+1}$, we have $t\in\Tcal^p$. Hence, on $\Fcal$,
\begin{equation*}
    \frac{1}{2^{l+1}}\sum_{2^l\leq t< 2^{l+1},t\in\Tcal^p} \1_{A_p(l)}(X_t) \geq \frac{2^l}{2^{l+1}} = \frac{1}{2}.
\end{equation*}
In particular, this implies that
\begin{equation*}
    \Ebb\left[ \frac{1}{2^{l+1}}\sum_{2^l\leq t< 2^{l+1},t\in\Tcal^p} \1_{A_p(l)}(X_t)\right]\geq \frac{1}{2}.
\end{equation*}
However, $\Ebb[\hat\mu_{\Xbb^p}(A_p(l))]=\mu(A_p(l))=2^{-p}$. Therefore, using the trivial stopping time $\tau=2^l$, we showed $\delta^p(1/2;2^l)\leq 2^{-p}$. Because this holds for all $l\in S_p$ which is infinite, we have $\delta^p(1/2)\leq 2^{-p}$. Thus, $\delta^p(1/2)\to 0$ as $p\to\infty$. This shows that $\Xbb\notin\Ccal_6$ and ends the proof of the theorem.
\end{proof}

A more natural condition on processes than $\Ccal_6$ would be one that does not involve these stopping times $\tau$. In particular, for a process $\Xbb\in \Ccal_4$, we can define instead for any $\epsilon>0$ and $T\geq 1$,
\begin{multline*}
    \bar\delta^p(\epsilon;T) := \sup\left\{0\leq \delta\leq 1:\forall A\in\Bcal \text{ s.t. }
    \sup_l \Ebb[\hat\mu_{\Xbb^l}(A)]\leq \delta,\right.\\
    \left. \Ebb\left[\sup_{T'\geq T}\frac{1}{T}\sum_{t\leq T , t\in\Tcal^p}\1_A(X_t)\right]\leq \epsilon \right\}.
\end{multline*}
As before, $\bar \delta^p(\epsilon;T)$ is non-decreasing in $T$ and $\bar \delta^p(\epsilon):=\lim_{T\to\infty}\delta^p(\epsilon;T)>0$. We can then observe that $\bar\delta^p(\epsilon)$ is non-increasing. Similarly to $\Ccal_6$, we can then define the following condition.

\begin{restatable}{condition}{ConditionSeven}
    \label{con:C7}
    $\Xbb\in\Ccal_4$ and for any $\epsilon>0$, we have $ \lim_{p\to\infty} \bar \delta^p(\epsilon) >0.$
    Denote by $\Ccal_7$ the set of all processes $\Xbb$ satisfying this condition.
\end{restatable}

As a simple remark, we have the inclusion $\Ccal_7\subset\Ccal_6$, since if for any given process $\Xbb\in\Ccal_4$, set $A\in\Bcal$ and online stopping time $\tau\geq T$,
\begin{equation*}
    \Ebb\left[\frac{1}{2\tau}\sum_{\tau\leq t<2\tau , t\in\Tcal^p}\1_A(X_t)\right] \leq \Ebb\left[\sup_{T'\geq T}\frac{1}{T}\sum_{t\leq T , t\in\Tcal^p}\1_A(X_t)\right].
\end{equation*}

Unfortunately, for oblivious rewards, we were unable to prove that $\Ccal_7$ is a necessary condition. Indeed, for a process $\Xbb\in\Ccal_4$, time $T\geq 1$ and $\epsilon>0$, if
\begin{equation}\label{eq:high_deviation}
    \Ebb\left[\sup_{T'\geq T}\frac{1}{T}\sum_{t\leq T , t\in\Tcal^p}\1_A(X_t)\right] > \epsilon,
\end{equation}
it is in general not true that there exists an online stopping time $\tau\geq T$ such that
\begin{equation}\label{eq:high_deviation_with_stopping_time}
    \Ebb\left[\frac{1}{2\tau}\sum_{\tau\leq t<2\tau , t\in\Tcal^p}\1_A(X_t)\right] > \eta \epsilon,
\end{equation}
even for a fixed multiplicative tolerance $0<\eta<1$, which should be independent of $\epsilon>0$. Thus, it seems unlikely that $\Ccal_6 = \Ccal_7$ in general for spaces $\Xcal$ admitting a non-atomic probability measure.

However, if one considers a stronger type of adversary, we can show that $\Ccal_7$ becomes necessary for universal learning. Precisely, one can introduce \emph{prescient} rewards, that are stronger than oblivious rewards in that rewards are allowed to depend on the complete sequence $\Xbb$ instead of the revealed contexts to the learner $\Xbb_{\leq t}$ at step $t$. Formally, these are defined as follows.

\begin{definition}[Reward models]
\label{def:prescient_rewards}
    The reward mechanism is said to be \emph{prescient} if there are conditional distributions $(P_{r\mid a,\mb x_{t'\geq 1}})_{t\geq 1}$ such that $r_t$ given the selected action $a_t$ and the sequence of contexts $\Xbb$, follows $P_{r\mid a,\mb x_{t'\geq 1}}$.
\end{definition}

In this model, given a process $\Xbb\in\Ccal_4$, a time $T\geq 1$ and $\epsilon>0$ satisfying Eq~\eqref{eq:high_deviation}, finding a time $\tau\geq T$ (measurable with respect to the sigma-algebra $\sigma(\Xbb)$, i.e., conditionally on $\Xbb$) such that Eq~\eqref{eq:high_deviation_with_stopping_time} is satisfied becomes trivial even with $\eta=1$. Therefore, the same proof as for \cref{thm:condition7_necessary} shows that the last condition on stochastic processes is necessary for prescient rewards.

\begin{theorem}\label{thm:prescient_implies_C7}
Let $\Xcal$ be a metrizable separable Borel space, and a finite action space $\Acal$ with $|\Acal|\geq 2$. Then, $\Ccal_{prescient} \subset \Ccal_7$.
\end{theorem}

\subsubsection{Condition 5 is necessary for universal learning with online rewards}
\label{subsubsec:C5_necessary_online}

In this section, we show that condition $\Ccal_5$ is necessary for universal learning with online rewards, tightening the result on the necessity of condition $\Ccal_6$ from the previous section. In fact, in \cref{subsec:sufficient_conditions} we show that $\Ccal_5$ is also sufficient, which together with the result from this section shows that $\Ccal_5$ exactly characterizes universally learnable processes for online rewards. We recall that this is the strongest reward model that we consider in this paper and allows the reward adversary to also take into account the past actions selected by the learner. We first briefly recall the definition of condition $\Ccal_5$.

\ConditionExistsScaleRate*

Before proving our main result, we need the following lemma that gives an equivalent formulation of the class of processes $\Ccal_5$. Intuitively, it shows that if $\Xbb\notin\Ccal_5$, for any tentative rate to add duplicates---yielding the extended process $\tilde\Xbb$---we can uniformly lower-bound the proportion of failure for the $\Ccal_1'$ condition.

\begin{lemma}
\label{lemma:other_form_C5}
    Let $\Xcal$ be a metrizable separable Borel space and $\Xbb$ a stochastic process on $\Xcal$. The following are equivalent.
    \begin{itemize}
        \item $\Xbb\in\Ccal_5$,
        \item For any $\epsilon>0$, there exists an increasing sequence of integers $(T_i)_{i\geq 0}$ such that letting $\Tcal = \bigcup_{i\geq 0}\Tcal^i\cap\{t\geq T_i\}$, for any sequence $\{A_k\}_{k\geq 1}$ of measurable sets of $\Xcal$ with $A_k\downarrow \emptyset$,
        \begin{equation*}
            \lim_{k\to\infty}\Ebb[\hat\mu_{(X_t)_{t\in\Tcal}}(A_k)] \leq \epsilon.
        \end{equation*}
    \end{itemize}
\end{lemma}

\begin{proof}
    By definition of the condition $\Ccal_5$, it is immediate that $\Xbb\in\Ccal_5$ implies the second proposition. It remains to prove the converse. We then suppose that $\Xbb$ satisfies the second proposition. Denote by $(T_i(l))_{i\geq 0}$ the sequence obtained from the proposition by setting $\epsilon=2^{-l}$. Now defining
    \begin{equation*}
        T_i = \max_{j\leq i} T_i(j),
    \end{equation*}
    it then suffices to argue that the sequence $(T_i)_{i\geq 0}$ satisfies the requirements for the $\Ccal_5$ condition. We write $\Tcal = \bigcup_{i\geq 0}\Tcal^i\cap\{t\geq T_i\}$ and $\Tcal(l) = \bigcup_{i\geq 0}\Tcal^i\cap\{t\geq T_i(l)\}$ for any $l\geq 0$. Now fix $l\geq 0$, and note that for any $i\geq l$, one has $T_i\geq T_i(l)$. As a result,
    \begin{equation*}
        \bigcup_{i\geq l}\Tcal^i\cap\{t\geq T_i\}\subset \bigcup_{i\geq l}\Tcal^i\cap\{t\geq T_i(l)\}.
    \end{equation*}
    Next, note that because the sets $\Tcal^i$ are increasing in $i$, we have $\Tcal\setminus\bigcup_{i\geq l}\Tcal^i\cap\{t\geq T_i\} \subset \{t< T_l\}$. Therefore, for any measurable set $A\in\Bcal$, one has
    \begin{equation*}
        \hat\mu_{(X_t)_{t\in\Tcal}}(A) = \limsup_{T\to\infty}\frac{1}{T}\sum_{t\leq T, t\in\Tcal} \1_A(X_t)\leq \limsup_{T\to\infty}\frac{T_l}{T} + \frac{1}{T}\sum_{t\leq T, t\in\Tcal(l)} \1_A(X_t) = \hat\mu_{(X_t)_{t\in\Tcal(l)}}(A).
    \end{equation*}
    Thus, for any sequence of measurable sets $A_k\downarrow\emptyset$, one has
    \begin{equation*}
        \lim_{k\to\infty}\Ebb[\hat\mu_{(X_t)_{t\in\Tcal}}(A_k)] \leq \lim_{k\to\infty}\Ebb[\hat\mu_{(X_t)_{t\in\Tcal(l)}}(A_k)] \leq 2^{-l}.
    \end{equation*}
    Because this holds for all $l\geq 0$, we obtain $\lim_{k\to\infty}\Ebb[\hat\mu_{(X_t)_{t\in\Tcal}}(A_k)]=0$ and the lemma is proved.
\end{proof}

We are now ready to prove the following theorem.

\begin{theorem}
\label{thm:C5_necessary_online}
    Let $\Xcal$ be a metrizable separable Borel space, and a finite action space $\Acal$ with $|\Acal|\geq 2$. Then, $\Ccal_{online}\subset \Ccal_5$.
\end{theorem}

\begin{proof}
    Fix $\Xbb\notin\Ccal_5$. If $\Xbb\notin \Ccal_4$, we already proved that (even for oblivious rewards) universal learning is not achievable. We therefore suppose that $\Xbb\in\Ccal_4$ and suppose by contradiction that there is a universally consistent learning rule $f_\cdot$ under $\Xbb$. We will construct by induction some online rewards on which the learning rule is not consistent. For convenience, we denote by $\hat a_t$ the action selected by the learning rule at time $t$. Last, since $|\Acal|\geq 2$, we can fix $a_1\neq a_2\in \Acal$ two arbitrary actions. These will be the only used actions for our constructions, all other actions $a\in\Acal\setminus\{a_1,a_2\}$ will have zero reward at all times.

    We start by constructing rewards that will depend on the actions of the learning rule. By Lemma \ref{lemma:other_form_C5}, we can fix $\epsilon$ such that for any increasing sequence $(T_i)_{i\geq 0}$, letting $\Tcal = \bigcup_{i\geq 0}\Tcal^i\cap \{t\geq T_i\}$, there exists a sequence of sets $A_k\downarrow\emptyset$ such that
    \begin{equation*}
        \Ebb[\hat\mu_{(X_t)_{t\in\Tcal}}(A_k)] \geq \epsilon,\quad \forall k\geq 0.
    \end{equation*}
    Here we used that the sequence of sets is decreasing so that $\Ebb[\hat\mu_{(X_t)_{t\in\Tcal}}(A_k)]$ is decreasing in $i$. 
    
    The end rewards are constructed by induction: at the phase $p$ of the construction, the rewards $r^\star_t$ have been constructed for all $t <  T_p^\star$ for some time $T_p^\star=2^{R_p^\star}$. Further, we have defined some disjoint sets $B_1,\ldots,B_p$, increasing times $T_1^\star,\ldots,T_{p-1}^\star$, and a policy $\pi^{(p)}$ such that $\pi^{(p)}(x)=a_2$ for all $x\notin B_1\cup\cdots B_p$, and for any $p'\leq p$,
    \begin{equation}\label{eq:for_induction_excess_error}
        \Ebb\sqb{\max_{T_{p'-1}^\star \leq  T < T_{p'}^\star} \frac{1}{T}\sum_{t=1}^T r_t^\star(\pi^{(p)}(X_t)) - r_t^\star(\hat a_t)} \geq \frac{\epsilon}{16} + \frac{\epsilon}{2^{p+10}},
    \end{equation}
    where we used the notation $T_0^\star=0$. Last, at phase $p$ we have also constructed a sequence of increasing indices $(Q_p(i))_{i\geq 0}$ with $Q_p(i)\geq 4i$ such that with $\Tcal^{(p)} = \bigcup_{i\geq 0}\Tcal^i\cap\{t\geq 2^{Q_p(i)}\}$, one has
    \begin{equation}\label{eq:past_sets_not_visited}
        \Ebb\sqb{\sup_{T\geq 1}\frac{1}{T}\sum_{t\leq T,t\in\Tcal^{(p)}} \1_{B_{p'}}(X_t) } \leq \frac{\epsilon}{2^{p'+10}},\quad p'\leq p.
    \end{equation}
    For instance, for $p=0$ we can simply take $Q_0(i)=2i$ for all $i\geq 0$. We then suppose that we completed phase $p\geq 0$ and proceed with the induction to construct the set $B_{p+1}$, time $T_{p+1}^\star$ and rewards $r_t^\star$ until time $T_{p+1}^\star.$
    
    Before doing so, we need to construct an auxiliary reward process. These rewards have the following behavior. Before $T_p^\star = 2^{R_p^\star}$, these are constructed identically as the rewards $\mb r^\star$. Then, at time $t\geq 2^{R_p^\star}$, either the rewards are always zero and this is called an inactive time; or the time is active, in which case the ``safe'' action $a_2$ always receives a reward $3/4$, and the ``uncertain'' action $a_1$ receives a reward that can either be $0$ or $1$ with equal probability. We say that the learning rule explores at an active time $t$ if it selects action $a_1$. At the high level, the rewards proceed by period and tentatively activate the times from $\Tcal^i$ for some $i\geq 0$. If the learning rule performs too many explorations, the trial fails and we instead aim to activate fewer times from $\Tcal^j$ for $j<i$. We construct the rewards inductively by period $[2^{r},2^{r+1})$ for $r\geq r_0$. Each of these periods will be associated with a level $i(r)\geq 0$, which roughly corresponds to the fact that the active times during period $r$ were times in $\Tcal^{i(r)}$. We also denote by $\Scal_t$ the set of active times up until time $t$ (included). The formal procedure to define the online rewards is given in Algorithm \ref{alg:define_rewards}, where $r_t(a)$ denotes the reward for action $a$ defined by the procedure at time $t$, for $t\geq 1$.

    \begin{algorithm}[h!]
\caption{Procedure to define the online rewards}\label{alg:define_rewards}
\hrule height\algoheightrule\kern3pt\relax

Let $(B_t)_{t\geq 1}$ be an i.i.d. $\Bcal(\frac{1}{2})$ sequence\,

\For{$t=1,\ldots,T_p^\star-1$}{
    Observe context $X_t$\,

    Define $r_t(a)=r_t^\star(a)$ for all $a\in\Acal$\,

    Observe action selected by learner $\hat a_t$\,
}

Initialize $i(R_p^\star)=0$ and let $\Scal_{T_p^\star-1}=\emptyset$\,

\For{$r\geq R_p^\star$}{
    \For {$t=2^{r}, \ldots, 2^{r+1}-1$}{
        
        Observe context $X_t$\,
    
        \uIf{$t\notin \Tcal^{i(r)}$}{
            Define $r_t(a)=0$ for all $a\in\Acal$ and $\Scal_t = \Scal_{t-1}$
        }
        \uElseIf{$\forall T_p^\star\leq t'<t, \, X_{t'}\neq X_t$}{
            Define $r_t(a) = \begin{cases}
                B_t &a=a_1\\
                \frac{3}{4} & a=a_2,\\
                0 & a\notin\{a_1,a_2\}
            \end{cases}$ for $a\in\Acal$\,

            $\Scal_t = \Scal_{t-1}\cup \{t\}$
        }
        \uElseIf{$\exists T_p^\star \leq t' <t$ such that $X_t=X_{t'}$, $t'\in \Scal_{t-1}$ and $\hat a_{t'} = a_1$}{
            Define $r_t(a) = 0$ for all $a\in\Acal$ and $\Scal_t = \Scal_{t-1}$
        }
        \Else{
            Define $r_t(a)=r_{t'}(a)$ for all $a\in\Acal$ where $t'<t$, $X_t=X_{t'}$ and $t'\in \Scal_{t-1}$\,

            $\Scal_t\gets \Scal_{t-1}\cup \{t\}$
        }

        Observe action selected by learner $\hat a_t$\,

        \lWhile{$\frac{1}{t}\sum_{u=T_p^\star}^t \1_{u\in \Scal_t}\1_{\hat a_u\neq a_2} \geq \frac{1}{2^{2i(r)}(i(r)+1)}$}{
                $i(r)\gets \max(0,i(r) -1)$
        }
   }
   Define $i(r+1) = \min\{i(r)+1,k\}$ where $k$ is such that $Q_p(k)\leq r+1< Q_p(k+1)$
}

\hrule height\algoheightrule\kern3pt\relax
\end{algorithm}

     Let $\Scal=\bigcup_{t\geq 1} \Scal_t$ be the set of all active times. We first give some properties on the learning procedure starting from time $T_p^\star$. As a first step, we show that the learner cannot make better predictions than the simple policy $\pi_0:x\in\Xcal\mapsto a_2\in\Acal$. Precisely, we show that the quantities $r_t(\hat a_t)-r_t(a_2)+\1_{t \in\Scal}\1_{\hat a_t\neq a_2}/4$ for $t\geq T_p^\star$ form the increments of a super-martingale with respect to the filtration $\sigma( \Xbb_{\leq t}, \hat{\mb a}_{\leq t}, \mb r_{\leq t-1})$. First, note that whether $t$ is active, i.e., $t\in\Scal$ only requires the knowledge of $\Xbb_{\leq t}$ and the actions $\hat{\mb a}_{\leq t}$, hence is measurable with respect to the given filtration. Next, if $t$ is inactive, all rewards are zero. We now consider active times. Denote by $u(t)$ the time of the first occurrence of $X_t$ starting from $T_p^\star$, i.e., $u(t) = \min\{T_p^\star\leq u\leq t: X_t=X_u\}$. Then, if $t$ is active, $r_t(a_1)-r_t(a_2) = B_{u(t)}-3/4.$ Moreover, by construction, the learning rule has not queried $a_1$ for any previous active time $u$ within the same period as $t$ such that $X_t=X_u$. However, these are the only times when $B_{t'}$ affected the rewards. As a result, all rewards that the learning rule has received before time $t$ are independent of $B_{u(t)}$ (whether $t$ is active or not). This shows that $B_{u(t)}$ is independent from $\Xbb_{\leq t}$, $\hat{\mb a}_{\leq t}$ and $\mb r_{\leq t-1}$ together. As a result,
    \begin{align*}
        \Ebb[r_t(\hat a_t)-r_t(a_2)+\1_{t\in\Scal}\1_{\hat a_t\neq a_2}/4\mid \Xbb_{\leq t}, \hat{\mb a}_{\leq t}, \mb r_{\leq t-1}] &= \1_{t\in \Scal}(-1/2\cdot \1_{\hat a_t\notin\{a_1,a_2\}}\\ 
        &\quad  + \1_{\hat a_t=a_1}\Ebb[B_{u(t)}-1/2 \mid \Xbb_{\leq t}, \hat{\mb a}_{\leq t}, \mb r_{\leq t-1}]) \\
        &=-1/2\cdot \1_{t\in\Scal} \1_{\hat a_t\notin\{a_1,a_2\}}\leq 0.
    \end{align*}
    This ends the proof that $(r_t(\hat a_t)-r_t(a_2)+\1_{t\in\Scal}\1_{\hat a_t\neq a_2}/4)_{t\geq T_p^\star}$ form the increments of a super-martingale, and these are bounded in absolute value by one. Azuma-Hoeffding's inequality then implies for any $T\geq T_p^\star$,
    \begin{equation*}
        \Pbb\sqb{\sum_{t=T_p^\star}^T r_t(\hat a_t)-r_t(a_2)\geq 2T^{3/4} - \frac{1}{4}\sum_{t=T_p^\star}^T \1_{t\in\Scal}\1_{\hat a_t\neq a_2}}\leq e^{-2\sqrt T}.
    \end{equation*}
    Borel-Cantelli's lemma then implies that on an event $\Ecal$ of probability one, there exists $\hat T\geq T_p^\star$ such that for any $T\geq \hat T$,
    \begin{equation*}
        \sum_{t=T_p^\star}^T r_t(\hat a_t)-r_t(a_2)< 2T^{3/4} - \frac{1}{4}\sum_{t=T_p^\star}^T \1_{t\in\Scal}\1_{\hat a_t\neq a_2}.
    \end{equation*}

    We now focus on the level $i(r)$ at each period. Note that this quantity is updated by the procedure along the learning process: it starts at $i(r-1)+1$ (or $0$ if $r=r_0$) at the beginning of the period $[2^r,2^{r+1})$, then can only decrease during the period. Starting from the end of the period $2^{r+1}$, the level $i(r)$ is never updated again. To avoid any confusions, we denote by $I(r)$ this final value of $i(r)$ once the period is completed. We aim to prove that the level at each period $i(r)$ eventually diverges to infinity. Fix $j\geq 0$. Because $f_\cdot$ is universally consistent under $\Xbb$, it has in particular vanishing excess error compared to $\pi_0$. Hence, we have
    \begin{equation*}
        \Pbb\sqb{\limsup_{T\to\infty} \frac{1}{T}\sum_{t=1}^T r_t(a_2) - r_t(\hat a_t) \geq \frac{1}{2^{2j+4}(j+1)} }=0.
    \end{equation*}
    As a result, by the dominated convergence theorem there exists $t_j\geq 1$ such that
    \begin{equation*}
        \Pbb\sqb{\sup_{T\geq t_j} \frac{1}{T}\sum_{t=1}^T r_t(a_2) - r_t(\hat a_t) \geq \frac{1}{2^{2j+4}(j+1)} }\leq \frac{\epsilon}{2^{j+10}}.
    \end{equation*}
    We denote by $\Fcal_j$ the complement event.
    Next, because $\Ecal$ has full probability, there exists $t_j'$ such that
    \begin{equation*}
        \Pbb\sqb{\sum_{t=T_p^\star}^T r_t(\hat a_t)-r_t(a_2) < 2T^{3/4} - \frac{1}{4}\sum_{t=T_p^\star}^T \1_{t\in\Scal}\1_{\hat a_t\neq a_2},\, \forall T\geq t_j'} \leq \frac{\epsilon}{2^{j+10}}.
    \end{equation*}
    We denote by $\Ecal_j$ the complement event. Now, we define an integer $R_j\geq R_p^\star$ such that $2^{R_j-j}\geq \max(t_j,t_j',2^{8j+16}(j+1)^4, 2^{2j+4}(j+1)T_p^\star,2^{Q_p(j)})$. Using the previous two equations shows that on $\Ecal_j\cap\Fcal_j$ of probability at most $1-\frac{\epsilon}{2^{j+9}}$, for all $T\geq 2^{R_j-j}$,
    \begin{align*}
        \frac{1}{T}\sum_{t=T_p^\star}^T \1_{t\in\Scal}\1_{t\neq a_2} &< \frac{4}{T}\sum_{t<T_p^\star} (r_t(\hat a_t)-r_t(a_2)) + \frac{8}{T^{1/4}} + \frac{1}{2^{2j+2}(j+1)} \\
        &\leq \frac{4T_p^\star}{T} + \frac{8}{T^{1/4}} + \frac{1}{2^{2j+2}(j+1)}
        \leq \frac{1}{2^{2j}(j+1)}.
    \end{align*}
    Also, for any $r\geq R_j-j$, one has $r\geq Q_p(j)$ so that the quantities $I(r)$ can freely increase until they reach $j$ from when the quantities $i(r)$ are always lower bounded by $j$. In particular, by the union bound, this shows that
    \begin{equation*}
        \Pbb\sqb{\forall j\geq 0,\inf_{r\geq R_j} I(r) \geq j} \geq \Pbb\sqb{\bigcap_{j\geq 0}\Ecal_j\cap\Fcal_j} \geq 1-\frac{\epsilon}{2^8}.
    \end{equation*}
    We denote by $\Fcal=\{\forall j\geq 0,\inf_{r\geq R_j}I(r)\geq j\}$ the corresponding event.

    We are now ready to show that $f_\cdot$ is not universally consistent. Because $\Xbb\notin\Ccal_5$, with $\Tcal = \bigcup_{i\geq 0}\Tcal^i \cap\{t\geq 2^{R_j} \}$, there exists a measurable sets $A_k\downarrow\emptyset$ such that for all $k\geq 1$ we have $\Ebb[\hat\mu_{(X_t)_{t\in\Tcal}}(A_k)]\geq \epsilon.$ Now because $A_k\downarrow\emptyset$, we have
    \begin{equation*}
        0\leq \lim_{k+\to\infty}\Pbb\paren{\exists t< T_p^\star:X_t\in A_k}\leq \sum_{t<T_p^\star}\lim_{k\to\infty}\Pbb(X_t\in A_k)=0.
    \end{equation*}
    Also, because $\Xbb\in\Ccal_4$, by Lemma \ref{prop:C5_equivalent_form} we have
    \begin{equation*}
        \lim_{k\to\infty} \Ebb\sqb{\sup_{i\geq 0}\hat\mu_{(X_t)_{t\in\Tcal^i}}(A_k)}=0.
    \end{equation*}
    As a result, there exists $l\geq 1$ such that
    \begin{equation}\label{eq:properties_of_Al}
        \Ebb\sqb{\sup_{i\geq 0}\hat\mu_{(X_t)_{t\in\Tcal^i}}(A_l)} \leq \frac{\epsilon}{2^{p+11}} \quad \text{and} \quad \Pbb\paren{\exists t< T_p^\star:X_t\in A_k} \leq \frac{\epsilon}{2^{p+11}}.
    \end{equation}
    We fix this index $l$ in the rest of the proof. Let $L_p^\star$ be an integer such that $L_p^\star\geq  \max(R_p^\star+10-\log_2\epsilon,R_{10-\log_2 \epsilon},4(\log_2(C_\epsilon)+10-\log_2\epsilon))$, where $C_\epsilon = \sqrt{2\ln\frac{8}{\epsilon}}$. Now by construction, since we have $\Ebb[\hat\mu_{(X_t)_{t\in\Tcal}}(A_l)]\geq \epsilon$, we have in particular
    \begin{equation*}
        \Ebb\sqb{ \sup_{T\geq 2^{L_p^\star}} \frac{1}{T}\sum_{t\leq T,t\in\Tcal} \1_{A_l}(X_t)} \geq \epsilon.
    \end{equation*}
    Thus, by the dominated convergence theorem, there exists an integer $R_{p+1}^\star> 2^{L_p^\star}$ such that
    \begin{equation}\label{eq:many_occurrence_in_Al}
        \Ebb\sqb{ \max_{2^{L_p^\star}\leq T<2^{R_{p+1}^\star}} \frac{1}{T}\sum_{t\leq T,t\in\Tcal} \1_{A_l}(X_t)} \geq \frac{\epsilon}{2}.
    \end{equation}
    We define $T_{p+1}^\star = 2^{R_{p+1}^\star}$. As a second step, we show that when during the learning process until time $T_{p+1}^\star$, for a large proportion of active times $t$ for which $X_t\in A_l$, the optimal arm in hindsight is $a_1$. Precisely, we aim to show that
    \begin{equation*}
        \Ebb\sqb{  \max_{2^{L_p^\star}\leq T<T_{p+1}^\star} \frac{1}{T}\sum_{t\leq T,t\in\Scal} \1_{A_l}(X_t) \cdot B_{u(t)}  } \geq \frac{\epsilon}{8}.
    \end{equation*}
    To prove this, we reason conditionally on $\Xbb$. Define
    \begin{equation*}
        \hat T = \argmax_{2^{L_p^\star}\leq T<T_{p+1}^\star} \frac{1}{T}\sum_{t\leq T,t\in\Tcal} \1_{A_l}(X_t).
    \end{equation*}
    Also, let $Exp = \{T_p^\star \leq t \leq \hat T: t\in\Scal, X_t\in A_l,\hat a_t=a_1\}$ the set of ``exploration'' times on $A_l$ when the learning rule selected action $a_1$ without prior knowledge on the value $B_{u(t)}$ for active time $t$. For any exploration time $t\in Exp$, we also define $N(t) = |\{T_p^\star \leq t'\leq t:t\in\Scal, X_{t'}=X_t\}|$ the number of active occurrences of $X_t$ before the exploration at $t$. Note that after the exploration, new duplicates of $X_t$ will never be active anymore. Last, denote by $Unexp = (A_l\cap\{X_t,T_p^\star\leq t \leq \hat T\})\setminus \{X_t,t\in Exp\}$ the set of points in $A_l$ that were left unexplored until horizon $\hat T$. As above, for $x\in Unexp$, we denote by $N(x) = |\{T_p^\star\leq t\leq \hat T: t\in\Scal,X_{t'}=x|$ the number of active occurrences of $x$ until $\hat T$. Also, by abuse of notation, for any $x\in Unexp$, we denote $u(x) = \min\{T_p^\star\leq t\leq \hat T: X_t=x\}$ the first occurrence of $X_t$. Conditionally on the realization of $\Xbb$ (which as a result makes $\hat T$ deterministic), the sequence $(\1_{t\in Exp} N(t)(B_{u(t)}-\frac{1}{2}))_{T_p^\star\leq t\leq \hat T}$ followed by the sequence $(N(x)(B_{u(x)}-\frac{1}{2}))_{x\in Unexp}$ form the increments of a martingale with filtration given by the $\sigma$-algebras $\sigma(\Xbb,\hat {\mb a}_{\leq t},\mb r_{\leq t-1})$. Indeed, conditionally on $\Xbb$, the past history $\hat{\mb a}_{\leq t-1},\mb r_{\leq t-1}$ and the selected action $\hat a_t$, at an exploration time $t\in Exp$, the value $B_{u(t)}$ is independent from $\Xbb$ and has never been revealed yet, hence is independent from the history as well. Similarly, for unrevealed points $x\in Unexp$, the variables $B_{u(x)}$ are together independent and also independent from $\Xbb$ and the history $\hat{\mb a}_{\leq \hat T}, \mb r_{\leq \hat T}$. The final term of the described martingale writes
    \begin{equation*}
        \sum_{t=T_p^\star}^{\hat T} \1_{t\in Exp}N(t)\paren{B_{u(t)}-\frac{1}{2}} + \sum_{x\in Unexp} N(x)\paren{B_{u(x)}-\frac{1}{2}} = \sum_{t=T_p^\star}^{\hat T}\1_{t\in\Scal}\1_{A_l}(X_t)\paren{ B_{u(t)} - \frac{1}{2}}.
    \end{equation*}
    We now bound these increments. For any $R_p^\star\leq r<R_{p+1}^\star$, during the period $[2^r,2^{r+1})$, one has $\Scal\cap[2^r,2^{r+1})\subset \Tcal^k$, where $k$ is such that $Q_p(k) \leq r<Q_p(k+1)$. Now recall that $Q_p(k)\geq 4k$ so that the number of active duplicates for a given point $x$ during period $r$ is at most $2^k\leq 2^{r/4}$. Hence, if $\hat T\in[2^{\hat r},2^{\hat r+1})$, the number of active duplicates of any point until $\hat T$ satisfies
    \begin{equation*}
        \max_{t\in Exp} N(t), \max_{x\in Unexp} N(x) \leq \sum_{r=r_0}^{\hat r} 2^{r/4} \leq \frac{2^{r/4}}{1-2^{-1/4}} \leq \frac{\hat T^{1/4}}{2^{1/4}-1}\leq 6 \hat T^{1/4}.
    \end{equation*}
    In particular, all increments of the constructed martingale have elements norm bounded by the above value. Azuma-Hoeffding's inequality then yields
    \begin{equation*}
        \Pbb\sqb{ \sum_{t=T_p^\star}^{\hat T}\1_{t\in\Scal}\1_{A_l}(X_t)\paren{ B_{u(t)} - \frac{1}{2}} \leq - C_\epsilon\hat T^{3/4} \mid \Xbb}\leq \frac{\epsilon}{8}.
    \end{equation*}
    Let $\Gcal$ be the complement event, i.e., the event when $\sum_{t=T_p^\star}^{\hat T}\1_{t\in\Scal}\1_{A_l}(X_t)\paren{ B_{u(t)} - \frac{1}{2}} > - C_\epsilon\hat T^{3/4}$. Then, using Eq~\eqref{eq:many_occurrence_in_Al} we obtain
    \begin{equation}\label{eq:lower_bound_proba_good_event}
        \Ebb\sqb{ \frac{\1_{\Fcal\cap \Gcal}}{\hat T}\sum_{t\leq \hat T,t\in\Tcal} \1_{A_l}(X_t)} \geq \Ebb\sqb{  \frac{1}{\hat T}\sum_{t\leq \hat T,t\in\Tcal} \1_{A_l}(X_t)} - \Pbb[\Fcal]-\Pbb[\Gcal] \geq \frac{\epsilon}{2}-\frac{\epsilon}{8}-\frac{\epsilon}{8} =\frac{\epsilon}{4}.
    \end{equation}

    As a last step, we show that under $\Fcal\cap\Gcal$, the learning rule incurs significant regret compared to the best action in hindsight for times with contexts falling in $A_l$. On $\Fcal\cap\Gcal$,
    \begin{equation*}
        \frac{1}{\hat T}\sum_{t=T_p^\star}^{\hat T}\1_{t\in\Scal}\1_{A_l}(X_t) B_{u(t)} \geq \frac{1}{2\hat T}\sum_{t=T_p^\star}^{\hat T}\1_{t\in\Scal}\1_{A_l}(X_t) -\frac{C_\epsilon}{\hat T^{1/4}} \geq \frac{1}{2\hat T}\sum_{t=T_p^\star}^{\hat T}\1_{t\in\Scal}\1_{A_l}(X_t) - \frac{\epsilon}{2^{10}}.
    \end{equation*}
    We used $\hat T\geq 2^{L_p^\star}$ in the last inequality. We now aim to compare the right-hand side of the last inequality to $\frac{1}{\hat T}\sum_{T_p^\star\leq t\leq \hat T, t\in\Tcal}\1_{A_l}(X_t) $. Because $\Fcal$ is satisfied, $\Tcal\setminus\Scal$ the set of inactive times that are counted within $\Tcal$ only contains times $t$ such that there exists $t'<t$ with $t'\in\Scal$ when the learning rule performed an exploration (see Algorithm \ref{alg:define_rewards}). Thus,
    \begin{equation*}
        \sum_{t=T_p^\star}^{\hat T}\1_{t\in\Scal}\1_{A_l}(X_t) \geq \sum_{t \leq \hat T,t\in\Tcal}\1_{A_l}(X_t) - T_p^\star - \sum_{t\in Exp}|\{t<t'\leq \hat T,t'\in\Tcal\setminus \Scal,X_{t'}=X_t\}|.
    \end{equation*}
    Letting $\hat j$ be the integer such that $R_{\hat j}\leq \hat R<R_{\hat j+1}$, i.e., $2^{R_{\hat j}}\leq \hat T<2^{R_{\hat j+1}}$, we observe that
    \begin{align*}
        \sum_{t\in Exp}|\{t<t'\leq \hat T,t'\in\Tcal\setminus \Scal,X_{t'}=X_t\}| &\leq 2^{\hat R-\hat j} + \sum_{t\in Exp}|\{2^{\hat R-\hat j},t<t'\leq \hat T,t'\in\Tcal^{\hat j},X_{t'}=X_t\}| \\
        &\leq 2^{\hat R-\hat j} + |Exp| 2^{\hat j}(\hat j +1)\\
        &\leq \frac{\hat T}{2^{\hat j-1}} + |Exp| 2^{\hat j}(\hat j +1).
    \end{align*}
    where we used the fact that because $(R_j)_{j\geq 1}$ is increasing, each distinct point is duplicated at most $2^{\hat j}$ times in any period $\Tcal\cap [2^r,2^{r+1})$ with $r<R_{\hat j+1}$. Next, because $\Fcal$ is satisfied we have in particular $I(\hat R)\geq \hat j$, implying that at time $\hat T$, we had the guarantee
    \begin{equation*}
        \frac{|Exp|}{\hat T}\leq \frac{1}{\hat T}\sum_{u=T_p^\star}^{\hat T}\1_{u\in \Scal} \1_{\hat a_u\neq a_2} < \frac{1}{2^{2I(\hat R)}(I(\hat R)+1)} \leq   \frac{1}{2^{2\hat j}(\hat j+1)} .
    \end{equation*}
    Combining the previous four equations and the fact that $\hat T\geq 2^{L_p^\star}$ shows that on $\Fcal\cap\Gcal$ one has
    \begin{align*}
         \frac{1}{\hat T}\sum_{t=T_p^\star}^{\hat T}\1_{t\in\Scal}\1_{A_l}(X_t) B_{u(t)} 
         &\geq \frac{1}{2\hat T}\sum_{t \leq \hat T,t\in\Tcal}\1_{A_l}(X_t)-\frac{\epsilon}{2^{10}} -\frac{T_p^\star}{2\hat T} - \frac{1}{2^{\hat j}}-\frac{1}{2^{\hat j+1}}\\
         &\geq \frac{1}{2\hat T}\sum_{t \leq \hat T,t\in\Tcal}\1_{A_l}(X_t)-\frac{\epsilon}{2^8}.
    \end{align*}
    In the last inequality, we used $\hat j\geq 10-\log_2\epsilon$, a consequence of $\hat T\geq 2^{L_p^\star}$. We are now ready to compare the reward of the learning rule to the best action in hindsight for times $t$ such that $X_t\in A_l$. Precisely, consider the following actions $a_t^\star$: at an active time $t\in\Scal$ and $X_t\in A_l$, we pose $a_t^\star= a_1$ if $B_{u(t)}=1$ and $a_t^\star=a_2$ otherwise. For any other active time $t\in\Scal$ and $X_t\notin A_l$, we pose $a_t^\star=a_2$ (which is in that case not necessarily the best action in hindsight). First note that
    \begin{align*}
        \frac{1}{\hat T}\sum_{t=T_p^\star}^{\hat T} \1_{A_l}(X_t)(r_t(a_t^\star) - r_t(\hat a_t)) &= \frac{1}{\hat T}\sum_{t=T_p^\star}^{\hat T} \1_{t\in\Scal}\1_{A_l}(X_t)\paren{ \frac{3+B_{u(t)}}{4} - r_t(\hat a_t)  }\\
        &\geq  \frac{1}{4\hat T}\sum_{t=T_p^\star}^{\hat T} \1_{t\in\Scal}\1_{A_l}(X_t)\1_{t\notin Exp} B_{u(t)}\\
        &\geq \frac{1}{4\hat T}\sum_{t=T_p^\star}^{\hat T} \1_{t\in\Scal}\1_{A_l}(X_t)B_{u(t)} - \frac{1}{4\hat T}\sum_{t=T_p^\star}^{\hat T} \1_{t\in Exp}\1_{A_l}(X_t).
    \end{align*}
    Also, note that
    \begin{equation*}
        \frac{1}{\hat T}\sum_{t=T_p^\star}^{\hat T} \1_{A_l^c}(X_t)(r_t(a_t^\star) - r_t(\hat a_t)) \geq  \frac{1}{\hat T}\sum_{t=T_p^\star}^{\hat T} \1_{t\in\Scal}\1_{A_l^c}(X_t)\paren{ \frac{3}{4} - r_t(\hat a_t) } \geq  - \frac{1}{4\hat T}\sum_{t=T_p^\star}^{\hat T} \1_{t\in Exp}\1_{A_l^c}(X_t).
    \end{equation*}
    Combining the two previous equations shows that on $\Fcal\cap\Gcal$,
    \begin{equation*}
        \frac{1}{\hat T}\sum_{t=T_p^\star}^{\hat T} r_t(a_t^\star) - r_t(\hat a_t) \geq \frac{1}{4\hat T}\sum_{t=T_p^\star}^{\hat T} \1_{t\in\Scal}\1_{A_l}(X_t)B_{u(t)} - \frac{|Exp|}{4\hat T} \geq \frac{1}{2\hat T}\sum_{t \leq \hat T,t\in\Tcal}\1_{A_l}(X_t)-\frac{\epsilon}{2^7}.
    \end{equation*}
    Combining this with Eq~\eqref{eq:lower_bound_proba_good_event} shows that
    \begin{equation}\label{eq:for_new_rewards}
        \Ebb\sqb{\max_{T_p^\star\leq T <T_{p+1}^\star}\frac{1}{ T}\sum_{t=1}^T r_t(a_t^\star) - r_t(\hat a_t)} \geq \Ebb\sqb{\frac{1}{\hat  T}\sum_{t=T_p^\star}^{\hat T} r_t(a_t^\star) - r_t(\hat a_t) -\frac{T_p^\star}{\hat T} } \geq \frac{\epsilon}{8}-\frac{\epsilon}{2^{10}} - \frac{\epsilon}{2^7}.
    \end{equation}
    
    As a last step before defining new rewards, we introduce the scale $\delta_l>0$ such that
    \begin{equation*}
        \Pbb\sqb{\min_{1\leq t,t'< 2^{R_{p+1}^\star},X_t\neq X_{t'}}\rho(X_t,X_{t'})\leq \delta_l }\leq \frac{\epsilon}{2^{10}}.
    \end{equation*}
    We denote by $\Hcal$ the complement event.
    
    We are now ready to introduce the new online rewards. To do so, we first need to introduce some notations for partitions of the space $\Xcal$. Let $(x^u)_{u\geq 1}$ be a dense sequence in $\Xcal$. We define the sets $P_u = (A_l \cap  B(x^u,\delta_l) )\setminus \bigcup_{v<u} B(x^u,\delta_l)$ for $u\geq 1$. We can easily check that the sequence of measurable sets $(P_u)_{u\geq 1}$ forms a partition of $A_l$, and that each set $P_u$ has diameter at most $\delta_l$. For any binary sequence $\mb b= (b_u)_{u\geq 1}$, we define online rewards that follow the same structure as defined with the procedure from Algorithm \ref{alg:define_rewards}, with the difference that rewards $r_t^{\mb b}$, at any active time $t\in \Scal$ with $X_t\in P_u$ for some $u\geq 1$, are constructed using the binary value $b_u$ instead of the random binary variable $B_{u(t)}$ where $u(t) = \min\{T_p^\star\leq u\leq t: X_t=X_u\}$. The procedure to construct the rewards $\mb r^{\mb b}$ until time $T_{p+1}^\star$ is given in Algorithm \ref{alg:define_rewards_new}.

    \begin{algorithm}[h!]
    \caption{Procedure to define the online rewards $\mb r^{\mb b}_{<T_{p+1}^\star}$}\label{alg:define_rewards_new}
    \hrule height\algoheightrule\kern3pt\relax

    \KwIn{Binary sequence $\mb b$}
    Let $(B_t)_{t\geq 1}$ be an i.i.d. $\Bcal(\frac{1}{2})$ sequence\,

    \For{$t=1,\ldots,T_p^\star-1$}{
        Observe context $X_t$\,
    
        Define $r_t(a)=r_t^\star(a)$ for all $a\in\Acal$\,
    
        Observe action selected by learner $\hat a_t$\,
    }
    
    Initialize $i(R_p^\star)=0$ and let $\Scal_{T_p^\star-1}=\emptyset$\,
    
    \For{$r= R_p^\star,\ldots,R_{p+1}^\star-1$}{
        \For {$t=2^{r}, \ldots, 2^{r+1}-1$}{
            
            Observe context $X_t$\,
        
            \uIf{$t\notin \Tcal^{i(r)}$}{
                Let $r^{\mb b}_t(a)=0$ for all $a\in\Acal$ and $\Scal_t = \Scal_{t-1}$
            }
            \uElseIf{$\forall T_p\star \leq t'<t, \, X_{t'}\neq X_t$; $X_t\in P_u$ for some $u\geq 1$}{
                Let $r^{\mb b}_t(a) = \begin{cases}
                    b_u &a=a_1\\
                    \frac{3}{4} & a=a_2,\\
                    0 & a\notin\{a_1,a_2\}
                \end{cases}$ for $a\in\Acal$\,

                $\Scal_t = \Scal_{t-1}\cup \{t\}$
            }
            \uElseIf{$\forall T_p\star \leq t'<t, \, X_{t'}\neq X_t$}{
                Let $r^{\mb b}_t(a) = \begin{cases}
                    B_t &a=a_1\\
                    \frac{3}{4} & a=a_2,\\
                    0 & a\notin\{a_1,a_2\}
                \end{cases}$ for $a\in\Acal$\,

                $\Scal_t = \Scal_{t-1}\cup \{t\}$
            }
            \uElseIf{$\exists T_p\star\leq t' <t$ such that $X_t=X_{t'}$, $t'\in \Scal_{t-1}$ and $\hat a_{t'} = a_1$}{
                Let $r^{\mb b}_t(a) = 0$ for all $a\in\Acal$ and $\Scal_t = \Scal_{t-1}$
            }
            \Else{
                Define $r^{\mb b}_t(a)=r_{t'}(a)$ for all $a\in\Acal$ where $t'<t$, $X_t=X_{t'}$ and $t'\in \Scal_{t-1}$\,
    
                $\Scal_t\gets \Scal_{t-1}\cup \{t\}$
            }
    
            Observe action selected by learner $\hat a_t$\,
    
            \lWhile{$\frac{1}{t}\sum_{u=T_p^\star}^t \1_{u\in \Scal_t}\1_{\hat a_u\neq a_2} \geq 2^{-2i(r)}$}{
                    $i(r)\gets \max(0,i(r) -1)$
            }
       }
       Define $i(r+1) = \min\{i(r)+1,k\}$ where $k$ is such that $Q_p(k)\leq r< Q_p(k+1)$
    }
    
    \hrule height\algoheightrule\kern3pt\relax
    \end{algorithm}
    
    Consider the case when the binary sequence $\mb b$ is sampled as an i.i.d. $\Bcal(\frac{1}{2})$ process. We argue that under the event $\Hcal$, these rewards $\mb r^{\mb b}$ from Algorithm \ref{alg:define_rewards_new} are not distinguishable from the rewards $\mb r$ from Algorithm \ref{alg:define_rewards}. First, observe that they share the same overall structure, the only difference is that when needed to define rewards $r_t^{\mb b}$ at an active time $t\in \Scal$, one may use $b_u$ instead of $B_t$, where $u$ is such that $X_t\in P_u$. Recall that $b_u$ is by hypothesis sampled as $b_u\sim\Bcal(\frac{1}{2})$ as $B_t$ and further, under the event $\Hcal$, all distinct points from $\Xbb_{<T_{p+1}^\star}$ falling within $A_l$ are at distance at least $\delta_l$. We only use $b_u$ for $r_t^{\mb b}$ when $X_t\in P_u$. Therefore, under $\Hcal$, one has $\{t'<t:X_t\in P_u\}=\emptyset$. This shows that the variable $b_u$ was never observed prior to time $t$ and as a result, is not distinguishable from a true random binary variable $B_t\sim\Bcal(\frac{1}{2})$. In particular, under $\Hcal$, the rewards $\mb r^{\mb b}$ when $\mb b\overset{i.i.d.}{\sim}\Bcal(\frac{1}{2})$, yield the same selected actions as the rewards $\mb r$ from Algorithm \ref{alg:define_rewards}. Now for any binary sequence $\mb b$, we define the policy 
    \begin{equation*}
        \pi^{\mb b}(x) = \begin{cases}
            a_1 & \text{if }b_u^k=1,x\in P_u,\\
            a_2 & \text{if }b_u^k=0,x\in P_u,\\
            a_2 &\text{if }x\notin A_l.
        \end{cases}
    \end{equation*}
    By construction, these are constructed exactly similarly to the best action in hindsight $a_t^\star$ for contexts falling in $A_l$ as defined previously. Therefore,
    \begin{align*}
        \Ebb_{\mb b\overset{i.i.d.}{\sim}\Bcal(\frac{1}{2})}&\sqb{\Ebb_{\Xbb,\mb a}\paren{\max_{T_p^\star\leq T <T_{p+1}^\star}\frac{1}{ T} \sum_{t=1}^T r^{\mb b}_t( \pi^{\mb b}(X_t)) - r^{\mb b}_t(\hat a_t)} }\\
        &\geq \Pbb[\Hcal]\cdot \Ebb_{\Xbb\mid \Gcal} \sqb{\Ebb_{\mb b\overset{i.i.d.}{\sim}\Bcal(\frac{1}{2}),\mb a}\paren{\max_{T_p^\star\leq T <T_{p+1}^\star}\frac{1}{ T} \sum_{t=1}^T r^{\mb b}_t( \pi^{\mb b}(X_t)) - r^{\mb b}_t(\hat a_t)} \mid \Xbb,\Gcal}\\
        &= \Pbb[\Hcal]\cdot \Ebb_{\Xbb\mid \Gcal} \sqb{\Ebb_{\mb a}\paren{\max_{T_p^\star\leq T <T_{p+1}^\star}\frac{1}{ T} \sum_{t=1}^T r_t(a_t^\star) - r_t(\hat a_t)} \mid \Xbb,\Gcal}\\
        &\geq \Ebb_{\Xbb,\mb a} \sqb{\max_{T_p^\star\leq T <T_{p+1}^\star}\frac{1}{ T} \sum_{t=1}^T r_t(a_t^\star) - r_t(\hat a_t)} - \Pbb[\Hcal^c] \geq \frac{\epsilon}{8} - \frac{\epsilon}{2^6} .
    \end{align*}
    In particular, there exists a realization $\mb b$ such that
    \begin{equation}\label{eq:main_step_for_proof}
        \Ebb\sqb{\max_{T_p^\star\leq T <T_{p+1}^\star}\frac{1}{ T} \sum_{t=T_p^\star}^T r^{\mb b}_t( \pi^{\mb b}(X_t)) - r^{\mb b}_t(\hat a_t)} \geq \frac{\epsilon}{8} - \frac{\epsilon}{2^6} .
    \end{equation}
    We fix this realization of $\mb b$ in the rest of the proof. We are now ready to close the induction by letting $B_{p+1}:=A_l\setminus(B_1\cup\ldots\cup B_p)$ and defining the policy $\pi^{(p+1)}$ so as to be consistent with the selected actions of $\pi^{(p)}$ on $B_1,\ldots,B_p$. We pose
    \begin{equation*}
        \pi^{(p+1)}(x) = \begin{cases}
            \pi^{(p)} &\text{if }x\in B_1\cup\ldots\cup B_p,\\
            \pi^{\mb b} &\text{otherwise}.
        \end{cases}
    \end{equation*}
    Observe that by construction, $\pi^{(p+1)}(x)=a_2$ for all $x\notin B_1\cup\ldots\cup B_{p+1}$. Next, we define the rewards $r_t^\star$ to be exactly $r_t^{\mb b}$ for any $t<T_{p+1}^\star$. Note that by the construction given in Algorithm \ref{alg:define_rewards_new}, these rewards are consistent with the rewards $r_t^\star$ that had already been constructed for $t<T_p^\star$. In the rest of the proof, we show that these satisfy the induction requirements.

    We first check that the fact that $\pi^{(p+1)}$ differs from $\pi^{(p)}$ on $A_l$ does not affect significantly the guarantees of the constructed rewards until time $T_p^\star$. Indeed, for any $T< T_p^\star$,
    \begin{equation*}
        \left|\sum_{t=1}^T r_t^\star(\pi^{(p+1)})-r_t^\star(\pi^{(p)}) \right| \leq |\{t\leq T: X_t\in A_l\}| \leq T\1_{\exists t\leq T:Xt\in A_l},
    \end{equation*}
    so that, using Eq~\eqref{eq:for_induction_excess_error} and Eq~\eqref{eq:properties_of_Al}, for any $p'\leq p$,
    \begin{align*}
        \Ebb&\sqb{\max_{T_{p'-1}^\star \leq  T < T_{p'}^\star} \frac{1}{T}\sum_{t=1}^T r_t^\star(\pi^{(p+1)}(X_t)) - r_t^\star(\hat a_t)} \\
        &\geq \Ebb\sqb{\max_{T_{p'-1}^\star \leq  T < T_{p'}^\star} \frac{1}{T}\sum_{t=1}^T r_t^\star(\pi^{(p+1)}(X_t)) - r_t^\star(\hat a_t)} - \Pbb(\exists t<T_p^\star: X_t\in A_l)\\
        &\geq \frac{\epsilon}{16} + \frac{\epsilon}{2^{p+10}} - \frac{\epsilon}{2^{p+11}} \geq \frac{\epsilon}{16} + \frac{\epsilon}{2^{p+11}}.
    \end{align*}
    Now we check that the guarantee also holds for $p'=p+1$. First, recall that by construction of Algorithm \ref{alg:define_rewards_new}, for any $r\geq R_p^\star$, one has that $i(r)\leq k$ where $k$ is such that $Q_p(k)\leq r<Q_p(k+1)$. In particular, the active times during the corresponding period satisfy $\Scal\cap[2^r,2^{r+1})\subset \Tcal^k$. As a result, we obtain $\Scal\subset \Tcal^{(p)}$, where we recall that $\Tcal^{(p)}:=\bigcup_{i\geq 0}\Tcal^i \cap \{t\geq 2^{Q_p(i)}\}$. Then, because $\pi^{(p+1)}$ only differs from $\pi^{\mb b}$ on $B_1\cup\ldots\cup B_p$, for any $T_p^\star\leq T<T_{p+1}^\star$,
    \begin{align*}
        \frac{1}{T}\sum_{t=1}^T r^{\mb b}_t(\pi^{\mb b}(X_t)) - r^{\mb b}_t(\pi^{(p+1)}(X_t)) &\leq\frac{1}{T} \sum_{t\leq T, t\in\Scal}(r^{\mb b}_t(\pi^{\mb b}(X_t)) - r^{\mb b}_t(\pi^{(p+1)}(X_t)) )\\
        &\leq \frac{1}{T}\sum_{t\leq T, t\in\Tcal^{(p)}} \sum_{p'=1}^p  \1_{B_{p'}}(X_t)\\
        &\leq \sum_{p'=1}^p \sup_{T\geq 1} \frac{1}{T}\sum_{t\leq T, t\in\Tcal^{(p)}} \1_{B_{p'}}(X_t).
    \end{align*}
    Therefore, combining Eq~\eqref{eq:main_step_for_proof} and the induction hypothesis Eq~\eqref{eq:past_sets_not_visited}, we obtain
    \begin{align*}
        \Ebb&\sqb{\max_{T_p^\star\leq T <T_{p+1}^\star}\frac{1}{ T} \sum_{t=1}^T r^{\mb b}_t( \pi^{(p+1)}(X_t)) - r^{\mb b}_t(\hat a_t)}\\
        &\geq \Ebb\sqb{\max_{T_p^\star\leq T <T_{p+1}^\star}\frac{1}{ T} \sum_{t=1}^T r^{\mb b}_t( \pi^{\mb b}(X_t)) - r^{\mb b}_t(\hat a_t)} - \sum_{p'=1}^p \Ebb\sqb{\sup_{T\geq 1} \frac{1}{T}\sum_{t\leq T, t\in\Tcal^{(p)}} \1_{B_{p'}}(X_t)} \\
        &\geq \frac{\epsilon}{8} - \frac{\epsilon}{2^6} -\frac{\epsilon}{2^{10}} \geq \frac{\epsilon}{16} + \frac{\epsilon}{2^{p+10}}.
    \end{align*}
    
    The last step consists in constructing the increasing indices $Q_{p+1}(i)$ for $i\geq 0$. By the dominated convergence theorem, for any $i\geq 0$, there exists $\tilde T_i\geq 1$ such that 
    \begin{equation*}
        \Ebb\sqb{\sup_{T\geq \tilde T_i} \frac{1}{T}\sum_{t\leq T,t\in \Tcal^i} \1_{B_{p+1}}(X_t) -\hat\mu_{(X_t)_{t\in\Tcal^i}}(B_{p+1}) } \leq \frac{\epsilon}{2^{p+12+i}}.
    \end{equation*}
    We then define by induction the sequence of integers $Q_{p+1}(i)$ such that $Q_{p+1}(0) \geq \max(Q_p(0),\log_2 \tilde T_0)$ and for all $i\geq 1$, $Q_{p+1}(i) \geq \max(Q_p(i),\log_2 \tilde T_i, Q_{p+1}(i-1))$. In particular, the sequence is increasing and the above equation shows that
    \begin{equation}\label{eq:excess_compared_to_limit_measure}
        \Ebb\sqb{\sup_{T\geq 2^{Q_{p+1}(i)}} \frac{1}{T}\sum_{t\leq T,t\in \Tcal^i} \1_{B_{p+1}}(X_t) -\hat\mu_{(X_t)_{t\in\Tcal^i}}(B_{p+1}) } \leq \frac{\epsilon}{2^{p+12+i}}.
    \end{equation}
    Now letting $\Tcal^{(p+1)} = \bigcup_{i\geq 0}\Tcal^i\cap\{t\geq 2^{Q_{p+1}(i)}\}$, we note that
    \begin{align*}
        \sup_{T\geq 1}\frac{1}{T}\sum_{t\leq T,t\in\Tcal^{(p+1)}} \1_{B_{p+1}}(X_t) &=\sup_{i\geq 0} \sup_{2^{Q_{p+1}(i)}\leq T<2^{Q_{p+1}(i+1)}} \frac{1}{T}\sum_{t\leq T,t\in\Tcal^{(p+1)}} \1_{B_{p+1}}(X_t)\\
        &\leq \sup_{i\geq 0} \sup_{2^{Q_{p+1}(i)}\leq T<2^{Q_{p+1}(i+1)}} \frac{1}{T}\sum_{t\leq T,t\in\Tcal^i} \1_{B_{p+1}}(X_t).
    \end{align*}
    As a result,
    \begin{align*}
        \Ebb&\sqb{ \sup_{T\geq 1}\frac{1}{T}\sum_{t\leq T,t\in\Tcal^{(p+1)}} \1_{B_{p+1}}(X_t)} \\
        &\leq \Ebb\sqb{ \sup_{i\geq 0} \hat\mu_{(X_t)_{t\in\Tcal^i}}(B_{p+1}) } + \sum_{i\geq 0}\Ebb\sqb{\sup_{T\geq 2^{Q_{p+1}(i)}} \frac{1}{T}\sum_{t\leq T,t\in \Tcal^i} \1_{B_{p+1}}(X_t) -\hat\mu_{(X_t)_{t\in\Tcal^i}}(B_{p+1}) } \\
        &\leq  \sup_{i\geq 0}\Ebb\sqb{  \hat\mu_{(X_t)_{t\in\Tcal^i}}(B_{p+1}) } + \frac{\epsilon}{2^{p+11}} \leq \frac{\epsilon}{2^{p+10}}.
    \end{align*}
    In the second inequality we used Eq~\eqref{eq:excess_compared_to_limit_measure}, and in the third inequality, we used Eq~\eqref{eq:properties_of_Al}. Finally, because for all $i\geq 0$, one has $Q_{p+1}(i)\geq Q_p(i)$, we have directly $\Tcal^{(p)}\subset\Tcal^{(p+1)}$, which shows that for all $p'\leq p$, we still have
    \begin{equation*}
        \Ebb\sqb{\sup_{T\geq 1}\frac{1}{T}\sum_{t\leq T,t\in\Tcal^{(p+1)}} \1_{B_{p'}}(X_t) } \leq \frac{\epsilon}{2^{p'+10}},\quad p'\leq p.
    \end{equation*}
    This ends the inductive construction of the rewards $\mb r^\star$.

    The last step of the proof is to show that $f_\cdot$ is not universally consistent under $\Xbb$ for these online rewards $\mb r^\star$. Having constructed the sequence of sets $(B_p)_{p\geq 1}$, we let $\pi^\star$ be the policy defined by
    \begin{equation*}
        \pi^\star(x) = \begin{cases}
            \pi^{(p)}(x)& \text{if }x\in B_p,\\
            a_2 &\text{ otherwise}.
        \end{cases}
    \end{equation*}
    Recall that the sequence of policies $\pi^{(p)}$ for $p\geq 1$ was constructed so that they are consistent: $\pi^{(p')}$ for $p'\geq p\geq 1$ all coincide on $A_p$. Further, all $\pi^{(p)}$ coincide on $(\bigcup_{p\geq 1}B_p)^c$ on which they select $a_2$. Now fix $p\geq 1$. Because the rewards are also constructed to be consistent over time, if $\hat a_t$ denotes the selected action at time $t$ for rewards $\mb r^\star$, the induction implies that for all $p'\geq p$ one has
    \begin{equation}\label{eq:induction_final_result}
        \Ebb\sqb{\max_{T_{p-1}^\star \leq  T < T_p^\star} \frac{1}{T}\sum_{t=1}^T r_t^\star(\pi^{(p')}(X_t)) - r_t^\star(\hat a_t)} \geq \frac{\epsilon}{16}.
    \end{equation}
    As a result, because $\pi^{(p')}$ and $\pi^\star$ coincide everywhere except on $\bigcup_{q>p'}B_q$, we have for any $T_{p-1}^\star \leq  T < T_p^\star$,
    \begin{equation*}
        \frac{1}{T}\sum_{t=1}^T r_t^\star(\pi^\star(X_t)) - r_t^\star(\hat a_t) \geq \frac{1}{T}\sum_{t=1}^T r_t^\star(\pi^{(p')}(X_t)) - r_t^\star(\hat a_t) - \textstyle\1\paren{\exists t<t_p^\star: X_t\in \bigcup_{q>p'}B_q}.
    \end{equation*}
    Because the sets $(B_p)_{p\geq 1}$ are all disjoint, we have
    $\Pbb\paren{\exists t<t_p^\star: X_t\in \bigcup_{q>p'}B_q}\to 0$ as $p'\to\infty$. Thus, using Eq~\eqref{eq:induction_final_result} yields
    \begin{equation*}
        \Ebb\sqb{ \max_{T_{p-1}^\star \leq  T < T_p^\star} \frac{1}{T}\sum_{t=1}^T r_t^\star(\pi^\star(X_t)) - r_t^\star(\hat a_t) } \geq \frac{\epsilon}{16}.
    \end{equation*}
    Because this holds for all $p\geq 1$, Fatou's lemma implies
    \begin{equation*}
        \Ebb\sqb{ \limsup_{T\to\infty}\frac{1}{T}\sum_{t=1}^T r_t^\star(\pi^\star(X_t)) - r_t^\star(\hat a_t) } \geq \limsup_{p\to\infty} \Ebb\sqb{ \max_{T_{p-1}^\star \leq  T < T_p^\star} \frac{1}{T}\sum_{t=1}^T r_t^\star(\pi^\star(X_t)) - r_t^\star(\hat a_t) }  \geq \frac{\epsilon}{16}.
    \end{equation*}
    As a result, the learning rule is not universally consistent under $\Xbb$, which ends the proof of the theorem.
\end{proof}

\subsection{A sufficient condition on learnable processes}
\label{subsec:sufficient_conditions}
In this section, we show that $\Ccal_5$ is sufficient universal learning for all reward models. We recall that the condition $\Ccal_5$ asks that there exists an increasing sequence $(T_i)_{i\geq 0}$ such that $\tilde\Xbb=(X_t)_{t\in\Tcal}\in\Ccal_1'$ where $\Tcal = \bigcup_{i\geq 0}\Tcal^i\cap\{t\geq T_i\}$ is obtained by adding the times $\Tcal^i$ according to the rate given by $(T_i)_{i\geq 0}$.

It is straightforward to see $\Ccal_1\subset\Ccal_5$ since for any $\Xbb\in\Ccal_1$, one can take any arbitrary sequence, for instance $T_i=i$ for $i\geq 0$, and satisfy property $\Ccal_5$. Before showing that $\Ccal_5$ is a sufficient condition for universal learning with online rewards, we state a known result showing that for $\Ccal_1'$ processes, there is a countable sequence of policies that is empirically dense within all measurable policies.

\begin{lemma}[\cite{hanneke:21} Lemma 24]
\label{lemma:countable_dense}
    Let $\Acal$ be a finite action space and $\Xcal$ a separable metrizable Borel space. There exists a countable sequence of measurable policies $(\pi^l)_{l\geq 1}$ from $\Xcal$ to $\Acal$ such that for extended process $\tilde\Xbb = (X_t)_{t\in\Tcal}\in\Ccal_1'$, and any measurable policy $\pi:\Xcal\to\Acal$,
    \begin{equation*}
        \inf_{l\geq 1} \Ebb\left[ \limsup_{T\to\infty} \frac{1}{T}\sum_{t\leq T, t\in\Tcal}\1[\pi^l(X_t)\neq \pi(X_t)]\right] = 0.
    \end{equation*}
\end{lemma}

We are now ready to prove the sufficiency of $\Ccal_5$.

\begin{theorem}
\label{thm:C6_learnable}
Let $\Xcal$ be a metrizable separable Borel space and $\Acal$ a finite action space. Then, $\Ccal_5\subset \Ccal_{online}$.
\end{theorem}

\begin{proof}
Let $\Xbb\in\Ccal_5$, and $(T_i)_{i\geq 0}$ such that letting $\Tcal=\bigcup_{i\geq 0} \Tcal^i\cap \{t\geq T_i\}$ we have $\tilde\Xbb= (X_t)_{t\in\Tcal}\in\Ccal_1'$. We suppose that $T_i = 2^{u(i)}$ for some indices $u(i)$ increasing in $i$. This is without loss of generality, because one could take $\tilde T_i = \min\{2^s, 2^s\geq T_i\}$ and still have a $\Ccal_1'$ process in the definition of $\tilde \Xbb$ (a slower sequence $(T_i)_i$ only reduces considered points, hence does not impact the $\Ccal_1'$ property). We may also suppose that $u(i)\geq 2i$. Also, letting $\eta_i = \sqrt{\frac{8\ln(i+1)}{2^i}}$ for $i\geq 0$, we suppose that $u(i)\geq \eta_i 2^{i+5}$. Last, we suppose that $u(0)=0$ which again can be done without loss of generality since the $\Ccal_1'$ property is not affected by the behavior of the process on the first $T_0$ times. Hence, $T_0=1$.

Similarly to the algorithm that was proposed for stationary rewards in \cite{blanchard:22e}, the learning rule associates a category $p$ to each time $t$ and acts separately on each category. To do so, the algorithm first computes the phase of $t$ as follows: $\textsc{Phase}(t)$ is the unique integer $i$ such that $T_i\leq t<T_{i+1}$. Then, we define the stage $\textsc{Stage}(t):=\lfloor \log_2 t\rfloor=l$ so that $t\in[2^l,2^{l+1})$, and the period $k=\textsc{Period}(t)$ as the unique integer $k$ such that $T_i^{l2^i+k}\leq t<T_i^{l2^i+k+1}$ where $i=\textsc{Phase}(t)$. (Recall that $T_i^{l2^i}=2^l$). We will refer to $[T_i^{l2^i+k},T_i^{l2^i+k+1})$ as period $k$ of stage $l$ of phase $i$. The category of $t$ is then defined in terms of number of occurrences of $X_t$ within its period.
\begin{equation*}
    \textsc{Category}(t,\Xbb_{\leq t}) := \left\lfloor \log_4 \sum_{t'=T_i^{l2^i+k}}^t \1[X_{t'}=X_t] \right\rfloor,
\end{equation*}
where $i=\textsc{Phase}(t)$, $l=\textsc{Stage}(t)$, $k=\textsc{Period}(t)$. For conciseness, we will omit the argument $\Xbb_{\leq t}$ of the function in the rest of the proof. In words, category $p$ contains duplicates with indices in $[4^p,4^{p+1})$ within the periods defined by $\Tcal$. Now using \cref{lemma:countable_dense}, let $(\pi^l)_{l\geq 1}$ be a sequence of dense functions from $\Xcal$ to $\Acal$ within measurable functions under $\Ccal_1'$ processes. The learning rule acts separately on times from different categories. We now fix a category $p$ and only consider points from this category. Essentially, between times $T_i$ and $T_{i+1}$, the learning rule performs the Hedge algorithm for learning with experts to select between the strategies $j$ for $1\leq j\leq i$, which apply $\pi^j$ and a strategy 0 which assigns a different $\EXPIX$ learner to each new instance within each period at scale $i$.

Precisely, during an initial phase $[1,2^{u(16p)})$, the learning rule only applies strategy 0. Then, let $l\geq u(16p)$ and $u(i)\leq l<u(i+1)$, we define the learning rule on stage $[2^l,2^{l+1})$ as follows. For $0\leq k<2^i$, before period $k$ of stage $l$, we construct probabilities $P_p(l,k;j)$ for $j=0,\ldots,i$. These will be probabilities of exploration for each strategy. At the first phase $k=0$ we initialize at the uniform distribution $P_p(l,0;j) = \frac{1}{i+1}$. During period $k$, each new time of category $p$ is assigned a strategy $\hat j(t)$ sampled independently from the past according to probabilities $P_p(l,k;\cdot)$. Duplicates of $X_t$ within the same category and period are also assigned the same strategy $\hat j(t)$. The learning rule then performs the assigned strategy: for $\hat j=0$, it performs an $\EXPIX$ algorithm and for $1\leq \hat j\leq i$, it applies the policy $\pi^{\hat j}$. At the end of the phase, the learning rule computes the average reward obtained by each strategy,
\begin{equation*}
    \tilde r_p(l,k;j) := \frac{1}{2^{l-i}}\sum_{T_i^{l2^i+k} \leq t < T_i^{l2^i+k+1} }  \frac{\1[\textsc{Category}(t)=p,\; \hat j(t)=j]}{P_p(l,k;j)} r_t,
\end{equation*}
and $\hat r_p(l,k+1;j) = \sum_{0\leq k'\leq k} \tilde r_p(l,k';j)$ the cumulative average reward of strategy $j$. These rewards are then used to define the probabilities for the next phase $P_p(l,k+1;\cdot)$ using the exponentially weighted averages.
\begin{equation*}
    P_p(l,k+1;j) = \frac{\exp(\eta_i \hat r_p(l,k+1;j))} {\sum_{j'=0}^i \exp(\eta_i  \hat r_p(l,k+1;j'))},
\end{equation*}
where $\eta_i = \sqrt{\frac{8\ln(i+1)}{2^i}}$ is the parameter of the Hedge algorithm for $2^i$ steps. The detailed algorithm is given in \cref{alg:online_learning_over_bandits}.

\begin{algorithm}[ht]
\caption{Learning rule for $\Ccal_5$ processes on times $\Tcal_p$}\label{alg:online_learning_over_bandits}
\hrule height\algoheightrule\kern3pt\relax

$\eta_i = \sqrt{\frac{8\ln(i+1)}{2^i}},i\geq 0$ \tcp*[f]{learning rates for Hedge}\\
$\hat r^j_p(l,0)=0, P_p(l,0;j) = \frac{1}{i+1},\quad p,l,j\geq 0$ \tcp*[f]{initialization}\\
\For{$t\geq 1$}{
    Observe context $X_t$\\
    $i=\textsc{Phase}(t)$, $l=\textsc{Stage}(t)$, $k=\textsc{Period}(t)$, $p=\textsc{Category}(t)$, \\
    $S_t = \{ t'\in [T_i^{l2^i+k},t): \textsc{Category}(t')=p , X_{t'}=X_t\}$\\
    \uIf(\tcp*[f]{initially play strategy $0$}){$t<2^{u(16p)}$}{
        $\hat a_t = \EXPIX_{\Acal}(\mb{\hat a}_{S_t}, \mb r_{S_t})$
    }
    \Else{
        \textbf{if} $S_t=\emptyset$ \textbf{then} $\hat j(t) \sim P_p(l,k;\cdot)$ \tcp*[f]{select strategy $\hat j(t)$}\\
        \textbf{else} $\hat j(t) = \hat j(\min S_t)$\\
        
        \textbf{if} $\hat j(t)=0$ \textbf{then} $\hat a_t = \EXPIX_{\Acal}(\mb{\hat a}_{S_t}, \mb r_{S_t})$ \tcp*[f]{play strategy $\hat j(t)$}\\
        \textbf{else} $\hat a_t = \pi^{\hat j(t)}(X_t)$
    }
    Receive reward $r_t$\\
    \If(\tcp*[f]{update probabilities}){$l\geq u(16p)$, $t=T_i^{l2^i+k+1}-1$}{
        $\hat r_p(l,k+1;j) = \hat r(l,k;j) + \frac{1}{2^{l-i}} \sum_{t\in [T_i^{l2^i + k},T_i^{l2^i+k+1})} \frac{\1[\textsc{Category}(t)=p,\hat j(t)=j]}{P_p(l,k;j)} r_t ,\quad 0\leq j\leq i$\\
        $P_p(l,k+1;j) = \frac{\exp(\eta_i \hat r_p(l,k+1;j))} {\sum_{j'=0}^i \exp(\eta_i  \hat r_p(l,k+1;j'))},\quad 0\leq j\leq i$
    }
}

\hrule height\algoheightrule\kern3pt\relax
\end{algorithm}

We now show that this is a universally consistent algorithm for $\Xbb$. We first introduce some notations. For $p\geq 0$,
\begin{equation*}
    \Tcal_p := \bigcup_{i\geq 1} [T_i,T_{i+1})\cap \left\{t\geq 1: T^k_i\leq t<T^{k+1}_i , 4^p\leq \sum_{t'=T^k_i}^t \1[X_{t'}=X_t] < 4^{p+1}\right\},
\end{equation*}
is the set of times in category $p$. We will also denote $\Xbb^p:=(X_t)_{t\in\Tcal^p}$. In this setting, the rewards are independent from the selected actions of the learner. First, note that the constructed rewards $\hat r_p(l,k;j)$ are estimates of the average reward that would have been obtained by strategy $j$ during period $k$ of stage $l$. For convenience, we denote $\Tcal_p(k,l) = [T^{l2^i+k},T^{l2^i+k+1})\cap\Tcal_p$. We denote by $R_p(l,k;j)$ the reward that would have been obtained had we selected always $\hat j=j$ on this period, and $r_p(l,k;j) = \frac{R_p(l,k;j)}{2^{l-i}}$ the average reward of strategy $j$ for $0\leq j\leq i$. For example, for strategy $1\leq j\leq i$ we have $R_p(l,k;j)=\sum_{t\in\Tcal_p(l,k)}r_t(\pi^j(X_t))$. Let $\Xcal_p(l,k) = \{X_t, t\in\Tcal_p(k,l)\}$ the set of visited instances during this period. For $x\in\Xcal_p(l,k)$ we denote $t_p(l,k;x)  =\min\{t\in\Tcal_p(k,l) : X_t=x\}$ the first time of occurrence of $x$ during this period, and $N_p(l,k;x) = |\{t\in\Tcal_p(l,k) : X_t=x\}|$ its number of occurrences. Let $0\leq j\leq i$. We use Hoeffding's inequality conditionally on $\Xbb$ and $P_p(l,k;j)$, to obtain
\begin{multline*}
    \Pbb\left[\left|\sum_{x\in\Xcal_p(l,k)} \1[\hat j(t)=j]\sum_{t\in\Tcal_p(l,k), X_t=x}  r_t - P_p(l,k;j) R_p(l,k;j)\right| \right. \\
    \left.\geq P_p(l,k;j) 4^{p+1} 2^{\frac{3}{4}(l-i)} \mid \Xbb,P_p(l,k;j)\right]\\
    \leq 2\exp \left(-2\frac{P_p(l,k;j)^2 2^{3/2(l-i)}}{|\Xcal_p(l,k)|}\right) \leq 2\exp \left(-2\frac{ 2^{3/2(l-i)}}{(i+1)^2e^{\eta_i2^{i+1}}|\Xcal_p(l,k)|}\right).
\end{multline*}
Now by construction of $\Tcal_p(l,k)$, each instance of $\Xcal_p(l,k)$ has at least $4^p$ duplicates within the same period. Hence $|\Xcal_p(l,k)|\leq \frac{2^{l-i}}{4^p}$. As a result, dividing the inner inequality by $P_p(l,k;j)2^{l-i}$, we obtain for $l\geq u(16p)$, with probability at least $1-2\exp \left(-\frac{ 2^{2p+(l-i)/2}}{(i+1)^2 e^{\eta_i2^{i+1}}}\right) := 1- p_1(l,k;p)$,
\begin{equation}\label{eq:accuracy_estimates}
     |\hat r_p(l,k;j) - r_p(l,k;j) | < \frac{4^{p+1}}{2^{(l-i)/4}} \leq \frac{4}{2^{l/16}},
\end{equation}
where in the last inequality we used $l\geq u(i)\geq 2i$ and $l\geq u(16p)\geq 32p$. We now focus on the rewards for strategy 0. For any $t\in \Tcal_p(l,k)$ we denote by $\tilde r_t$ the reward that would have been obtained had we selected strategy $0$ for time $t$, i.e. $\hat j(t_p(l,k;X_t))=0$. In particular, we have $R_p(l,k;0) = \sum_{t\in\Tcal_p(l,k)}\tilde r_t$. Let $\pi^*:\Xcal\to\Acal$ be a measurable policy, we now compare $R_p(l,k;0)$ to the rewards obtained by the policy $\pi^*$ on $\Tcal_p(l,k)$. Intuitively, we wish to apply \cref{thm:multiarmed_bandits} independently for each $\EXPIX$ algorithm corresponding to elements of $\Xcal_p(l,k)$. However, these runs are not independent for general adaptive adversaries. Therefore, we will need to go back to the standard analysis of $\EXPIX$. Using the same notations as in this analysis, for $t\in\Tcal_p(l,k)$, denote $u(t) = |\{t'\leq t:t'\in\Tcal_p(l,k), X_{t'}=X_t\}|$ the index of $t$ for its corresponding $\EXPIX$ learner. Let $\eta_u=2\gamma_u=\sqrt{\frac{\ln|\Acal|}{u|\Acal|}}$ be the parameters used by the learner at step $u$. Also, denote by $p_{t,a}$ the probability that the $\EXPIX$ learner chose $a\in\Acal$ at time $t$. Further, for $a\in\Acal$ denote by $\ell_{t,a}=1-r_t(a)$ and $\tilde\ell_{t,a}=\frac{\ell_{t,a}}{p_{t,a}+\gamma_{u(t)}}\1[a\text{ selected}]$. We keep in mind that the term ``selected'' refers to the selection of the $\EXPIX$ algorithm, but not necessarily the selection of our learning rule, which potentially did not apply strategy 0 at that time. To avoid confusion, for $t\in\Tcal_p(l,k)$, denote $\tilde a_t$ the action that would be selected by the $\EXPIX$ learner at time $t$. Last, we define 
\begin{equation*}
    A_p(l,k) = \sum_{t\in\Tcal_p(l,k)}\tilde\ell_{t,\pi^*(X_t)} - \ell_{t,\pi^*(X_t)}\quad\text{and}\quad B_p(l,k) = \sum_{t\in\Tcal_p(l,k)}  \sum_{a\in\Acal}\eta_{u(t)}(\tilde \ell_{t,a}-\ell_{t,a}).
\end{equation*}
Then, the same arguments as in \cref{prop:EXP.IX_parrallel} give
\begin{align*}
    \sum_{t\in\Tcal_p(l,k)} r_t(\pi^*(X_t))-r_t(\tilde a_t)&\leq  A_p(l,k)+B_p(l,k) + \sum_{x\in\Xcal_p(l,k)}3\sqrt{|\Acal|\ln|\Acal| N_p(l,k;x)} \\
    &\leq A_p(l,k)+B_p(l,k) + 3\sqrt{|\Acal|\ln|\Acal|4^{p+1}}|\Xcal_p(l,k)|\\
    &\leq A_p(l,k)+B_p(l,k) + 6\sqrt{|\Acal|\ln|\Acal|}2^{-p}2^{l-i},
\end{align*}
where in the last inequality, we used the fact that $|\Xcal_p(l,k)|\leq \frac{2^{l-i}}{4^p}$. Now similarly to \cref{prop:EXP.IX_parrallel}, note that conditionally on $\Xbb$, the increments of $A_p(l,k)$ and $B_p(l,k)$ form a super-martingale with increments upper bounded by $2\sqrt{\frac{|\Acal|4^{p+1}}{\ln|\Acal|}}$ and $2|\Acal|\sqrt{\frac{|\Acal|4^{p+1}}{\ln|\Acal|}}$ respectively. Thus, Azuma's inequality implies
\begin{align*}
    \Pbb[A_p(l,k) \leq 8p|\Acal| 2^{p+\frac{3}{4}(l-i)}\mid \Xbb]&\geq 1-e^{-2p^2 2^{(l-i)/2}},\\
    \Pbb[B_p(l,k) \leq 8p |\Acal|^2 2^{p+\frac{3}{4}(l-i)}\mid \Xbb]&\geq 1-e^{-2p^2 2^{(l-i)/2}}.
\end{align*}
Thus, denoting $\delta_p=6\frac{\sqrt{|\Acal|\ln|\Acal|}}{2^p}$, for any $l\geq 2i,u(16p)$, with probability at least $1-2e^{-2p^2 2^{(l-i)/2}}:=1-p_2(l,k;p)$, we have
\begin{equation}\label{eq:accuracy_strat0}
    R_p(l,k;0)\geq \sum_{t\in\Tcal_p(l,k)}r_t(\pi^*(X_t)) -16|\Acal|^2 2^{-i} 2^{15l/16}- \delta_p 2^{l-i}.
\end{equation}
In the first phase where $l<u(16p)$, we will need to proceed differently. Let $\Tcal^{init} = \bigcup_{p\geq 0} \{t\in\Tcal_p:t<2^{u(16p)}\}$. Observe that on these times, the learning uses a distinct $\EXPIX$ learner for each new instance within each category and period. In \cref{prop:EXP.IX_parrallel} we showed that this learning rule is universally consistent under processes visiting a sublinear number of distinct instances almost surely. We now show that this is the case for the process $(X_t)_{t\in\Tcal^{init}}$ where for any $t,t'\in\Tcal^{init}$, we view $X_t$ and $X_{t'}$ as duplicates if and only if $X_t=X_{t'}$ and they have same category and period. For $l\geq 1$, let $p(l)$ denote the index $p$ such that $u(16p)\leq l<u(16(p+1))$ and $i(l)$ be the index $i$ such that $u(i)\leq l<u(i+1)$. Fix $T\geq 1$ and let $l\geq 0$ such that $2^l\leq T<2^{l+1}$. We now count the number of distinct instances $N(T)$ of $(X_t)_{t\in\Tcal^{init}}$ before time $T$. To do so, we distinguish whether $t\leq 2^{l/2}$ or $t>2^{l/2}$ as follows,
\begin{align*}
    N(T) \leq \sum_{p\geq 0}\sum_{l'\leq u(16p),l}\sum_k |\Xcal_p(l',k)|
    &\leq 2^{l/2} +  \sum_{p\geq p(\frac{l}{2})}  \sum_{\frac{l}{2}\leq l'\leq l}\sum_k |\Xcal_p(l',k)| \\
    &\leq 2^{l/2} + \sum_{p\geq p(\frac{l}{2})}  \sum_{\frac{l}{2}\leq l'\leq l}\sum_k \frac{2^{l'-i(l')}}{4^p}\\
    &\leq 2^{l/2} + \sum_{p\geq p(\frac{l}{2})}\frac{2^{l+1}}{4^p}\\
    &\leq 2^{l/2} + \frac{2^{l+1}}{4^{p(l/2)-1}}\\
    &\leq \sqrt T + \frac{8T}{4^{p(\log_4(T))}} = o(T).
\end{align*}
Now let $\pi^*:\Xcal\to\Acal$ a measurable policy. Because of the above estimate, \cref{prop:EXP.IX_parrallel} implies that on an event $\Ecal$ of probability one,
\begin{equation*}
    \limsup_{T\to\infty}\frac{1}{T}\sum_{t\leq T,t\in\Tcal^{init}} r_t(\pi^*(X_t))-r_t\leq 0.
\end{equation*}
Now recall that $l\geq u(i)\geq 2i, \eta_i 2^{i+5}$, hence $\frac{2^{(l-i)/2}}{e^{\eta_i2^{i+1}}}\geq 2^{l/4-\eta_i2^{i+2}}\geq 2^{l/8}$. As a result,
\begin{equation*}
    \sum_{p\geq 0} \sum_{l\geq 32p}\sum_{k} (i+1)p_1(l,k;p)+ p_2(l,k;p)<\infty.
\end{equation*}
Then, the Borel-Cantelli lemma implies that on an event $\Fcal$ of probability one, there exists $\hat l$ such that for all $p\geq 0$, $l\geq \max(\hat l,u(16p))$ Eq~\eqref{eq:accuracy_estimates} holds, for all $p\geq 0$ and $l\geq \hat l$, Eq~\eqref{eq:accuracy_strat0} holds, and $\Ecal$ is satisfied. We suppose that this event is met in the rest of the proof.

The probabilities $P_p(l,k;j)$ are chosen according to the Hedge algorithm. As a result, we have that for any $l\geq \max(\hat l,u(16p))$, $0\leq k<2^i$,
\begin{equation*}
    \max_{0\leq j\leq i} \sum_{k'\leq k} \hat r_p(l,k;j) - \sum_{k'\leq k}\sum_{j=0}^i P_p(l,k;j)\hat r_p(l,k;j) \leq \frac{\ln (i+1)}{\eta_i} + \frac{(k+1)\eta_i}{8}.
\end{equation*}
We then use Eq~\eqref{eq:accuracy_estimates} and $k+1\leq 2^i$ to obtain
\begin{equation*}
    \max_{0\leq j\leq i} \sum_{k'\leq k} r_p(l,k;j) - \sum_{k'\leq k}\sum_{j=0}^i P_p(l,k;j) \hat r_p(l,k;j) 
    \leq 2^i\frac{4}{2^{l/16}} + \frac{\eta_i}{4} 2^i
\end{equation*}
As a result,
\begin{equation}\label{eq:outer_hedge}
    \max_{0\leq j\leq i} \sum_{k'\leq k} R_p(l,k;j) - \sum_{k'\leq k} \sum_{t\in\Tcal_p(l,k)}r_t  \leq 4 \cdot 2^{15l/16} + \frac{\eta_i}{4} 2^l.
\end{equation}
Now because $l\geq u(16p)$, we have $i\geq 16p$, we have
\begin{align*}
    \sum_{0\leq k'\leq k}\sum_{t\in\Tcal_p(l,k')}r_t &\geq  \sum_{0\leq k'\leq k} R_p(l,k';0) - 4 \cdot 2^{15l/16} - \frac{\eta_{16p}}{4}2^l\\
    &\geq \sum_{0\leq k'\leq k}\sum_{t\in\Tcal_p(l,k)}r_t(\pi^*(X_t)) - 20|\Acal|^2 2^{15l/16} -\left(\delta_p+\frac{\eta_{16p}}{4}\right) 2^l,
\end{align*}
where in the second inequality we used Eq~\eqref{eq:accuracy_strat0}. Therefore, summing these equations, for any $T\geq 2^{\hat l},2^{u(16p)}$,
\begin{equation}\label{eq:tail_upper_bound}
    \sum_{2^{u(16p)}<t\leq T,t\in\Tcal_p} r_t(\pi^*(X_t))-r_t \leq 2^{\hat l}+ c|\Acal|^2 T^{15/16} + 2\left(\delta_p+\frac{\eta_{16p}}{4}\right) T,
\end{equation}
where $c=\frac{20}{1-2^{-15/16}}$.
An important remark is that $\sum_{p\geq 0}(\delta_p + \frac{\eta_{16p}}{4})<\infty $, which will allow us to consider only a finite number of $p\geq 0$ when comparing the performance of the learning rule compared to $\pi^*$.

Before doing so, we show that for all $p\geq 0$, we have $\Xbb^p = (X_t)_{t\in\Tcal_p}\in\Ccal_1'$. By definition, letting $\Tcal = \bigcup_{i\geq 0}\Tcal^i\cap\{t\geq T_i\}$, we have that $\tilde\Xbb = (X_t)_{t\in\Tcal}\in\Ccal_1'$. Then note that each instance of $[T^k_i,T^{k+1}_i)\cap\Tcal_p$ has at least one duplicate in $[T^k_i,T^{k+1}_i)\cap\Tcal$ and to each instance of $[T^k_i,T^{k+1}_i)\cap\Tcal$ corresponds at most $4^{p+1}$ duplicates in $[T^k_i,T^{k+1}_i)\cap\Tcal_p$. As a result, for any set $A\in\Bcal$, we have $\hat \mu_{\Xbb^p}(A)\leq 4^{p+1} \hat \mu_{\tilde \Xbb}(A)$, which yields $\Ebb[\hat \mu_{\Xbb^p}(A)]\leq 4^{p+1}\Ebb[\hat \mu_{\tilde \Xbb}(A)]$. Using the definition of $\Ccal_1'$ processes ends the proof that $\Xbb^p\in\Ccal_1'$ for all $p\geq 0$.

Now let $\epsilon>0$ and $p_0$ such that $\sum_{p\geq p_0}(\delta_p + \frac{\eta_{16p}}{4})<\epsilon$. Recall that if $t\in\Tcal_p$, we have $t\geq 4^p$. Therefore, summing Eq~\eqref{eq:tail_upper_bound} gives
\begin{align*}
    \sum_{p\geq p_0}\sum_{2^{u(16p)}\leq t<T,t\in\Tcal_p}  r_t(\pi^*(X_t))-r_t &\leq \sum_{p_0\leq p\leq \log_4 T}\sum_{2^{u(16p)}\leq t<T,t\in\Tcal_p} r_t(\pi(X_t))-r_t\\
    &\leq 2^{\hat l} \log_4 T + c|\Acal|^2T^{15/16}\log_4 T+ \epsilon T.
\end{align*}
We now treat the case of $p<p
_0$. Because $\Xbb^p\in\Ccal_1'$, by \cref{lemma:countable_dense}, there exists $r^p\geq 1$ such that
\begin{equation*}
    \Ebb\left[\limsup_{T\to\infty} \frac{1}{T}\sum_{t\leq T, t\in\Tcal_p} \1[\pi^*(X_t)\neq \pi^{r_p}(X_t)]\right] \leq \frac{\epsilon^2}{2p_0^2}.
\end{equation*}
By dominated convergence theorem, let $l^p$ such that 
\begin{equation*}
    \Ebb\left[\sup_{T\geq 2^{l^p}} \frac{1}{T}\sum_{t\leq T, t\in\Tcal_p} \1[\pi^*(X_t)\neq \pi^{r_p}(X_t)]\right] \leq \frac{\epsilon^2}{p_0^2}.
\end{equation*}
Using the Markov inequality, we have
\begin{equation*}
    \Pbb\left[\sup_{T\geq 2^{l^p}} \frac{1}{T}\sum_{t\leq T, t\in\Tcal_p} \1[\pi^*(X_t)\neq \pi^{r_p}(X_t)] \geq \frac{\epsilon}{p_0}\right] \leq \frac{\epsilon}{p_0}.
\end{equation*}
By union bound, on an event $\Gcal$ of probability at least $1-\epsilon$, for all $p<p_0$ and $T\geq 2^{l^p}$, we have $\sum_{t\leq T, t\in\Tcal_p} \1[\pi(X_t)\neq \pi^{r_p}(X_t)] < \frac{\epsilon}{p_0}T$. Next, let $l_0 = \max(u(r^p),l^p, p< p_0)$. Thus, any phase $l\geq l_0$, has $r^p\leq i$ for all $p<p_0$. Last, let $i_0$ such that $\eta_{i_0}\leq 2\frac{\epsilon}{p_0}$. On the event $\Ecal\cap\Fcal\cap \Gcal$, for $p<p_0$, for any $l\geq \hat l_1:=\max(l_0,32p_0,u(i_0),\hat l)$ and $0\leq k<2^i$, Eq~\eqref{eq:outer_hedge} yields
\begin{equation*}
    \sum_{2^l\leq t<T_i^{l2^i+k+1},t\in\Tcal_p} r_t(\pi^{r^p}(X_t)) - r_t \leq 4\cdot 2^{15l/16} + \frac{\eta_i}{4}2^l  \leq 4\cdot 2^{15l/16} + \frac{\epsilon}{2p_0}2^l
\end{equation*}
As a result, for $T\geq 1$, letting $i(T),l(T)$ the indices $i,l$ such that $2^{u(i)}\leq T< 2^{u(i+1)}$ and $2^l\leq T<2^{l+1}$, on $\Ecal\cap\Fcal\cap\Gcal$,
\begin{align*}
    \sum_{p<p_0}\sum_{2^{u(16p)}\leq t\leq T, t\in\Tcal_p} r_t(\pi^*(X_t))-r_t &\leq 2^{\hat l_1}+2^{-i(T)}T + \sum_{p<p_0}\sum_{t<T, t\in\Tcal_p}\1[\pi^*(X_t)\neq \pi^{r_p}(X_t)]\\
    &\quad\quad\quad\quad\quad\quad+ \sum_{p<p_0}\sum_{\hat l_1\leq l'\leq l} \left(4\cdot 2^{15l'/16} + \frac{\epsilon}{2p_0} 2^{l'}\right)\\
    &\leq 2^{\hat l_1} + 2^{-i(T)}T + \epsilon T+ cp_0T^{15/16} + \epsilon T.
\end{align*}
Finally, putting everything together, for $T$ sufficiently large, we have
\begin{align*}
    \sum_{t\leq T} &r_t(\pi^*(X_t))-r_t 
    \leq  \sum_{t\in\Tcal^{init},t\leq T} \bar r_t(\pi^*(X_t))-r_t  + \sum_{p\geq 0} \sum_{2^{u(16p)}\leq t\leq T,t\in\Tcal_p}  r_t(\pi^*(X_t))-r_t\\
    &\leq 2^{\hat l_1+1}\log_4 T +2^{-i(T)}T+ c(|\Acal|^2+p_0)T^{15/16}\log_4 T + 3\epsilon T + \sum_{t\in\Tcal^{init},t\leq T}  r_t(\pi^*(X_t))-r_t,
\end{align*}
which shows that on $\Ecal\cap\Fcal\cap\Gcal$, 
\begin{equation*}
    \limsup_{T\to\infty}\frac{1}{T} \sum_{t=1}^T  r_t(\pi^*(X_t))-r_t \leq 3\epsilon.
\end{equation*}
We denote by $(x)_+=\max(0,x)$ the positive part. Recall that $\Pbb[\Ecal\cap\Fcal\cap\Gcal]\geq 1-\epsilon$. Thus,
\begin{equation*}
    \Ebb\left[\left( \limsup_{T\to\infty}\frac{1}{T}\sum_{t=1}^T r_t(\pi(X_t))-r_t\right)_+\right]\leq 4\epsilon.
\end{equation*}
Because this holds for any $\epsilon>0$, this shows that almost surely, $\limsup_{T\to\infty}\frac{1}{T}\sum_{t=1}^T r_t(\pi(X_t))-r_t \leq 0$. As a result, the learning rule is universally consistent on $\Xbb$. This ends the proof of the theorem.
\end{proof}

To the best of our knowledge, while we believe that for general spaces $\Xcal$ with non-atomic probability measures, one may have a gap $\Ccal_5\subsetneq \Ccal_6$, it seems plausible that $\Ccal_5 = \Ccal_7$. As a consequence, this would imply that we have an exact characterization for processes admitting universal learning with prescient rewards $\Ccal_{prescient} = \Ccal_5=\Ccal_7$.

\paragraph{Comparison to a more natural condition $\Ccal_8$.}
In the rest of this section, we compare condition $\Ccal_5$ to another potentially more natural sufficient condition. \cite{blanchard:22e} showed that given any $\Xbb\in \Ccal_2$ process, only allowing for a finite number of duplicates in $\Xbb$ yields a $\Ccal_1'$ process. Precisely, for any $M$, letting
\begin{equation*}
    \Tcal^{\leq M}= \left\{t\geq 1: \sum_{t'\leq t}\1[X_{t'}=X_t]\leq M \right\},
\end{equation*}
the set of times when contexts are duplicates of index at most $M$, one has $(X_t)_{t\in\Tcal^{\leq M}}\in\Ccal_1'$. However, if one does not restrict the maximum number of duplicates, one loses the $\Ccal_1'$ property. A natural condition on stochastic processes would therefore be that for some increasing rate of maximum number of duplicates, the $\Ccal_1'$ property is conserved. For any process $\Xbb$, we denote the occurrence count as $N_t(x) = \sum_{i=1}^t \1[X_t=x]$ for all $x\in\Xcal$. Then, the condition on stochastic processes can be formally defined as follows.

\paragraph{Condition 8.}There exists an increasing function $\Psi:\Nbb\to\Nbb$ with $\Psi(T)\to\infty$ as $T\to\infty$ such that for any sequence of measurable sets $A_i\in\Bcal$ for $i\geq 1$ with $A_i\downarrow\emptyset$,
\begin{equation*}
    \Ebb \left[\limsup_{T\to\infty} \frac{1}{T} \sum_{t=1}^T \1_{A_i}(X_t)\1_{N_t(X_t)\leq \Psi(T)}\right] \to 0.
\end{equation*}

Although this condition is indeed sufficient for universal learning, we show that the more involved $\Ccal_5$ class of processes is larger, and strictly larger whenever $\Xcal$ admits a non-atomic probability measure.

\begin{proposition}\label{prop:larger_class}
    Let $\Xcal$ be a metrizable separable Borel space, then $\Ccal_8\subset\Ccal_5$. Further, if there exists a non-atomic probability measure on $\Xcal$, then $\Ccal_8\subsetneq\Ccal_5$.
\end{proposition}

\begin{proof}
We first show $\Ccal_8\subset\Ccal_5$. Indeed, suppose that $\Xbb\in\Ccal_8$, then there exists $\Psi:\Nbb\to\Nbb$ increasing to infinity such that for any measurable sets $A_k\downarrow\emptyset$, we have
\begin{equation*}
    \Ebb\left[\limsup_{T\to\infty} \frac{1}{T}\sum_{t\leq T, N_t(X_t)\leq \Psi(T)}\1_{A_k}(X_t)\right] \underset{k\to\infty}{\longrightarrow} 0.
\end{equation*}
Now let $T_i\geq 1$ such that $\Psi(T_i) \geq 1+i2^i$. We now show that $(T_i)_i$ satisfies the condition of condition $\Ccal_5$. Let $\Tcal = \bigcup_{i\geq 0}\Tcal^i \cap\{t\geq T_i\}$, and $A_k\downarrow\emptyset$. For any $T\geq 1$, we denote $\Xcal(T)=\{X_t,t\leq T\}$ the set of visited instances. Now fix $k\geq 0$. Then, for $T\geq T_k$, let $i\geq k$ such that $T_i\leq T<T_{i+1}$,
\begin{align*}
    \frac{1}{T}\sum_{t\leq T,t\in\Tcal} \1_{A_k}(X_t) &\leq \frac{1}{2^k} + \frac{1}{T}\sum_{2^{-k}T< t\leq T,t\in\Tcal} \1_{A_k}(X_t)\\
    &= \frac{1}{2^k} + \frac{1}{T}\sum_{x\in\Xcal(T)\cap A_k} |\{2^{-k}T< t\leq T, t\in\Tcal: X_t=x\}|.
\end{align*}
In $\Tcal$, we accept at most one duplicate per phase. Because $T_i\leq T<T_{i+1}$, the interval $[2^{-k}T,T]$ intersects at most $1 + k2^i$ phases. Thus, for any $x\in\Xcal(T)$, $|\{2^{-k}T< t\leq T, t\in\Tcal: X_t=x\}|\leq 1+k2^i\leq 1+i2^i\leq \Psi(T)$. Thus, for any $T\geq T_k$,
\begin{align*}
    \frac{1}{T}\sum_{t\leq T,t\in\Tcal} \1_{A_k}(X_t) &\leq \frac{1}{2^k} + \frac{1}{T}\sum_{x\in\Xcal(T)\cap A_k} \min( |\{t\leq T: X_t=x\}|, \Psi(T))\\
    &=\frac{1}{2^k} + \frac{1}{T}\sum_{t\leq T, N_t(X_t)\leq \Psi(T)} \1_{A_k}(X_t).
\end{align*}
Using the hypothesis on $\Psi$ applied to $A_k\downarrow\emptyset$ yields $\Ebb \left[\limsup_{T\to\infty}\frac{1}{T}\sum_{t\leq T,t\in\Tcal} \1_{A_k}(X_t)  \right] \underset{k\to\infty}{\longrightarrow} 0.$ Hence, this shows that $\tilde\Xbb=(X_t)_{t\in\Tcal} \in\Ccal_1'$ and $\Xbb\in\Ccal_5$.

Next, suppose that there exists a non-atomic probability measure on $\Xcal$. We will construct explicitly a process $\Xbb\in\Ccal_5\setminus\Ccal_8$. By \cref{lemma:disjoint_non-atomic}, there exists a sequence of disjoint measurable sets $(A_i)_{i\geq 0}$ together with non-atomic probability measures $(\nu_i)_{i\geq 0}$ such that $\nu_i(A_i)=1$. We now fix $x_0\in A_0$ an arbitrary instance (we will not use the set $A_0$ any further) and define subsets of indices as follows, $S_i = \{k\geq 1: k\equiv 2^{i-1}\bmod 2^i\}$. Note that the sets $(S_i)_{i\geq 1}$ form a partition of $\Nbb$. We now introduce independent processes $\Zbb^i$ for $i\geq 1$ such that $\Zbb^i=(Z^i_t)_{t\geq 1}$ is an i.i.d. process with distribution $\nu_i$. Last, for all $i\geq 1$ we denote $n_i = 2^{\lfloor \log_2 i\rfloor}$. Now consider the following process $\Xbb$ where $X_1 = x_0$ and for any $t\geq 1$,
\begin{equation*}
    X_t = Z^i_{\lfloor\frac{t}{n_i}\rfloor}  ,\quad 2^k\leq t<2^{k+1}, k\equiv 2^{i-1}\bmod 2^i.
\end{equation*}
When the process is in phase $i$, it corresponds to an i.i.d. process on $A_i$ which is duplicated $n_i$ times. Note that we used $n_i$ duplicates instead of $i$ so that each point is duplicated exactly $n_i$ times (we do not have boundary issues at the end of the phase). We now show that $\Xbb\notin\Ccal_8$. Let $\Psi:\Nbb\to\Nbb$ an increasing function with $\Psi(T)\to\infty$ as $T\to\infty$. For $i\geq 1$, we first construct an increasing sequence of times $T_i$ such that $\Psi(T_i)> n_i$. Then, for any $k\geq 1$, consider times $T^k =k 2^i + 2^{i-1}$ which belong to $S_i$. Then, consider the event $\Fcal_i$ such that the process $\Zbb^i$ only takes distinct values in $A_i$. Note that $\Pbb[\Fcal_i]=1$ because the $\nu_i$ is non-atomic and $\nu_i(A_i)=1$. Then, on $\Fcal_i$, by construction, we have for any $k\geq 0$, with $T^k\geq T_i$,
\begin{align*}
    \frac{1}{2T^k-1}\sum_{t=1}^{2T^k-1} \1_{A_i}(X_t)\1_{N_t(X_t)\leq \Psi(2T^k-1)} &\geq \frac{1}{2T^k}\sum_{t=T^k}^{2T^k-1} \1_{A_i}(X_t)\1_{N_t(X_t)\leq n_i} \\
    & = \frac{1}{2T^k}\sum_{t=T^k}^{2T^k-1} \1_{A_i}(X_t)\\
    &\geq  \frac{T^k}{2T^k-1}.
\end{align*}
Hence, on the event $\Fcal_i$, we have $\limsup_{T\to\infty}\frac{1}{T}\sum_{i=1}^{T} \1_{A_i}(X_t)\1_{N_t(X_t)\leq \Psi(T)}\geq \frac{1}{2}$. Because $\Pbb[\Fcal_i]=1$, we obtain
\begin{equation*}
    \Ebb \left[\limsup_{T\to\infty} \frac{1}{T} \sum_{i=1}^T \1_{A_i}(X_t)\1_{N_t(X_t)\leq \Psi(T)}\right]\geq \frac{1}{2}.
\end{equation*}
Now consider $B_i = \bigcup_{j\geq i}A_i$. Then, we have $B_i\downarrow \emptyset$ and for any $i\geq 1$,
\begin{equation*}
    \Ebb \left[\limsup_{T\to\infty} \frac{1}{T} \sum_{i=1}^T \1_{B_i}(X_t)\1_{N_t(X_t)\leq \Psi(T)}\right]\geq \Ebb \left[\limsup_{T\to\infty} \frac{1}{T} \sum_{i=1}^T \1_{A_i}(X_t)\1_{N_t(X_t)\leq \Psi(T)}\right] \geq \frac{1}{2}.
\end{equation*}
As a result, $\Xbb\notin\Ccal_8$. 

We now show that $\Xbb\in\Ccal_5$. To do so, we first prove that $\Xbb\in\Ccal_2$. Let $(B_l)_{l\geq 1}$ be a sequence of disjoint measurable sets. Because $\Zbb^i$ are i.i.d. processes, we have $\Zbb^i\in\Ccal_2$. In particular, on an event $\Ecal_i$ of probability one, we have
\begin{equation*}
    |\{l: \Zbb^i_{\leq T}\cap B_l\neq\emptyset\}|= o(T).
\end{equation*}
Now consider the event $\Ecal = \bigcap_{i\geq 1} \Ecal_i$. This has probability one by the union bound. Let $\epsilon>0$ and $i^* = \lceil \frac{2}{\epsilon}\rceil$. In particular, we have $\frac{1}{n_{i^*}}\leq \epsilon$. On the event $\Ecal$, for any $i\leq i^*$, there exist $T_i$ such that for all $T\geq T_i$,
\begin{equation*}
    |\{l: \Zbb^i_{\leq T}\cap B_l\neq\emptyset\}|\leq \frac{\epsilon}{2^i} T.
\end{equation*}
Now consider $T^0 = \max_{i\leq i^*} T_i n_i$. Then, for any $T\geq T^0$, we have
\begin{align*}
    |\{l: \Xbb_{\leq T}\cap B_l\neq\emptyset\}|&\leq \sum_{i=1}^{i^*} |\{l: \Zbb^i_{\leq \lfloor T/n_i\rfloor}\cap B_l\neq\emptyset\}|\\
    &+ |\{l: \exists t\leq T: X_t\in B_l, 2^k\leq t<2^{k+1}, k\equiv 0\bmod 2^{i^*}\}|\\
    &\leq \epsilon T + |\{X_t,\quad  t\leq T,2^k\leq t<2^{k+1}, k\equiv 0\bmod 2^{i^*}\}|\\
    &\leq \epsilon T + 2\frac{T}{n_{i^*}},
\end{align*}
where in the last inequality we used the fact that in a phase $i>i^*$, each point is duplicated $n_i\geq n_{i^*}$ times. As a result, on the event $\Ecal$, we have
\begin{equation*}
    \limsup \frac{|\{l: \Xbb_{\leq T}\cap B_l\neq\emptyset\}|}{T}\leq 3\epsilon.
\end{equation*}
Because this holds for all $\epsilon>0$, we obtain that on $\Ecal$, $|\{l: \Xbb_{\leq T}\cap B_l\neq\emptyset\}| = o(T)$. Because $\Ecal$ has probability one, this ends the proof that $\Xbb\in \Ccal_2$. Now consider the following times $T_j = 4^j$ for $j\geq 0$ and define $\Tcal=\bigcup_{j\geq 0}\Tcal^j\cap\{t\geq T_i\}$. We aim to show $\tilde\Xbb =(X_t)_{t\in\Tcal}\in\Ccal_1'$. First, note that for any $j\geq 0$, the phases $[2^k,2^{k+1})$ contained in $[T_j,T_{j+1})$ satisfy $k\leq 2j+1$. Let $i(j) = 1+ \log_2(2j+1)$. We have $k\in \bigcup_{i\leq i(j)} S_i$, which implies that each instance $X_t$ is duplicated consecutively at most $n_{i(j)} $ times in $\Xbb$ within $[T_j,t_{j+1})$. However, the sections defined by $\Tcal$ have length at least $2^{-j}T_j = 2^j$. Further, all the phases were constructed so that there are no boundary issues: if $n_{i(j)}\leq 2^j$, then $\Tcal$ does not contain any duplicates during the period $[T_j,T_{j+1})$. Because $n_{i(j)} \leq i(j) = o(2^j)$, there exists $j_0\geq 0$ such that $\Tcal$ does not contain any duplicate on $[T_{j_0},\infty)$. Let $\Tcal(0) = \{t\geq 1: N_t(X_t)=1\}$ the set of first appearances. Then, for any $A\in\Bcal$ and $T\geq 1$,
\begin{equation*}
    \sum_{t\leq T, t\in\Tcal}\1_A(X_t) \leq T_{j_0} + \sum_{t\leq T, t\in\Tcal(0)}\1_A(X_t).
\end{equation*}
Now because $\Xbb\in\Ccal_2$, we have $(X_t)_{t\in\Tcal(0)}\in\Ccal_1'$ which implies $(X_t)_{t\in\Tcal}\in\Ccal_1'$ by the above inequality. This ends the proof of the proposition.
\end{proof}

\subsection{Universal learning with fixed excess error tolerance}
\label{subsec:fixed_excess_error}

In this section, we show that as an application of the methods developed in \cite{blanchard:22e} and in this paper, achieving a fixed excess regret $\epsilon>0$ is always possible for $\Ccal_2$ processes. This is stated in \cref{prop:can_fixed_error}. We first need to state a result from \cite{blanchard:22e} showing that $\Ccal_2$ processes without duplicates are $\Ccal_1'$ extended processes.

\begin{lemma}[\cite{blanchard:22e}]
\label{lemma:equivalent_from_C2}
    Let $\Xbb$ be a stochastic process on $\Xcal$, and define for any $M\geq 1$,
\begin{equation*}
    \Tcal^{\leq M} = \left\{t\geq 1: \sum_{t'\leq t}\1[X_{t'}=X_t]\leq M \right\},
\end{equation*}
the set of times which are duplicates of index at most $M$. In particular, $\Tcal^{\leq 1}$ is the set of times where we delete all duplicates. The following are equivalent.
\begin{enumerate}
    \item $\Xbb\in\Ccal_2$.
    \item For all $M\geq 1$, $(X_t)_{t\in\Tcal^{\leq M}} \in\Ccal_1'$.
\end{enumerate}
\end{lemma}

We are now ready to prove \cref{prop:can_fixed_error}.

\begin{proof}[ of \cref{prop:can_fixed_error}]
    We first describe the algorithm that depends on a parameter $M\geq 1$ which we will fix later. We use the notation $\Tcal^{\leq M}$ from \cref{lemma:equivalent_from_C2} for the set of times that are duplicates of index at most $M$. Note that whether $t\in\Tcal^M$ or $t\notin\Tcal^M$ can be decided in an online manner. Next we fix a sequence $\Pi=(\pi^l)_{l\geq 1}$ of policies that are dense within $\Ccal_1'$ processes from \cref{lemma:countable_dense}. The learning rule $f_\cdot$ simply performs the $\EXPINF$ strategy on the sequence $\Pi$ for times in $\Tcal^{\leq M}$ and for other times performs independent copies of the $\EXPIX$ algorithm in parallel for each distinct instance. Formally, for any $t\geq 1$, instances $\mb x_{\leq t}$ and observed rewards $\mb r_{\leq t-1}$, we define
    \begin{equation*}
        f_t(\mb x_{\leq t-1},\mb r_{\leq t-1},x_t)=\begin{cases}
            \EXPINF(\mb x_{U_t},\mb {\hat a}_{U_t},\mb r_{U_t},x_t) &\text{if }t\in\Tcal^M\\
            \EXPIX_\Acal(\mb {\hat a}_{S_t},\mb r_{S_t}) & \text{o.w.}
        \end{cases}
    \end{equation*}
    where $U_t=\{t'\leq t-1:t\in\Tcal^M\}$ and $S_t = \{t'<t:x_t=x_{t'}, t'\in\Tcal^M\}$ and $\hat a_{t'}$ denotes the action selected at time $t'\leq t-1$. 
    
    Let $\Xbb\in\Ccal_2$. We now prove that this learning rule achieves low excess error compared to a fixed measurable policy $\pi^*:\Xcal\to\Acal$. We denote by $\hat a_t(M)$ its selected action at time $t$. First, by \cref{lemma:equivalent_from_C2}, $\tilde\Xbb = (X_t)_{t\in\Tcal^M}\in\Ccal_1'$. Further, as discussed in \cref{subsec:sufficient_conditions}, the same proof of universal consistence of $\EXPINF$ under $\Ccal_1$ processes for stationary rewards given in \cite{blanchard:22e} shows that $\EXPINF$ is universally consistent under $\Ccal_1'$ extended processes for adversarial rewards. This is a consequence from the fact that the regret guarantee of $\EXPIX$---\cref{thm:multiarmed_bandits}---holds for adversarial rewards as well. Thus, on an event $\Ecal$ of probability one,
    \begin{equation*}
        \limsup_{T\to\infty} \frac{1}{T}\sum_{t\leq T,t\in\Tcal^M} r_t(\pi^*(X_t))-r_t(\hat a_t(M)) \leq 0.
    \end{equation*}
    Next, similarly to the proof of \cref{prop:EXP.IX_parrallel}, let $\epsilon(T) = \frac{1}{T}|\{X_t: t\leq T, t\notin\Tcal^M\}|$. The same proof as in \cref{prop:EXP.IX_parrallel} shows that on an event $\Fcal$ of probability one, for all $T\geq 1$,
    \begin{multline*}
        \frac{1}{T}\sum_{t\leq T, t\notin\Tcal^M} r_t(\pi^*(X_t)) - r_t(\hat a_t(M)) \\
        \leq  8|\Acal| \frac{\ln T}{T^{1/4}} + 3c \sqrt{|\Acal| \ln |\Acal|} \frac{1}{\ln T} + \sqrt{\epsilon(T)}+3\sqrt{|\Acal|\ln|\Acal| } \epsilon(T)^{1/4}.
    \end{multline*}
    Note that to each element of $\{X_t: t\leq T, t\notin\Tcal^M\}$ correspond least $M$ duplicates in $\Tcal^M$ so that $\epsilon(T)\leq \frac{1}{M}$. As a result, combining the two previous equations yields on $\Ecal\cap\Fcal$ of probability one,
    \begin{equation*}
        \limsup_{T\to\infty} \frac{1}{T}\sum_{t\leq T,t\in\Tcal^M} r_t(\pi^*(X_t))-r_t(\hat a_t(M)) \leq 4\frac{\sqrt{|\Acal|\ln|\Acal| }}{M^{1/4}}.
    \end{equation*}
    Thus, taking $M\geq 4^4|\Acal|^2\ln^2|\Acal|\epsilon^{-4}$ gives a learning rule with the desired $\epsilon$ excess error almost surely. This ends the proof of the proposition. 
\end{proof} 

\section{Model extensions}
\label{sec:model_extensions}

\subsection{Infinite action spaces}
\label{subsec:infinite_action_spaces}

The previous sections focused on the case of finite action spaces. For infinite action spaces, we argue that as a direct consequence from the analysis of the stationary case in \cite{blanchard:22e}, one can obtain a characterization of learnable processes and same optimistically universal learning rules.

For countably infinite action spaces, they showed that $\EXPINF$ performed with the countable sequence of dense policies given by \cref{lemma:countable_dense} is universally consistent under $\Ccal_1$ processes with stationary rewards, and that $\Ccal_1$ is necessary. As discussed in \cref{subsec:sufficient_conditions,subsec:fixed_excess_error}, the same arguments as in \cite{blanchard:22e} show that $\EXPINF$ is universally consistent under $\Ccal_1$ processes for adversarial rewards as well. Further, since adversarial rewards generalize stationary rewards, $\Ccal_1$ is still necessary for universal learning. Thus, $\Ccal_{online} = \Ccal_{prescient} = \Ccal_{oblivious} = \Ccal_{memoryless} = \Ccal_{stat} = \Ccal_1$ and $\EXPINF$ is optimistically universal in all reward settings.

For uncountable separable metrizable Borel action spaces $\Acal$, even for stationary rewards, universal learning is impossible \cite{blanchard:22e}. Hence, $\Ccal_{online} = \Ccal_{prescient} = \Ccal_{oblivious} = \Ccal_{memoryless} = \Ccal_{stat} = \emptyset$.

\subsection{Unbounded rewards}
\label{subsec:unbounded_rewards}

We now turn to the case of unbounded rewards $\Rcal=[0,\infty)$. We further suppose that for any $t\geq 1$, and history $\mb x \in\Xcal^\infty, \mb a_{\leq t}\in \Acal^t, \mb r_{\leq t-1} \in\Rcal^{t-1}$, the random variable $r_t(a_t\mid \Xbb=\mb x, \mb{\hat a_{\leq t-1}} = \mb a_{\leq t-1}, \mb r{(\hat a)}_{\leq t-1} = \mb r_{\leq t-1})$ is integrable so that the immediate expected reward is well defined. Again, in this case, adversarial rewards yield the same results as stationary rewards. Clearly, for uncountable separable metrizable Borel action spaces, under unbounded rewards, universal learning is still impossible $\Ccal_{online} = \Ccal_{prescient} = \Ccal_{oblivious} = \Ccal_{memoryless} = \Ccal_{stat} = \emptyset$, because this was alreay the case for bounded rewards.

For countable action spaces $\Acal$, condition $\Ccal_3$ is necessary even under the full-feedback noiseless setting \cite{hanneke:21,blanchard:22b}, hence necessary for contextual bandits as well. Also, \cite{blanchard:22e} proposed the algorithm which runs an independent $\EXPINF$ learner on each distinct context instance, which is universally consistent under $\Ccal_3$ processes. As in the previous section, this guarantee still holds for adversarial rewards, and $\Ccal_3$ is still necessary for universal learning. Therefore, $\Ccal_{online} = \Ccal_{prescient} = \Ccal_{oblivious} = \Ccal_{memoryless} = \Ccal_{stat} = \Ccal_3$.

\subsection{Uniformly-continuous rewards}
\label{subsec:uniformly_continuous_rewards}

We assume that the rewards are bounded again. In the previous sections, we showed that for finite action sets, universal learning is possibly under large classes of processes, namely at least on $\Ccal_5$ processes. However, for countable action sets, this is reduced to $\Ccal_1$ and for uncountable action sets, universal learning is not achievable. Therefore, imposing no constraints on the rewards is too restrictive for universal learning in the last cases. Here, we investigate the case when $\Acal$ is a separable metric space given with a metric $d$, and the rewards are uniformly-continuous. Crucially, modulus of continuity should be uniformly bounded over time as well. We recall the definition of uniformly-continuous rewards.

\DefinitionUniformlyContinuousRewards*

In the definition, the expectation is taken over the rewards' randomness, in the event when the context sequence until $t$ is exactly $\mb x_{\leq t}$, the learner selected actions $\mb a_{\leq t-1}$ and received rewards $\mb r_{\leq t-1}$ in the first $t-1$ steps. For instance, for stationary rewards, only $x_t$ is relevant in this expectation, while for online rewards, $\mb x_{\leq t},\mb a_{\leq t-1},\mb r_{\leq t-1}$ may be relevant. The above definition is not written for prescient rewards for simplicity. For these, we need to condition on the complete sequence $\Xbb$: 
\begin{multline*}
    \forall t\geq 1, \forall (\mb x,\mb a_{\leq t-1},\mb r_{\leq t-1}) \in\Xcal^\infty \times\Acal^{t-1}\times \Rcal^{t-1},\forall a,a'\in \Acal,\\
    \quad d(a,a') \leq \Delta(\epsilon)\Rightarrow \left|\Ebb[r_t(a)-r_t(a') \mid \Xbb = \mb x, \mb a_{\leq t-1},\mb r_{\leq t-1}]\right|\leq \epsilon.
\end{multline*}
As in the unrestricted rewards case, we consider the set of processes $\Ccal^{uc}_{setting}$ admitting universal learning for uniformly-continuous rewards under any chosen reward setting. The uniform-continuity assumption defined above generalizes the corresponding assumption proposed in \cite{blanchard:22e} for stationary rewards. They also proposed a weaker continuity assumption on the immediate expected rewards, however, similarly as in \cref{subsec:infinite_action_spaces} one can easily check that with this reward assumption, adversarial settings give the same results as the stationary case.

The goal of this section is to show that under the mild uniform-continuity assumption on the rewards, one can recover all the results from the finite action space case, when the action space is totally-bounded. We first start by showing that the derived necessary conditions still hold. To do so, we will use the following reduction lemma.

\begin{lemma}\label{lemma:reduction}
    Let $\Xcal$ be a metrizable separable Borel space and let $(\Acal,d)$ be a separable metric space. Let $S\subset \Acal$ such that $\min_{a,a'\in S}d(a,a')>0$. Then, we have $\Ccal_{setting}^{uc}(\Acal)\subset \Ccal_{setting}(S)$ for any $setting\in\{stat,memoryless,oblivious,prescient,online\}$.
    
    Further, if there is a learning rule for uniformly continuous rewards in $\Acal$ that is universally consistent under a set of processes $\tilde \Ccal$ on $\Xcal$, there is also a learning rule for unrestricted rewards in $S$ that is universally consistent under all $\tilde\Ccal$ processes.
\end{lemma}

\begin{proof}
The first claim was proven in \cite{blanchard:22e} for the specific case of stationary rewards. They show that the case of uniformly-continuous rewards on $\Acal$ is at least harder than the unrestricted rewards on $S$ through a simple reduction. Here, we show that the reduction can be extended to adversarial rewards as well. Denote $\eta = \frac{1}{3}\min_{a,a'\in S}d(a,a')$. Any realization $r:S\to [0,1]$ can be extended to a $1/\eta$-Lipschitz function $F(r):\Xcal\to \Acal$ by
\begin{equation*}
    F(r)(a) = \max\left(0,\max_{a'\in S}r(a')-d(a,a') \frac{\bar r}{\eta}\right),\quad a \in \Acal.
\end{equation*}
Then, a general reward mechanism $(r_t)_{t\geq 1}$ on $S$ can be extended to a reward mechanism on $\Acal$ such that for any realization, $r_t:a\in\Acal\to [0,1]$ is $1/\eta$-Lipschitz. Hence, the mechanism $(r_t)_{t\geq 1}$ is uniformly-continuous. From now, the same arguments as in the proof of \cite[Lemma 6.3]{blanchard:22e} show that the reduction holds and that $\Ccal^{uc}_{setting}(\Acal)\subset \Ccal_{setting}$ for the considered setting. Intuitively, since for any realization, $r_t:a\in\Acal\to [0,1]$ has zero value outside of the balls $B_d(a,\eta)$ for $a\in S$, that on the ball $B_d(a,\eta)$ for $a\in S$, the action $a$ has maximum reward, and that these balls are disjoint, without loss of generality, one can assume that a universally consistent learning rule always selects actions in $S$ under these rewards, in which case, the problem becomes equivalent to having unrestricted rewards on the action set $S$. The formal learning rule reduction is defined in the original proof, and one can check that the reduction is invariant in the process $\Xbb$. Hence, this also proves the second claim of the lemma.
\end{proof}

This lemma allows to use the necessary conditions to the unrestricted reward setting by changing the terms ``finite action set'' (resp. ``countably infinite action set'') into ``totally-bounded action set'' (resp. ``non-totally-bounded action set''). The second claim of \cref{lemma:reduction} will be useful to show that no optimistically universal learning exists for adversarial uniformly-continuous rewards either. More precisely, the following result is a direct consequence from the first claim of \cref{lemma:reduction}.

\begin{proposition}\label{prop:simple_upper_bounds}
Let $\Xcal$ be a metrizable separable Borel space and let $\Acal$ be a non-totally-bounded metric space. Then, for any reward setting, $\Ccal^{uc}\subset \Ccal_1$.
Let $\Acal$ be a totally-bounded metric space with $|\Acal|>2$. Then, for any reward setting, $\Ccal^{uc} \subset \Ccal_2$. Further, if $\Xcal$ admits a non-atomic probability measure, $\Ccal_{memoryless}^{uc}\subsetneq \Ccal_2$, $\Ccal_{oblivious}^{uc}\subset \Ccal_6$ and $\Ccal_{prescient}^{uc}\subset \Ccal_7$.
\end{proposition}

We now show that we can recover the sufficient conditions from previous sections as well. For uniformly-continuous rewards, we can show that there exists a countable set of dense policies under $\Ccal_1'$ processes, as was the case for unrestricted rewards and countable action sets.

\begin{lemma}
\label{lemma:density_uniformly_continuous}
    Let $\Acal$ be a separable metric space. There is a countable set of measurable policies $\Pi$ such that for any extended process $\tilde \Xbb=(X_t)_{t\in\Tcal}\in\Ccal_1'$, any measurable policy $\pi^*:\Xcal\to\Acal$, and any uniformly-continuous possibly stochastic rewards $(r_t)_t$, with probability one over the rewards,
    \begin{equation*}
        \begin{cases}\inf_{\pi\in\Pi} \limsup_{T\to\infty} \frac{1}{T}\sum_{t\leq T, t\in\Tcal} r_t(\pi^*(X_t)) - r_t(\pi(X_t)) \leq 0,\\
        \inf_{\pi\in\Pi} \limsup_{T\to\infty} \frac{1}{T}\sum_{t\leq T, t\in\Tcal} \bar r_t(\pi^*(X_t)) - \bar r_t(\pi(X_t)) \leq 0,
        \end{cases}
    \end{equation*}
    where $\bar r_t=\Ebb r_t$ is the immediate average reward.
\end{lemma}

\begin{proof}
For any $\epsilon>0$, let $\Delta(\epsilon)$ be the $\epsilon-$modulus of continuity of the sequence of rewards $(\bar r_t)_t$. By \cite[Lemma 6.1]{blanchard:22e} (and with a straightforward adaptation for extended processes), on an event $\Ecal$ of probability one, for any $i\geq 1$, there exists $\pi^i\in\Pi$ such that $\limsup_{T\to\infty}\frac{1}{T} \sum_{t\leq T, t\in\Tcal} \1[d(\pi^*(X_t),\pi^i(X_t)) \geq 2^{-i}]\leq 2^{-i}$, for all $i\geq 1$, $\frac{1}{T}\sum_{t\leq T, t\in\Tcal}r_t(\pi^i(X_t))-\bar r_t(\pi^i(X_t))\to 0$ and similarly for $\pi^*$, where $\bar r_t$ is the immediate expected reward at time $t$. We now suppose that this event is met. Let $\epsilon>0$, let $i\geq 1$ such that $2^{-i}\leq \Delta(\epsilon)$. Then,
\begin{align*}
    \sum_{t\leq T, t\in\Tcal} \bar r_t(\pi^*(X_t)) - \bar r_t(\pi^i(X_t)) &\leq  \sum_{t\leq T, t\in\Tcal} (\bar r_t(\pi^*(X_t)) - \bar r_t(\pi^i(X_t)))\1_{d(\pi^i(x),\pi^*(x)) < \Delta(\epsilon)}\\
    &\quad\quad+  \sum_{t\leq T, t\in\Tcal} \1_{d(\pi(x),\pi^*(x))\geq 2^{-i}}\\
    &\leq  \epsilon T +  \sum_{t\leq T, t\in\Tcal} \1_{d(\pi(x),\pi^*(x)) \geq 2^{-i}}.
\end{align*}
As a result, $ \limsup_{T\to\infty}\frac{1}{T}\sum_{t\leq T, t\in\Tcal} \bar r_t(\pi^*(X_t)) - \bar r_t(\pi^i(X_t)) \leq \epsilon + \Delta(\epsilon).$ Further, because the event $\Ecal$ is satisfied, $\limsup_{T\to\infty}\frac{1}{T}\sum_{t\leq T, t\in\Tcal}  r_t(\pi^*(X_t)) - r_t(\pi^i(X_t)) \leq \epsilon + \Delta(\epsilon).$  This holds for any $\epsilon>0$. Now because $\Delta(\epsilon)\to 0$ as $\epsilon\to 0$, we proved that on $\Ecal$,
\begin{equation*}
    \begin{cases}\inf_{\pi\in\Pi}\limsup_{T\to\infty} \frac{1}{T}\sum_{t\leq T, t\in\Tcal}  r_t(\pi^*(X_t)) -  r_t(\pi(X_t)) \leq 0,\\
    \inf_{\pi\in\Pi}\limsup_{T\to\infty} \frac{1}{T}\sum_{t\leq T, t\in\Tcal}  \bar r_t(\pi^*(X_t)) -  \bar r_t(\pi(X_t)) \leq  0.
    \end{cases}
\end{equation*}
This ends the proof of the lemma.
\end{proof}

We are now ready to generalize our algorithms from previous sections, using $\Pi$ as a countable set of functions that are dense within all policies in the uniformly-continuous rewards context. First, note that using $\EXPINF$ directly with the countable family described in \cref{lemma:density_uniformly_continuous} is universally consistent on all $\Ccal_1$ processes. This shows that we always have $\Ccal_1\subset \Ccal^{uc}$ for all models. In particular, together with \cref{prop:simple_upper_bounds}, this shows that for non-totally-bounded metric action spaces $\Acal$, we have $\Ccal^{uc}=\Ccal_1$ for all reward models.

Next, we turn to the case of finite action spaces and context spaces $\Xcal$ that do not admit a non-atomic measure. In this case, we showed that the algorithm that simply uses different $\EXPIX$ for each distinct instance is optimistically universal. In the case of uniformly-continuous rewards, we can replace $\EXPIX$ with $\EXPINF$ over a countable set of actions. This yields an optimistically universal learning rule for any totally bounded action spaces $\Acal$.

\begin{theorem}\label{thm:uc_expinf}
Let $\Xcal$ be a metrizable separable Borel space that does not admit a non-atomic probability measure. Let $\Acal$ be a totally-bounded metric space. Then, there exists an optimistically universal learning rule for uniformly-continuous rewards (in any setting) and learnable processes are exactly $\Ccal^{uc}_{stat}=\Ccal^{uc}_{online}=\Ccal_2$. 
\end{theorem}

\begin{proof}
We first describe the learning rule. For any $\epsilon>0$, let $\Acal(\epsilon)$ be an $\epsilon-$net of $\Acal$. By abuse of notation, for any $a\in\Acal$, we use the same notation $a$ for the expert which selects action $a$ at all time steps. Now consider the countable set of experts $\bigcup_{i\geq 1}\Acal(2^{-i}) = \{a_1,a_2,\ldots \}$, where the sets are concatenated by increasing order of index $i$. Now consider the learning rule that uses a distinct $\EXPINF$ over this set of experts, for each distinct instance. Formally, the learning rule is
\begin{equation*}
    f_t(\mb x_{\leq t-1},\mb r_{\leq t-1},x_t) = \EXPINF(\mb {\hat a}_{S_t},\mb r_{S_t})
\end{equation*}
where $S_t = \{t'<t:x_{t'} = x_t\}$ is the set of times that $x_t$ was visited previously and $\hat a_{t'}$ denotes the action selected at time $t'$ for $t'<t$. We now show that this learning rule is universally consistent on all $\Ccal_2$ processes for uniformly bounded rewards. In the proof of \cref{thm:bad_borel_spaces} we showed that for spaces $\Xcal$ that do not admit a non-atomic probability measure, any $\Ccal_2$ process visits a sublinear number of distinct instances almost surely. Therefore, for $\Xbb\in\Ccal_2$, on an event $\Ecal$ of probability one, we have
$|\{x\in\Xcal:\{x\}\cap\Xbb_{\leq T}\neq\emptyset\}|=o(T).$ It now suffices to adapt the proof of \cref{prop:EXP.IX_parrallel}. Let $(r_t)_t$ be an uniformly continuous reward mechanism. For $\epsilon>0$, let $\Delta(\epsilon)>0$ its $\epsilon-$modulus of continuity. We keep the same notations as in the proof of \cref{prop:EXP.IX_parrallel}. Let $S_T=\{x:\{x\}\cap\Xbb_{\leq T}\neq\emptyset\}$, $\epsilon(T)=\frac{|S_T|}{T}$ and for $x\in S_T$, let $\Tcal_T(x)=\{t\leq T:X_t=x\}$. Further, for any $x\in S_T$ we pose $\Tcal_T(x) = \{t\leq T:X_t = x\}$. Let $\Hcal_0(T) = \{x\in S_T: |\Tcal_T(x)|<\frac{1}{\sqrt {\epsilon(T)}}\}$, $\Hcal_1(T) = \{x\in S_T: \frac{1}{\sqrt {\epsilon(T)}}\leq|\Tcal_T(x)|<\ln ^8 T\}$ and $\Hcal_2(T) = \{x\in S_T:|\Tcal_T(x)|\geq \ln ^8 T\}$. Now let $\pi:\Xcal\to\Acal$ be a measurable policy. We still have
\begin{equation*}
    \frac{1}{T}\sum_{x\in \Hcal_0(T)}|\Tcal_T(x)| \leq \sqrt{\epsilon(T)}.
\end{equation*}
Next, we turn to points $x\in\Hcal_2(T)$. By \cref{thm:infinite-exp4}, conditionally on the realization $\Xbb$, for any $x\in \Hcal_2(T)$, with probability at least $1-\frac{1}{T^3}$,
\begin{equation*}
    \max_{i\leq \ln T} \sum_{t\in \Tcal_T(x)}  r_t(a_i) -  r_t(\hat a_t) \leq 4c  |\Tcal_T(x)|^{3/4} (\ln T)^{3/2}\leq 4c\frac{|\Tcal_T(x)|}{\sqrt{\ln T}}.
\end{equation*}
Therefore, since $|\Hcal_2(T)|\leq T$, by union bound, with probability at least $1-\frac{1}{T^2}:=1-p_2(T)$,
\begin{equation*}
    \sum_{x\in \Hcal_2(T)} \max_{i\leq \ln T} \sum_{t\in \Tcal_T(x)}  r_t(a_i) -  r_t(\hat a_t) \leq 4c  \frac{T}{\sqrt{\ln T}}.
\end{equation*}
We then treat points in $\Hcal_1(T)$ for which we will need to go back to the proof of the regret bounds for $\EXPINF$ and the underlying $\EXPIX$ algorithm which is used as subroutine. First we recall the structure of $\EXPINF$. Let $i(k)=\sum_{r<k}r^3$. It works by periods $[i(k)+1,i(k)+k^3)$ on which a new $\EXPIX$ learner to find the best expert within the first $k$ experts in the sequence provided to $\EXPINF$. We will refer to this as period $k$. As useful inequalities, we have $\frac{k^4}{4}\leq i(k)\leq \frac{(k+1)^4}{4}$. Let $k_0=\lceil \epsilon(T)^{-1/8}\rceil$ and focus on a period $k$ for $k\geq k_0$ of an $\EXPINF$ run. We denote by $\hat a_u$ the action selected at horizon $u$ by $\EXPINF$. Following the same arguments as in \cref{prop:EXP.IX_parrallel} and the analysis of $\EXPIX$ in \cite{neu2015explore}, for any $j\leq k_0$
\begin{equation*}
    \sum_{u=1}^{k^3} (\ell_{u,\hat a_{i(k)+u}}-\tilde \ell_{u,a_j})\leq \frac{\ln k}{\eta_{k^3}} + \sum_{u=1}^{k^3}\eta_u\sum_{i=1}^k \tilde \ell_{u,a_i}.
\end{equation*}
As a result,
\begin{align*}
    \sum_{u=1}^{k^3} \ell_{u,\hat a_{i(k)+u}}-\ell_{u,a_j}\leq 3\sqrt{k\ln k \cdot k^3} + \sum_{u=1}^{k^3} (\tilde\ell_{u,a_j}-\ell_{u,a_j}) +  \sum_{u=1}^{k^3} \sum_{i=1}^k\eta_u(\tilde\ell_{u,a_j}-\ell_{u,a_j})
\end{align*}
Now for any $a\in\Acal$, let $a^{(k_0)}=\argmin_{1\leq i\leq k_0}d(a,a_i)$ the nearest neighbor of $a$ where ties are broken alphabetically. We will sum this inequality for all $\EXPINF$ runs for $x\in\Hcal_1(T)$, and periods $k\geq k_0$ that were completed, i.e. $|\Tcal_T(x)|\geq i(k+1)$, taking $a_j=\pi(x)^{(k_0)}$. Before doing so, note that $\sum_{k'\leq k}\sqrt{3(k')^4\ln k'}\leq (k+1)^3\sqrt{\ln k}\leq 4 i(k+1)^{3/4}\sqrt{\ln i(k+1)}$. Further, for simplicity, denote by $A(T)$ (resp. $B(T)$) the sum that is obtained after summing all the terms $\sum_{u=1}^{k^3} (\tilde\ell_{u,a_j}-\ell_{u,a_j})$ (resp. $\sum_{u=1}^{k^3}\sum_{i=1}^k \eta_u(\tilde\ell_{u,a_j}-\ell_{u,a_j})$). Using these notations, we obtain
\begin{multline*}
    \sum_{x\in\Hcal_1(T)}\sum_{t\in\Tcal_T(x)}r_t(\pi(X_t)^{(k_0)})-r_t(\hat a_t) \leq  \sum_{x\in\Hcal_1(T)}\left(\frac{k_0^4}{4}+4|\Tcal_T(x)|^{3/4} +4 |\Tcal_T(x)|^{3/4}\sqrt{3\ln|\Tcal_T(x)|}\right)\\
    + A(T)+B(T).
\end{multline*}
where in the first inequality, $\frac{k_0^4}{4}$ accounts for the first $k_0$ initial periods and $4|\Tcal_T(x)|^{3/4}$ accounts for the last phase which potentially was not completed. Now recall that for each $x\in\Hcal_1(T)$, $\epsilon(T)^{-1/2}\leq |\Tcal_T(x)|<\ln^8 T$. Let $n_0\geq 1$ such that for any $n\geq n_0$, $8n^{3/4}\sqrt{3\ln n}\leq n^{7/8}$. Since on the event $\Ecal$, we have $\epsilon(T)\to 0$, there exists an index $\hat T$ such that for $T\geq \hat T$, $\epsilon(T)^{-1/2}\geq n_0$. Therefore, on $\Ecal$, for $T\geq \hat T$ we have
\begin{align*}
    \sum_{x\in\Hcal_1(T)}\left(\frac{k_0^4}{4}+20|\Tcal_T(x)|^{3/4} + |\Tcal_T(x)|^3\sqrt{3\ln|\Tcal_T(x)|}\right) &\leq 2\sqrt{\epsilon(T)}T + \sum_{x\in\Hcal_1(T)} |\Tcal_T(x)|^{7/8}\\
    &\leq (2\sqrt{\epsilon(T)} + \epsilon(T)^{1/16}) T.
\end{align*}
Next, using the same arguments as in the proof of \cref{prop:EXP.IX_parrallel}, observe that conditionally on $\Xbb$, $(A(T'))_{T'\leq T}$ is a super-martingale, with increments bounded in absolute value by $2\sqrt{\frac{k \cdot k^3}{\ln k}}\leq 2k^2\leq 4\sqrt{i(k+1)}\leq 4\ln ^4 T$. Therefore, Azuma's inequality implies that
\begin{equation*}
    \Pbb[A(T)\leq 8T^{3/4}\ln^4 T\mid \Xbb]\geq 1-e^{-2\sqrt T}.
\end{equation*}
Simialrly, $(B(T'))_{T'\leq T}$ is a super-martingale, with increments bounded in absolute value by $2k\sqrt{\frac{k \cdot k^3}{\ln k}}\leq 8 i(k+1)\leq 8\ln^8 T$. Therefore,
\begin{equation*}
    \Pbb[B(T)\leq 16T^{3/4}\ln^8 T\mid \Xbb]\geq 1-e^{-2\sqrt T}.
\end{equation*}
Therefore, by the Borel-Cantelli lemma, on an event $\Gcal$ of probability one, $\limsup_{T\to\infty}\frac{1}{T}(A(T)+B(T))\leq 0$. Finally, let $j(T) = \min(\epsilon(T)^{-1/8},\ln T)$. Putting everything together, we proved that on $\Ecal\cap\Fcal\cap\Gcal,$ for $T\geq \hat T$,
\begin{equation*}
    \frac{1}{T}\sum_{t\leq T}r_t(\pi(X_t)^{(j(T))})-r_t(\hat a_t)\leq 3\sqrt{\epsilon(T)} + \epsilon(T)^{1/16} + \frac{4c}{\sqrt{\ln T}} + \frac{1}{T}(A(T)+B(T)).
\end{equation*}
In particular, this hows that on $\Ecal\cap\Fcal\cap\Gcal$,
\begin{equation*}
    \limsup_{T\to\infty}\frac{1}{T}\sum_{t\leq T}r_t(\pi(X_t)^{(j(T))})-r_t(\hat a_t) \leq 0.
\end{equation*}
Now using Hoeffding's bound, with probability at least $1-2e^{-2\sqrt T}$, we have
\begin{equation*}
    \left|\sum_{t=1}^T r_t(\pi(X_t)^{(j(T))}) - \bar r_t(\pi(X_t)^{(j(T))})\right|\leq 2T^{3/4}.
\end{equation*}
We have the same bound for $\pi$. Therefore, the Borel-Cantelli lemma implies that on an event $\Hcal$ of probability one, $\frac{1}{T}\sum_{t=1}^T r_t(\pi(X_t)^{(j(T))}) - \bar r_t(\pi(X_t)^{(j(T))})\to 0$ and $\frac{1}{T}\sum_{t=1}^T r_t(\pi(X_t)) - \bar r_t(\pi(X_t))\to 0$. We now suppose that $\Ecal\cap\Fcal\cap\Gcal\cap\Hcal$ is met.

Now fix $\epsilon>0$. Let $k_0$ such that $2^{-k_0}\leq \Delta(\epsilon)$. Because $\Ecal$ is met, $\epsilon(T)\to 0$ and $j(T)\to\infty$. Thus, there exists $\tilde T\geq \hat T$ such that for any $T\geq \tilde T$, $\epsilon(T)\leq n_0^{-2}$ and $\Acal(2^{-k_0})\subset \{a_i, j\leq j(T)\}$. Now for $T\geq \tilde T$ and any $a\in\Acal$, we have $d(a,a^{(j(T))})\leq \Delta(\epsilon)$. As a result, using $\Hcal$,
\begin{align*}
    \limsup_{T\to\infty}\frac{1}{T}\sum_{t=1}^T r_t(\pi(X_t)) - r_t(\hat a_t) 
    &\leq \limsup_{T\to\infty}\frac{1}{T}\sum_{t=1}^T \bar r_t(\pi(X_t)) - \bar r_t(\pi(X_t)^{(j(T))})\\
    &\leq  \limsup_{T\to\infty}\frac{\tilde T}{T} + \epsilon\\
    &\leq \epsilon.
\end{align*}
In the second inquality, we used the uniform-continuity assumption on the rewards and the fact that for $T\geq \tilde T$, $d(\pi(X_t),\pi(X_t)^{(j(T))}) \leq \min_{a\in\Acal(2^{-k_0})}d(a,\pi(X_t)) \leq 2^{-k_0}\leq \Delta(\epsilon)$. Because this holds for any $\epsilon>0$ and $\Ecal\cap\Fcal\cap\Gcal\cap\Hcal$ has probability one, this proves that the learning rule is universally consistent under $\Xbb$. Then, the learning rule is universally consistent under any $\Ccal_2$ process. By \cref{prop:simple_upper_bounds}, this shows that the learnable processes are exactly $\Ccal_2$ and that this is an optimistically universal learning rule. This ends the proof of the theorem.
\end{proof}

The last algorithms needed to be adapted to the uniformly-continuous rewards setting are the algorithms for $\Ccal_5$ processes in finite action spaces. Precisely, we will show that we for totally-bounded metric action spaces $\Acal$, the set of learnable processes for uniformly-continuous adversarial rewards contains $\Ccal_5$ processes. Recall that the class of constructed algorithms in \cref{thm:C6_learnable} proceed separately on different categories of times. The category of $t$ is defined based on the number of duplicates of $X_t$ within its associated period. For each category of times, the learning rule performs a form of Hedge algorithm to perform the best strategy among strategy 0 which simply assigns a different $\EXPIX$ learner to distinct instances from the period; and strategy $j$ for $j\geq 1$ which selected actions according to a fixed policy $\pi^j$, where $\tilde \Pi=\{\pi^l,l\geq 1\}$ was a dense of policies within $\Ccal_1'$ processes.

We make the following modifications to these learning rules. First, we replace $\tilde\Pi$ with the countable set $\Pi$ of measurable policies that are dense in the uniformly-continuous rewards setting, as given by \cref{lemma:density_uniformly_continuous}. Second, for every category $p$, strategy 0 will use $\EXPIX$ learners from $\Acal(\gamma_p)$, a $\gamma_p-$nets of $\Acal$, where $\gamma_p$ is to be defined. With these modifications, we obtain the following result.

\begin{theorem}\label{thm:lower_bound_C5}
Let $\Xcal$ be a metrizable separable Borel space and let $\Acal$ be a totally-bounded metric space. Then, $\Ccal_5\subset\Ccal^{uc}_{online}$.
\end{theorem}

\begin{proof}
Fix $\Xbb\in\Ccal_5$ and let $(T_i)_{i\geq 0}$ such that with $\Tcal=\bigcup_{i\geq 0}\Tcal^i\cap\{t\geq T_i\}$, we have $(X_t)_{t\in\Tcal}\in\Ccal_1'$. We first define how we modify the learning rule from \cref{thm:C6_learnable} for this process. The functions \textsc{Phase, Stage, Period, Category} are left unchanged. In the initial phase when $t<2^{u(16p)}$, we replace $\EXPIX_\Acal$ with $\EXPINF$ run with the dense sequence of $\Acal$ with the specific order described in the previous \cref{thm:uc_expinf}. We briefly recap the procedure. Let $\Acal(\epsilon)$ be an $\epsilon-$net of $\Acal$. We consider the sequence of experts $\bigcup_{i\geq 1}\Acal(2^{-i})$ where we confuse $a\in\Acal$ with the constant policy equal to $a$ and we concatenate the nets by increasing order of index $i$. $\EXPINF$ is then run with this sequence of experts. Next, we enumerate $\Pi=\{\pi^l,l\geq 1\}$ and use these policies as well for the learning rule (strategies $j\geq 1$). Last, when playing strategy 0 after the initial phase, we replace $\EXPIX_\Acal$ with $\EXPIX_{\Acal(\gamma_p)}$, where $\gamma_p$ will be defined shortly. In the original proof, we defined $\delta_p:=6\frac{\sqrt{|\Acal|\ln|\Acal|}}{2^p}$, $\eta_i:=\sqrt{\frac{8\ln(i+1)}{2^i}}$ and showed that the average error of the learning rule on $\Tcal_p$ outside of the intitial phase is $\Ocal(\delta_p+\frac{\eta_{16p}}{4})$. Then, $\sum_{p\geq 0}(\delta_p+\frac{\eta_{16p}}{4})<\infty$ allowed the learner to converge separately on each $\Tcal_p$. We now replace $\delta_p$ with $\delta_p:=4\sqrt{\frac{|\Acal(\gamma_p)|\ln|\Acal(\gamma_p)|}{2^p}}$ and choose $\gamma_p$ such that $\sum_p \delta_p<\infty$. We pose
\begin{equation*}
    \gamma_p = \min\{2^{-i}:|\Acal(2^{-i})|\ln|\Acal(2^{-i})| \leq 2^{p/4}\}.
\end{equation*}
Thus, we still have $\sum_p\delta_p<\infty$ and $\gamma_p\to 0$. We now show that the modified learning rule is universally consistent under online uniformly-continuous rewards on $\Acal$. Fix $(r_t)_t$ such a reward mechanism and for $\epsilon>0$, let $\Delta(\epsilon)$ the $\epsilon-$modulus of continuity of the sequence of immediate rewards. As in the original proof of \cref{thm:C6_learnable}, let $\Tcal^{init} = \bigcup_{p\geq 0} \{t\in\Tcal_p:t<2^{u(16p)}\}$ be the initial phase. The process $(X_t)_{t\in\Tcal^{init}}$ still visits a sublinear number of distinct instances almost surely, where we say that two instances $t,t'\in\Tcal^{init}$ are duplicates if and only if they have same category, period and $X_t=X_{t'}$. As a result, in the proof of \cref{thm:uc_expinf}, we showed that for any $\pi^*:\Xcal\to\Acal$, on an event $\Ecal$ of probability one,
\begin{equation*}
    \limsup_{T\to\infty}\frac{1}{T}\sum_{t\leq T,t\in\Tcal^{init}}r_t(\pi^*(X_t))-r_t \leq 0.
\end{equation*}
We then turn to non-initial phases and adapt the original proof of \cref{thm:C6_learnable}. For any $a\in\Acal$, we denote $a^{(\gamma)}=\argmin_{a'\in\Acal(\gamma)}d(a,a')$, the nearest neighbor of $a$ within the $\gamma-$net where ties are broken alphabetically. Keeping the same event $\Fcal$, Eq~\eqref{eq:accuracy_estimates} is unchanged and Eq~\eqref{eq:accuracy_strat0} becomes
\begin{equation*}
    R_p(l,k;0)\geq \sum_{t\in\Tcal_p(l,k)}r_t(\pi^*(X_t)^{(\gamma_p)}) -16|\Acal(\gamma_p)|^2 2^{-i} 2^{15l/16}- \delta_p 2^{l-i}.
\end{equation*}
Eq~\eqref{eq:outer_hedge} is left unchanged. For $p\geq 0$, let $\epsilon(p)=\min\{2^{-i}: \gamma_p\leq \Delta(2^{-i})\}$. Note that because $\gamma_p\to 0$, we have $\epsilon(p)\to 0$ as $p\to\infty$. Following the same arguments as in the original proof and noting that $|\Acal(\gamma_p)|\leq 2^{p/4}$, Eq~\eqref{eq:tail_upper_bound} is replaced by
\begin{align*}
    \sum_{2^{u(16p)}<t\leq T,t\in\Tcal_p} r_t(\pi^*(X_t)^{(\gamma_p)})-r_t &\leq 2^{\hat l}+ c2^{p/2}T^{15/16} + \left(\delta_p+\frac{\eta_{16p}}{4}\right) T\\
    &\leq 2^{\hat l}+ cT^{31/32} + \left(\delta_p+\frac{\eta_{16p}}{4}\right).
\end{align*}
Now fix $\epsilon>0$, and let $p_0$ such that $\sum_{p\geq p_0}(\delta_p+\frac{\eta_{16p}}{4})<\epsilon$ and $\epsilon(p_0)<\epsilon$. Following the original arguments,
\begin{align*}
    \sum_{p\geq p_0}\sum_{2^{u(16p)}\leq t<T,t\in\Tcal_p}  r_t(\pi^*(X_t)^{(\gamma_p)})-r_t \leq 2^{\hat l} \log_4 T + cT^{31/32}\log_4 T+ \epsilon T.
\end{align*}
Now using Azuma's inequality, with probability at least $1-4e^{-2\sqrt T}$, we have
\begin{align*}
    \left|\sum_{p\geq p_0}\sum_{2^{u(16p)}\leq t<T,t\in\Tcal_p}  r_t(\pi^*(X_t)^{(\gamma_p)}) -\bar r_t(\pi^*(X_t)^{(\gamma_p)}) \right|\leq 2T^{3/4}\\
    \left|\sum_{p\geq p_0}\sum_{2^{u(16p)}\leq t<T,t\in\Tcal_p}  r_t(\pi^*(X_t)) -\bar r_t(\pi^*(X_t)) \right|\leq 2T^{3/4}.
\end{align*}
Therefore, using Borel-Cantelli, on an event $\Gcal$ of probability one, there exists $\hat T_1$ such that for $T\geq \hat T_1$, the above two equations hold. Then, on $\Ecal\cap\Fcal\cap\Gcal$, for $T$ sufficiently large,
\begin{align*}
    \sum_{p\geq p_0}\sum_{2^{u(16p)}\leq t<T,t\in\Tcal_p}  r_t(\pi^*(X_t))-r_t &\leq 2^{\hat l} \log_4 T + cT^{31/32}\log_4 T+ \epsilon T + 4T^{3/4}\\
    &\quad\quad\quad + \sum_{p\geq p_0}\sum_{2^{u(16p)}\leq t<T,t\in\Tcal_p}  \bar r_t(\pi^*(X_t))- \bar r_t(\pi^*(X_t)^{(\gamma_p)})\\
    &\leq 2^{\hat l} \log_4 T + 4T^{3/4}+cT^{31/32}\log_4 T+ 2\epsilon T ,
\end{align*}
where in the last inequality we used the uniform continuity of the immediate expected rewards since for $p\geq p_0$, one has $\gamma_p\leq \gamma_{p_0} \leq \Delta(\epsilon(p_0)) \leq \Delta(\epsilon)$.
This implies that on the event $\Ecal\cap\Fcal\cap\Gcal$,
\begin{equation*}
    \limsup_{T\to\infty}\frac{1}{T} \sum_{p\geq p_0}\sum_{2^{u(16p)}\leq t<T,t\in\Tcal_p}  r_t(\pi(X_t))-r_t \leq 2\epsilon.
\end{equation*}
Now for $p<p_0$, by \cref{lemma:density_uniformly_continuous}, on an event $\Hcal_p$ of probability one, there exists $l^p$ such that
\begin{equation*}
    \limsup_{T\to\infty}\frac{1}{T}\sum_{t\leq T,t\in\Tcal_p}  r_t(\pi^*(X_t))- r_t(\pi^{l_p}(X_t))\leq \frac{\epsilon}{p_0}.
\end{equation*}
Following the arguments in the proof of \cref{thm:C6_learnable}, on the event $\Ecal\cap\Fcal\cap\Gcal\cap\bigcap_{p<p_0}\Hcal_p$ of probability one, for $T$ large enough,
\begin{align*}
    \sum_{p<p_0}\sum_{2^{u(16p)}\leq t\leq T, t\in\Tcal_p}  r_t(\pi^*(X_t))-r_t 
    &\leq \sum_{p<p_0} \sum_{2^{u(16p)}\leq t\leq T,t\in\Tcal_p}  r_t(\pi^*(X_t)) -  r_t(\pi^{l_p}(X_t))\\
    &\quad\quad + \sum_{p<p_0} \sum_{2^{u(16p)}\leq t\leq T,t\in\Tcal_p}  r_t(\pi^{l_p}(X_t)) -r_t\\
    &\leq \sum_{p<p_0} \sum_{2^{u(16p)}\leq t\leq T,t\in\Tcal_p}  r_t(\pi^*(X_t)) -  r_t(\pi^{l_p}(X_t)) \\
    &\quad\quad\quad\quad\quad+ 2^{\hat l_1} + 2^{-i(T)}T + cp_0T^{15/16} + \epsilon T.
\end{align*}
As a result,
\begin{equation*}
    \limsup_{T\to\infty}\frac{1}{T} \sum_{p<p_0}\sum_{2^{u(16p)}\leq t\leq T, t\in\Tcal_p}  r_t(\pi^*(X_t))-r_t \leq 2\epsilon.
\end{equation*}
Combining all the estimates together, we proved that on $\Ecal\cap\Fcal\cap\Gcal\cap\bigcap_{p<p_0}\Hcal_p$ of probability one,
\begin{equation*}
    \limsup_{T\to\infty}\frac{1}{T} \sum_{t=1}^T  r_t(\pi^*(X_t))-r_t \leq 4\epsilon.
\end{equation*}
This holds for all $\epsilon>0$. The same arguments as in the original proof conclude that the learning rule is universally consistent under $\Xbb$. This ends the proof of the theorem.
\end{proof}

As a summary, we generalized all results from the case of the unrestricted reward to uniformly-continuous rewards with the corresponding assumptions on action spaces.

\acks{Moise Blanchard and Patrick Jaillet were partly funded by ONR grant N00014-18-1-2122.}

\bibliography{learning}
\end{document}